%% file: neurips_2022.tex
\pdfoutput=1
\documentclass{article}

\PassOptionsToPackage{numbers, compress}{natbib}
\usepackage[final]{neurips_2022}

\usepackage[utf8]{inputenc} 
\usepackage[T1]{fontenc}    
\usepackage{hyperref}       
\usepackage{url}            
\usepackage{booktabs}       
\usepackage{amsfonts}       
\usepackage{nicefrac}       
\usepackage{microtype}      
\usepackage{xcolor}         

\title{Sampling with Riemannian Hamiltonian Monte Carlo in a Constrained Space}

\author{%
  Yunbum Kook \\
  Georgia Tech\\
  \texttt{yb.kook@gatech.edu}\\
  \And
  Yin Tat Lee \\
  Microsoft Research,\\ University of Washington\\
  \texttt{yintat@uw.edu} \\
  \AND
  Ruoqi Shen \\
  University of Washington \\
  \texttt{shenr3@cs.washington.edu}\\
  \And
  Santosh S. Vempala \\
  Georgia Tech \\
  \texttt{vempala@gatech.edu} \\
}

\usepackage{amsmath}
\usepackage{amssymb}
\usepackage{mathtools}
\usepackage{amsthm}
\usepackage{verbatim}
\usepackage{thmtools} 
\usepackage{thm-restate}

\numberwithin{equation}{section}
\numberwithin{figure}{section}
\theoremstyle{plain}
\newtheorem{thm}{Theorem}
\theoremstyle{plain}

\theoremstyle{plain}
\newtheorem{lem}[thm]{Lemma}
\theoremstyle{plain}

\theoremstyle{remark}
\newtheorem*{rem*}{Remark}
\theoremstyle{definition}

\theoremstyle{remark}
\newtheorem{rem}[thm]{Remark}
\theoremstyle{plain}
\newtheorem*{question*}{Question}
\theoremstyle{plain}

\theoremstyle{plain}
\newtheorem*{corollary}{Corollary}
\theoremstyle{plain}
\newtheorem{definition}{Definition}

\newcommand{\Par}[1]{\left(#1\right)}
\newcommand{\Brack}[1]{\left[#1\right]}

\newcommand{\bx}{\overline{x}}
\newcommand{\bv}{\overline{v}}
\newcommand{\bh}{\overline{H}}

\newcommand{\N}{\mathcal N}
\newcommand{\half}{\frac 1 2}

\newcommand{\bg}{\overline{g}}

\newcommand{\TV}{\mathrm{TV}}

\newcommand{\del}{\partial}

\usepackage{comment}
\definecolor{blue}{rgb}{0.3, 0.3, 0.7}
\definecolor{purple}{rgb}{0.8, 0.3, 0.7}
\usepackage{tcolorbox}

\hypersetup{
	colorlinks,
	allcolors=blue
}

\usepackage[lined,boxed,ruled,norelsize,algo2e]{algorithm2e}
\usepackage{graphicx}
\usepackage{array}
\usepackage{tikz}
\usetikzlibrary{calc}
\usetikzlibrary{positioning}
\usetikzlibrary{arrows.meta}

\usepackage[font=small,labelfont=bf]{caption}

\usepackage{blindtext}
\newcommand{\norm}[1]{\left\lVert#1\right\rVert}
\usepackage{paralist}
\allowdisplaybreaks

\begin{document}
\global\long\def\defeq{\stackrel{\mathrm{{\scriptscriptstyle def}}}{=}}%
\global\long\def\norm#1{\left\Vert #1\right\Vert }%
\global\long\def\R{\mathbb{R}}%
\global\long\def\tcal{\mathbb{\mathcal{T}}}%
\global\long\def\mcal{\mathbb{\mathcal{M}}}%
\global\long\def\sp{\textrm{sp}}%
\global\long\def\half{\frac{1}{2}}%
\global\long\def\mhalf{-\frac{1}{2}}%
\global\long\def\ot{\frac{1}{3}}%
\global\long\def\tt{\frac{2}{3}}%

\global\long\def\Rn{\mathbb{R}^{n}}%
\global\long\def\tr{\mathrm{Tr}}%
\global\long\def\diag{\mathrm{diag}}%
\global\long\def\Diag{\mathrm{Diag}}%
\global\long\def\cov{\mathrm{Cov}}%
\global\long\def\E{\mathbb{E}}%
\global\long\def\P{\mathbf{P}}%
\global\long\def\Var{\mathrm{Var}}%
\global\long\def\rank{\mathrm{rank}}%
\global\long\def\pdet{\mathrm{pdet}}%
\global\long\def\range{\mathrm{Range}}%
\global\long\def\nulls{\mathrm{Null}}%
\global\long\def\ox{\overline{x}}%
\global\long\def\init{\mathrm{(init)}}%
\global\long\def\xmid{x_{\text{mid}}}%
\global\long\def\vmid{v_{\text{mid}}}%
\global\long\def\mid{\text{mid}}%
\global\long\def\tx{\widetilde{x}}%
\global\long\def\tv{\widetilde{v}}%
\global\long\def\T{\mathcal{T}}%
\global\long\def\Px{P^{(x)}}%
\global\long\def\Pv{P^{(y)}}%
\global\long\def\vbad{V_{\text{bad}}}%
\global\long\def\vgood{V_{\text{good}}}%
\global\long\def\idftn{\text{id}}%

\maketitle

\begin{abstract}
We demonstrate for the first time that ill-conditioned, non-smooth,
constrained distributions in very high dimension, upwards of 100,000,
can be sampled efficiently \emph{in practice}. Our algorithm incorporates
constraints into the Riemannian version of Hamiltonian Monte Carlo
and maintains sparsity. This allows us to achieve a mixing rate independent
of condition numbers.

On benchmark data sets from systems biology and linear programming,
our algorithm outperforms existing packages by orders of magnitude.
In particular, we achieve a 1,000-fold speed-up for sampling from
the largest published human metabolic network (RECON3D). Our package
has been incorporated into the COBRA toolbox.
\end{abstract}

\input{10introduction.tex}

\input{21algorithm.tex}

\input{22algorithm.tex}

\input{23algorithm.tex}

\input{30expts.tex}

\paragraph*{Acknowledgement.}
The authors are grateful to Ben Cousins for helpful discussions, and to Ronan Fleming, Ines Thiele and their research groups for advice on metabolic models. This work was supported in part by NSF awards DMS-1839116, DMS-1839323, CCF-1909756, CCF-2007443 and CCF-2134105.

\bibliographystyle{plain}
\bibliography{neurips_2022}

\appendix

 \input{93appendix_exp} 
 \input{91appendix_ideal_detail}    
 \input{92appendix_discretization}

\input{90appendix_algo}

 \input{94appendix_prelim}

\end{document}

%% file: 10introduction.tex
\section{Introduction}

\paragraph*{Sampling is Fundamental.}

Sampling algorithms arise naturally in models of statistical physics,
e.g., Ising, Potts models for magnetism, Gibbs model for gases, etc.
These models directly suggest Markov chain algorithms for sampling the corresponding configurations.  In the Ising model where the vertices of a graph are assigned a spin, i.e., $\pm1$, in each step, we pick a vertex at random and flip its spin with some probability. The probability is chosen so that the distribution of the vector of all spins approaches a target distribution where the probability exponentially decays with the number of agreements in spin for pairs corresponding to edges of the graph. In the Gibbs model, particles move randomly with collisions and their motion is often modeled as reflecting Brownian motion.
Sampling with Markov chains is today the primary algorithmic approach for high-dimensional
sampling. For some fundamental problems, sampling with Markov chains
is the only known efficient approach or the only approach to have
guarantees of efficiency. Two notable examples are sampling perfect
matchings of a bipartite graph and sampling points from a convex body.
These are the core subroutines for estimating the permanent of a nonnegative
matrix and estimating the volume of a convex body, respectively. The
solution space for these problems scales exponentially with the dimension.
In spite of this, polynomial-time algorithms have been discovered
for both problems. The current best permanent algorithm scales as
$n^{7}$ (time) \cite{bezakova2008accelerating,jerrum2004polynomial},
while the current best volume algorithm scales as $n^{3}$ (number
of membership tests) \cite{jia2021reducing}. For the latter, the
first polynomial-time algorithm had a complexity of $n^{27}$ \cite{dyer1991random},
and the current best complexity is the result of many breakthrough
discoveries, including general-purpose algorithms and analysis tools.

\paragraph*{Sampling is Ubiquitous.}

The need for efficient high-dimensional sampling arises in many fields.
A notable setting is \emph{metabolic networks} in systems biology.
A constraint-based model of a metabolic network consists of $m$ metabolites
and $n$ reactions, and a set of equalities and inequalities that
define a set of feasible steady state reaction rates (fluxes): 
\[
\Omega=\left\{ v\in\R^{n} \,\vert\, Sv=0,\,l\le v\le u,\,c^{T}v=\alpha\right\} ,
\]
where $S$ is a stoichiometric matrix with coefficients for each metabolite
and reaction. The linear equalities ensure that the fluxes into and
out of every node are balanced. The inequalities arise from thermodynamical
and environmental constraints. Sampling constraint-based models is
a powerful tool for evaluating the metabolic capabilities of biochemical
networks \cite{lewis2012constraining,thiele2013community}. While
the most common distribution used is uniform over the feasible region,
researchers have also argued for sampling from the Gaussian density
restricted to the feasible region; the latter has the advantage that
the feasible set does not have to be bounded. A previous approach
to sampling, using hit-and-run with rounding \cite{haraldsdottir2017chrr},
has been incorporated into the COBRA package \cite{heirendt2019creation}
for metabolic systems analysis (Bioinformatics).

A second example of mathematical interest is the problem of computing
the \href{http://math.sfsu.edu/beck/birkhoff/volumes.html}{volume}
of the Birkhoff polytope. For a given dimension $n$, the Birkhoff
polytope is the set of all doubly stochastic $n\times n$ matrices
(or the convex hull of all permutation matrices). This object plays
a prominent role in algebraic geometry, probability, and other fields.
Computing its volume has been pursued using algebraic representations;
however exact computations become intractable even for $n=11$, requiring
years of computation time. Hit-and-run has been used to show that
sampling-based volume computation can go to higher dimension \cite{cousins2016practical},
with small error of estimation. However, with existing sampling implementations,
going beyond $n=20$ seems prohibitively expensive.

A third example is from machine learning, a field that is increasingly
turning to \emph{sampling }models of data according to their performance
in some objective. One such commonly used criterion is the logistic
regression function. The popularity of logistic regression has led
to sampling being incorporated into widely used packages such as STAN
\cite{stan}, PyMC3 \cite{salvatier2016probabilistic}, and Pyro \cite{bingham2019pyro}. However, those packages in general do not {run on} the constraint-based models we are interested in.

\paragraph*{Problem Description.}

In this paper, we consider the problem of sampling from distributions
whose densities are of the form
\begin{equation}
e^{-f(x)}\text{ subject to }Ax=b,x\in K\label{eq:problem-1}
\end{equation}
where $f$ is a convex function and $K$ is a convex body. We assume
that a self-concordant barrier $\phi$ for $K$ is given. Note that
any convex body has a self-concordant barrier \cite{lee2021universal}
and there are explicit barriers for convex bodies that come up in 
practical applications~\cite{nesterov1994interior}, so this is a mild assumption.
We introduce an efficient algorithm for the problem when
$K$ is a product of convex bodies $K_{i}$, each with small dimension. Many
practical instances can be written in this form. As a special case, the algorithm can handle $K$ in the
form of $\{x\in\Rn:l_{i}\leq x_{i}\leq u_{i}\text{ for all }i\in[n]\}$
with $l_{i}\in\R\cup\{-\infty\}$ and $u_{i}\in\R\cup\{+\infty\}$,
which is the common model structure in systems biology. Moreover,
any generalized linear model $\exp(-\sum f_{i}(a_{i}^{\top}x-b_{i}))$,
e.g., the logistic model, can be rewritten in the form 
\begin{equation}
\exp(-\sum t_{i})\text{ subject to }Ax=b+s,(s,t)\in K\label{eq:problem-2}
\end{equation}
where $K=\Pi K_{i}$ and each $K_{i}=\{(s_{i},t_{i}):f_{i}(s_{i})\leq t_{i}\}$
is a two-dimensional convex body.

\paragraph{The Challenges of Practical Sampling.}

High dimensional sampling has been widely studied in both the theoretical
computer science and the statistics communities. Many popular samplers
are first-order methods, such as MALA \cite{roberts1996exponential},
basic HMC \cite{neal2011mcmc,duane1987hybrid} and NUTS \cite{hoffman2014no},
which update the Markov chain based on the gradient information of
$f$. The runtime of such methods can depend on the condition number
of the function $f$ \cite{dwivedi2018log,lee2020logsmooth,chen2020fast,cheng2018underdamped,shen2019randomized}.
However, the condition number of real-world applications can be very
large. For example, RECON1~\cite{king2016bigg}, a reconstruction
of the human metabolic network, can have condition number as large
as $10^{6}$ due to the dramatically different orders of different
chemicals' concentrations. Motivated by sampling from ill-conditioned
distributions, another class of samplers use higher-order information
such as Hessian of $f$ to take into account the local structure of
the problems \cite{simsekli2016stochastic,chewi2020exponential}.
However, such samplers cannot handle non-smooth distributions, such
as hinge-loss, lasso, or uniform densities over polytopes.

For non-smooth distributions, the best polytime methods are based
on discretizations of Brownian motion, e.g., the Ball walk \cite{kannan1997random}
(and its affine-invariant cousin, the Dikin walk \cite{kannan2012random}),
which takes a random step in a ball of a fixed size around the current
point. Hit-and-Run \cite{lovasz2006hit} builds on these by avoiding
an explicit step size and going to a random point along a random line
through the current point. Both approaches hit the same bottleneck
--- in a polytope that contains a unit ball, the step size should
be $O(1/\sqrt{n})$ to avoid stepping out of the body with large probability.
This leads to quadratic bounds (in dimension) on the number of steps
to ``mix''. 

Due to the reduction mentioned in (\ref{eq:problem-2}), non-smooth
distributions can be translated to the form in (\ref{eq:problem-1})
with constraint $K$. Both the first and higher-order sampler and
the polytime non-smooth samplers have their limitations in handling
distributions with non-smooth objective function or constraint $K$. 
Given the limitations of all previous samplers, a natural question
we want to ask is the following.
\begin{question*}
Can we develop a practically efficient sampler that can handle the constrained problem in
\eqref{eq:problem-1} and preserve sparsity\footnote{When $A$ is sparse, preserving the sparsity of $A$ can greatly enhance both the runtime and the space efficiency.} with mixing time independent
of the condition number?
\end{question*}
In some applications, smoothness and condition number can be controlled
with tailor-made models. Our goal here is to propose a general solver
that can sample from any non-smooth distributions as given. For traditional
samplers such as the Ball walk and Hit-and-Run, as mentioned earlier,
the step size needs to be small so that the process does not step
out. An approach that gets around this bottleneck is Hamiltonian Monte
Carlo (HMC), where the next step is given by a point along a Hamiltonian-preserving
curve according to a suitably chosen Hamiltonian. It has two advantages.
First, the steps are no longer straight lines in Euclidean space,
and we no longer have the concern of ``stepping out''. Second, the
process is \emph{symplectic} (so measure-preserving), and hence the
filtering step is easy to compute. It was shown in \cite{lee2018convergence}
that significantly longer steps can be taken and the process with
a convergence analysis in the setting of Hessian manifolds, leading
to subquadratic convergence for uniformly sampling polytopes.

To make this practical, however, is a formidable challenge. There are two
high-level difficulties. One is that many real-world instances are
\emph{highly skewed} (far from isotropic) and hence it is important to use
the local geometry of the density function. This means efficiently
computing or maintaining second-order information such as a Hessian
of the logarithm of the density. This can be done in the Riemannian
HMC (RHMC) framework \cite{girolami2011riemann,lee2018convergence},
but the computation of the next step requires solving the Hamiltonian
ODE to high accuracy, which in turn needs the computation of leverage
scores, a procedure that takes at least matrix-multiplication time
in the worst case. Another important difficulty is maintaining hard
linear constraints. Existing high-dimensional packages do not allow for
constraints (they must be somehow incorporated into the target density),
and RHMC is usually considered with a full-dimensional feasible region
such as a full-dimensional polytope. This can also be done in the
presence of linear equalities by working in the affine subspace defined
by the equalities, but this has the effect of \emph{losing any sparsity}
inherent in the problem and turning all coefficient matrices and objective
coefficients into dense objects, thereby potentially incurring a quadratic
blow-up. 

\paragraph*{Our Solution: Constrained Riemannian Hamiltonian Monte Carlo (CRHMC).}

We develop a constrained version of RHMC, maintaining both \emph{sparsity}
and \emph{constraints}. Our refinement of RHMC ensures that the process satisfies
the given constraints throughout, without incurring a significant
overhead in time or sparsity. It works even if the resulting feasible
region is poorly conditioned. Since many instances in practice are
ill-conditioned and have degeneracies, we believe this is a crucial
aspect. Our algorithm outperforms existing packages by orders of magnitude. 

In Section \ref{sec:Algorithm}, we give the main ingredients of the algorithm and discuss how we overcome the challenges that prevent us from sampling efficiently in practice. Following that,
in Section \ref{sec:Experiments}, we present empirical results on
several benchmark datasets, showing that CRHMC successfully samples
much larger models than previously known to be possible, and is significantly
faster in terms of rate of convergence (``number of steps'') and
total sampling time. Our complete package is available on \href{https://github.com/ConstrainedSampler/PolytopeSamplerMatlab}{GitHub}.
{We refer the reader to \hyperref[app:outline]{Appendix} for theory, notations, and definitions.}

%% file: 21algorithm.tex
\section{Algorithm: Constrained RHMC}

\label{sec:Algorithm}

In this section, we propose a constrained Riemannian Hamiltonian Monte
Carlo (CRHMC\footnote{pronounced ``crumch''.}) algorithm to sample
from a distributions of the form
\[
e^{-f(x)}\text{ subject to }c(x)=0 \text{ and $x\in K$ for some convex body }K,
\]
where the constraint function $c:\Rn\rightarrow\R^{m}$ satisfies
the property that the Jacobian $Dc(x)$ has full rank for all $x$
such that $c(x)=0$. It is useful to keep in mind the case when $c(x)=0$
is an affine subspace $Ax=b$, in which case $Dc(x)=A$, and the full-rank
condition simply says that the rows of $A$ are independent.

We refer readers to \cite{andersen1983rattle, brubaker2012family, reich1993symplectic} for preliminary versions of CRHMC called the constrained Hamiltonian Monte Carlo (CHMC).
In particular, a framework in \cite{brubaker2012family} can be extended to CRHMC when $K=\Rn$, and in fact they mention CRHMC as a possible variant. 
However, their algorithm for CRHMC requires eigenvalue decomposition and is not efficient for large problems, {which takes $n^3$ time and $n^2$ space per MCMC step in practice. 
In this section, we propose an algorithm that overcomes those limitations and satisfies the additional constraint $K$ by using a local metric induced by the Hessian of self-concordant barriers, leading to $n^{1.5}$ time and $n$ space in practice.

\subsection{Basics of CRHMC\label{subsec:Basics-of-RHMC}}

To introduce our algorithm, we first recall the RHMC algorithm (Algorithm
\ref{alg:HMC}). In RHMC, we extend the space $x$ to the pair $(x,v)$,
where $v$ denotes the \emph{velocity}. Instead of sampling from $e^{-f(x)}$,
RHMC samples from the distribution $e^{-H(x,v)}$, where $H(x,v)$
is the Hamiltonian, and then outputs $x$. To make sure the distribution
is correct, we choose the Hamiltonian such that the marginal of $e^{-H(x,v)}$
along $v$ is proportional to $e^{-f(x)}$. One common choice of $H(x,v)$
is 
\begin{equation}
H(x,v)=f(x)+\frac{1}{2}v^{\top}M(x)^{-1}v+\frac{1}{2}\log\det M(x),\label{eq:standard_Ham}
\end{equation}
where $M(x)$ is a position-dependent positive definite matrix defined
on $\Rn$.

\begin{algorithm2e}[h!]

\caption{$\texttt{Riemannian Hamiltonian Monte Carlo}$ (RHMC)}

\label{alg:HMC}

\SetAlgoLined

\textbf{Input:} Initial point $x^{(0)}$, step size $h$

\For{$k=1,2,\cdots$}{

\tcp{Step 1: resample $v$}

Sample $v^{(k-\frac{1}{2})}\sim\mathcal{N}(0,M(x^{(k-1)}))$ and set
$x^{(k-\frac{1}{2})}\leftarrow x^{(k-1)}$.

\ 

\tcp{Step 2: Hamiltonian dynamics}

Solve the ODE 
\begin{equation}
\frac{dx}{dt}=\frac{\partial H(x,v)}{\partial v},\ \frac{dv}{dt}=-\frac{\partial H(x,v)}{\partial x} \label{eq:ham_ODE}
\end{equation}
with $H$ defined in (\ref{eq:standard_Ham}) and the initial point
given by $(x^{(k-\frac{1}{2})},v^{(k-\frac{1}{2})})$.

Set $x^{(k)}\leftarrow x(h)$ and $v^{(k)}\leftarrow v(h)$.

}

\textbf{Output:} $x^{(k)}$

\end{algorithm2e}

\label{subsec:Basics-of-CRHMC}
To extend RHMC to the constrained case, we need to make sure both
Step 1 and Step 2 satisfy the constraints, so the Hamiltonian dynamics
has to maintain $c(x)=0$ throughout Step 2.
Note that
\begin{equation}
\frac{d}{dt}c(x_{t})=Dc(x_{t})\cdot\frac{dx_{t}}{dt}=Dc(x_{t})\cdot\frac{\partial H(x_{t},v_{t})}{\partial v_{t}}, \label{eq:dc_cond}
\end{equation}
where $Dc(x)$ is the Jacobian of $c$ at $x$. With $H$ defined
in (\ref{eq:standard_Ham}), Condition (\ref{eq:dc_cond}) becomes
$Dc(x)M(x)^{-1}v=0$. However, for full rank $Dc(x)$, if $M(x)$
is invertible, then $\range(v)=\range(\mathcal{N}(0,M(x)))=\Rn$ immediately
violates this condition due to $\textrm{dim(}\nulls(Dc(x)M^{-1}(x)\textrm{)})=n-m$.
To get around this issue, we use a non-invertible matrix $M(x)$ with
its pseudo-inverse $M(x)^{\dagger}$ to satisfy $Dc(x)M(x)^{\dagger}v=0$
for any $v\in\range(M(x))$. Since we want the step to be able to
move in all directions satisfying $c(x)=0$, we impose the following
condition with $\range(M(x))=\range(M(x)^{\dagger})$ in mind:
\begin{equation}
\range(M(x))=\nulls(Dc(x))\text{ for all }x\in\Rn,\label{eq:CHMC_cond1}
\end{equation}
which can be achieved by $M(x)$ proposed soon.

Under the condition (\ref{eq:CHMC_cond1}), we sample $v$ from $\mathcal{N}(0,M(x))$
in Step 1, which is equivalent to sampling from $e^{-H(x,v)}$ subject
to $v\in\range(M(x))=\nulls(Dc(x))$. Also, the stationary distribution
of CRHMC should be proportional to
\[
e^{-H(x,v)}\text{ subject to }c(x)=0\text{ and }v\in\nulls(Dc(x)).
\]
Here, to maintain $v\in\nulls(Dc(x))$ during Step 2 we add a Lagrangian
term to $H$. Without the Lagrangian term, $v_{t}$ would escape from
$\nulls(Dc(x_{t}))=\range(M(x_{t}))$ in Step 2 as seen in the proof of Lemma~\ref{lem:CHMC_inva}, which contradicts $\range(v_{t})=\range(\mathcal{N}(0,M(x_{t})))=\range(M(x_{t}))$.
The constrained Hamiltonian we propose is (See its rigorous derivation in Lemma~\ref{lem:CHMC_inva})
\begin{equation}
H(x,v)=\overline{H}(x,v)+\lambda(x,v)^{\top}c(x)\quad\text{with}\quad\overline{H}(x,v)=f(x)+\frac{1}{2}v^{\top}M(x)^{\dagger}v+\log\pdet(M(x))\label{eq:const_H}
\end{equation}
where
$
\lambda(x, v)
=
(Dc(x)Dc(x)^{\top})^{-1}\left(D^{2}c(x)[v,\frac{dx}{dt}]-Dc(x)\frac{\partial\overline{H}(x, v)}{\partial x}\right).
$
Here, $\pdet$ denotes pseudo-determinant and $\lambda(x,v)$ is picked
so that $v\in\nulls(Dc(x))$. An algorithmic description of CRHMC is the same as Algorithm~\ref{alg:HMC}
with the constrained $H$ in place of the unconstrained $\overline{H}$. We show the convergence of CRHMC to the correct distribution $\exp(-f(x))$ in Appendix \ref{subsec:Stationarity-of-CRHMC}.

\paragraph{Choice of $M$ via Self-concordant Barriers.}

The construction of the Hamiltonian \eqref{eq:const_H} relies on having
a family of positive semi-definite matrix $M(x)$ satisfying the condition
(\ref{eq:CHMC_cond1}) (i.e., $\range(M(x))=\nulls(Dc(x))$). One
natural choice is the orthogonal projection to $\nulls(Dc(x))$: 
\begin{equation}
	Q(x)=I-Dc(x)^{\top}(Dc(x)Dc(x)^{\top})^{-1}Dc(x)\label{eq:Qx},
\end{equation}
which is similar to the choice in \cite{brubaker2012family}. 

For the problem we care about, there are additional constraints on
$x$ other than $\{c(x)=0\}$. In the standard HMC algorithm, we have  $\frac{dx}{dt}\sim\mathcal{N}(0,M(x)^{-1})$.
For example, for a simple constraint $K=[0, 1]$, to ensure every direction is moving towards/away from $x=0$ multiplicatively,
a natural choice of $M$ is $M(x)=\diag(x^{-2})$. 
For general convex body $K$, we can use
a \textit{self-concordant barrier}, a function defined on $K$ such that $\phi(x)$ is self-concordant and  $\phi(x)\rightarrow+\infty$
as $x\rightarrow\partial K$. 
Using the barrier $\phi$, we can define the local metric based on 
$g(x)=\nabla^{2}\phi(x)$.
Intuitively, as the sampler approaches $\partial K$, the local metric stretches accordingly so that the Hamiltonian dynamics never passes the barrier, respecting $x\in K$ throughout.

In summary, we need $M(x)$ to have its range match the null space of $Dc(x)$ and agree with $g(x)$ in its range. 
We can verify that $M(x)=Q(x)^{\top}g(x) Q(x)$, where $Q(x)$ is the symmetric matrix defined in (\ref{eq:Qx}), satisfies these two constraints.

%% file: 22algorithm.tex
\subsection{Efficient Computation of $\partial H/\partial x$ and $\partial H/\partial v$\label{subsec:Efficient}}

With $M(x)=Q(x)^{\top}g(x) Q(x)$, we have all the pieces of the algorithm. However, using this naive algorithm to compute $\partial H/\partial x$
and $\partial H/\partial v$, we face several challenges. 

\begin{enumerate}\setlength{\itemindent}{-2em}
	\item The algorithm involves computing the pseudo-inverse and its derivatives, which takes $O(n^{3})$ except for very special matrices. 
	\item The Lagrangian term in the constrained Hamiltonian dynamics requires additional computation such as the second-order derivative of $c(x)$. 
	\item A naive approach to computing leverage scores in $\partial H/\partial x$ results in a very dense matrix.
\end{enumerate}
Those challenges make the algorithm hard to implement and inefficient, especially when the dimension is high. In the following paragraphs, we give an overview of how we overcome each of the challenges above. We defer a more detailed discussion of our approaches and the proofs to Appendix~\ref{subsec:postponed_sec_efficient}. 

\paragraph{Avoiding Pseudo-inverse and Pseudo-determinant.}

We are able to show equivalent formulas for $M(x)^{\dagger}$ and $\log\pdet M(x)$ that can take advantage of sparse linear system solvers. In particular, we show that $M(x)^{\dagger}=g(x)^{-\frac{1}{2}}\cdot (I- P(x))\cdot g(x)^{-\frac{1}{2}}$, where 
\begin{equation}
\label{eq:formula_inverse}
P(x)=g(x)^{-\frac{1}{2}}\cdot Dc(x)^{\top}(Dc(x)\cdot g(x)^{-1}\cdot Dc(x)^{\top})^{-1}Dc(x)\cdot g(x)^{-\frac{1}{2}}.
\end{equation} As mentioned earlier, a majority of convex bodies appearing in practice
are of the form $K=\prod_{i}K_{i}$, where $K_{i}$ are constant dimensional
convex bodies. 
In this case, we will choose $g(x)$ to be a block diagonal matrix with each block of size $O(1)$. Hence, the bottleneck of applying $P(x)$ to a vector is simply solving a linear system of the form
$(Dc\cdot g^{-1}\cdot Dc^{\top})u=b$ for some $b$. The existing sparse linear system solvers can solve large classes of sparse
linear system much faster than $O(n^{3})$ time \cite{demmel1997applied}.
For $\log\pdet M(x)$, we show
\begin{equation}
\label{eq:formula_pdet}
\log\pdet(M(x))=\log\det g(x)+\log\det\left(Dc(x)\cdot g(x)^{-1}\cdot Dc(x)^{\top}\right)-\log\det\left(Dc(x)\cdot Dc(x)^{\top}\right).
\end{equation}
This simplification allows us to take advantage of sparse Cholesky decomposition.
 We prove \eqref{eq:formula_inverse} and \eqref{eq:formula_pdet} in Lemma~\ref{lem:formula_inverse} and Lemma~\ref{lem:formula_pdet} in Appendix~\ref{subsec:Avoiding-pseudo-inverse}. 
The formulas \eqref{eq:formula_inverse} and \eqref{eq:formula_pdet} avoid the expensive pseudo-inverse and pseudo-determinant computations, and significantly improve the practical performance of our algorithm. 

\paragraph{Simplification for Subspace Constraints.}
For the case $c(x)=Ax-b$, the Hamiltonian is now
$${H}(x,v)= f(x)+\frac{1}{2}v^{\top}g^{-\frac{1}{2}}\left(I-P\right)g^{-\frac{1}{2}}v+\frac{1}{2}\left(\log\det g+\log\det Ag^{-1}A^{\top}-\log\det AA^{\top}\right)+\lambda^{\top}c,$$
where $P=g^{-\frac{1}{2}}A^{\top}(Ag^{-1}A^{\top})^{-1}Ag^{-\frac{1}{2}}$.
 The key observation is that the algorithm
only needs to know $x(h)$ in the HMC dynamics, and not $v(h)$. Thus,
we can replace ${H}$ by any other that produces the same $x(h)$. 
We show in Lemma~\ref{lem:ham_deri} (Appendix~\ref{subsec:Simplification-for-subspace}) that the dynamics corresponding to ${H}$ above is equivalent to the dynamics that corresponds to a much simpler Hamiltonian:
\[
H(x,v)=f(x)+\frac{1}{2}v^{\top}g^{-\frac{1}{2}}\left(I-P\right)g^{-\frac{1}{2}}v+\frac{1}{2}\left(\log\det g+\log\det Ag^{-1}A^{\top}\right).\] 	
Furthermore, we have
\begin{align*}
	\frac{dx}{dt}  =g^{-\frac{1}{2}}\left(I-P\right)g^{-\frac{1}{2}}v,\;\;
	\frac{dv}{dt}  =-\nabla f(x)+\frac{1}{2}Dg\left[\frac{dx}{dt},\frac{dx}{dt}\right]-\frac{1}{2}\tr(g^{-\frac{1}{2}}\left(I-P\right)g^{-\frac{1}{2}}Dg).\end{align*}

\paragraph{Efficient Computation of Leverage Score.}

Even after simplifying the Hamiltonian
as above, we still have
a term for the leverage scores, $\tr(g^{-\frac{1}{2}}\left(I-P\right)g^{-\frac{1}{2}}Dg)$
in $\frac{dv}{dt}$ so that we need to compute the diagonal entries of $P=g^{-\frac{1}{2}}A^{\top}(Ag^{-1}A^{\top})^{-1}Ag^{-\frac{1}{2}}$
to compute $\frac{dv}{dt}$. Since $(Ag^{-1}A^{\top})^{-1}$
can be extremely dense even when $A$ is very sparse, a naive approach
such as direct computation of the inverse can lead to a dense-matrix multiplication. To avoid dense-matrix multiplication, our approach is based on
the fact that certain entries of $(Ag^{-1}A^{\top})^{-1}$ can be
computed as fast as computing sparse Cholesky decomposition of $Ag^{-1}A^{\top}$
\cite{takahashi1973formation,campbell1995computing}, which can be
$O(n)$ time faster than computing $(Ag^{-1}A^{\top})^{-1}$ in many
settings. We first compute the Cholesky decomposition
to obtain a sparse triangular matrix $L$ such that $LL^{\top}=Ag^{-1}A^{\top}$. Then, we show that only entries of $Ag^{-1}A^{\top}$ in
$\sp(L)\cup\sp(L^{\top})$ matter in computing $\diag(A^{\top}(Ag^{-1}A^{\top})^{-1}A)$, where $\sp(L)$ is the sparsity pattern of $L$. We give the details of our approach in Appendix~\ref{subsec:Efficient-Computation-of}.

%% file: 23algorithm.tex
\subsection{Discretization\label{subsec:Discretization}}

Explicit integrators such as leapfrog integrator, which are commonly used for Hamiltonian Monte Carlo, are no longer symplectic on general Riemannian manifolds (see Appendix~\ref{subsec:Discretized-CRHMC-based}).
Even though there have been some attempts \cite{pihajoki2015explicit} to make explicit integrators work in the Riemannian setting, its variants do not work for ill-conditioned problems.

Our algorithm uses the \emph{implicit midpoint method} (Algorithm~\ref{alg:IMM}) to discretize the Hamiltonian process into steps
of step size $h$ and run the process for $T$ iterations. 
This integrator is reversible and symplectic (so measure-preserving) \cite{hairer2006geometric}, which allows us to use a Metropolis filter to ensure the distribution is correct so that we no longer need to solve ODE to accuracy to maintain the correct stationary distribution.
We write  $H(x,v)=\overline{H}_{1}(x,v)+\overline{H}_{2}(x,v)$,
where
\begin{align*}
	\overline{H}_{1}(x,v) & =f(x)+\frac{1}{2}\left(\log\det g(x)+\log\det Ag(x)^{-1}A^{\top}\right),\\
	\overline{H}_{2}(x,v) & =\frac{1}{2}v^{\top}g(x)^{-\frac{1}{2}}\left(I-P(x)\right)g(x)^{-\frac{1}{2}}v.
\end{align*}
Starting from $(x_{0},v_{0})$, in the first step of the integrator,
we run the process on the Hamiltonian $\overline{H}_{1}$ with step
size $\frac{h}{2}$ to get $(x_{1/3},v_{1/3})$. In the second step of the integrator, we run the process on $\overline{H}_{2}$
with step size $h$ by solving
\begin{align*}
	x_{\frac{2}{3}}  =x_{\frac{1}{3}}+h\frac{\partial\overline{H}_{2}}{\partial v}\left(\frac{x_{\frac{1}{3}}+x_{\frac{2}{3}}}{2},\frac{v_{\frac{1}{3}}+v_{\frac{2}{3}}}{2}\right),\ \;\;\;\;
	v_{\frac{2}{3}}  =v_{\frac{1}{3}}-h\frac{\partial\overline{H}_{2}}{\partial x}\left(\frac{x_{\frac{1}{3}}+x_{\frac{2}{3}}}{2},\frac{v_{\frac{1}{3}}+v_{\frac{2}{3}}}{2}\right),
\end{align*}
iteratively using the Newton's method. 
This step involves computing the Cholesky decomposition of $(Ag^{-1}A^\top)^{-1}$ using the Cholesky decomposition of $Ag^{-1}A^\top$. 
In the third step, we run the process on the Hamiltonian $\overline{H}_{1}$ with step size $\frac{h}{2}$ again to get $(x_{1},v_{1})$. 

We state the complete algorithm (Algorithm~\ref{alg:CHMC_JL} and Algorithm~\ref{alg:IMM}) with details on the step size in Appendix~\ref{subsec:Discretized-CRHMC-based} and give the theoretical guarantees in Appendix~\ref{subsec:Theoretical-Guarantees} (convergence of implicit midpoint method)
and Appendix~\ref{subsec:Choosing} (independence of condition number).

%% file: 30expts.tex
\section{Experiments}
\label{sec:Experiments}

In this section, we demonstrate the efficiency of our sampler using experiments on real-world datasets and compare our sampler with existing samplers.
We demonstrate that CRHMC is able to sample larger models than previously known to be possible, and is significantly faster in terms of rate of convergence and sampling time in Section~\ref{subsec:Sub-linear-Mixing-Time}, along with convergence test in Section~\ref{subsec:unif-test}.
We examine its behavior on benchmark instances such as simplices and Birkhoff polytopes in Section~\ref{subsec:CRHMC-on-Benchmark}.

\subsection{Experimental Setting\label{subsec:exp_setting}}

\paragraph*{Settings.}

We performed experiments on the Standard DS12 v2 model from MS Azure cloud, which has a 2.1GHz Intel Xeon Platinum 8171M CPU and 28GB memory.
In the experiments, we used our MATLAB and C++ implementation of CRHMC\footnote{Our package can be run to sample from general logconcave densities and has a feature for parallelization.}, which is available \href{https://github.com/ConstrainedSampler/PolytopeSamplerMatlab}{here} and has been integrated into the COBRA toolbox.

We used twelve constraint-based metabolic models from molecular systems biology in the \href{https://github.com/opencobra/cobratoolbox}{COBRA Toolbox v3.0}~\cite{heirendt2019creation}
and ten real-world LP examples randomly chosen from \href{http://www.netlib.org/lp/data/}{NETLIB LP test sets}. 
A polytope from each model is defined by 
$
\{x\in\R^{n}:Ax=b,\,l\leq x\leq u\}
$
for $A\in\R^{m\times n},b\in\R^{m},$ and $l,u\in\R^{n}$, which is
input to CRHMC for uniform sampling. 
We describe in Appendix~\ref{subsec:expDetails} how we preprocessed these dataset, along with full information about the datasets in Table~\ref{tab:data}.

\paragraph*{Comparison.}

We used as a baseline the Coordinate Hit-and-Run (CHAR) implemented in two different languages. The former is Coordinate Hit-and-Run with Rounding (CHRR) written in MATLAB~\cite{cousins2016practical,haraldsdottir2017chrr}
and the latter is the same algorithm (CDHR) with an R interface and a C++ library, VolEsti~\cite{chalkis2020volesti}. 
We refer readers to Appendix~\ref{subsec:expDetails} for the details of these algorithms and our comparison setup.
We note that popular sampling packages such as STAN and Pyro were not included in the experiments as they do not support constrained-based models. Even after transforming our dataset to their formats, the transformed dataset were too ill-conditioned for those algorithms to run.
{CHMC in \cite{brubaker2012family} works only for manifolds implicitly defined by $\{c(x)=0\}$ for continuously differentiable $c(x)$ with $Dc(x)$ full-rank everywhere, so we could not use it for comparison.}

\begin{figure*}
\begin{minipage}{0.47\textwidth}
\begin{centering}
\includegraphics[scale=0.3]{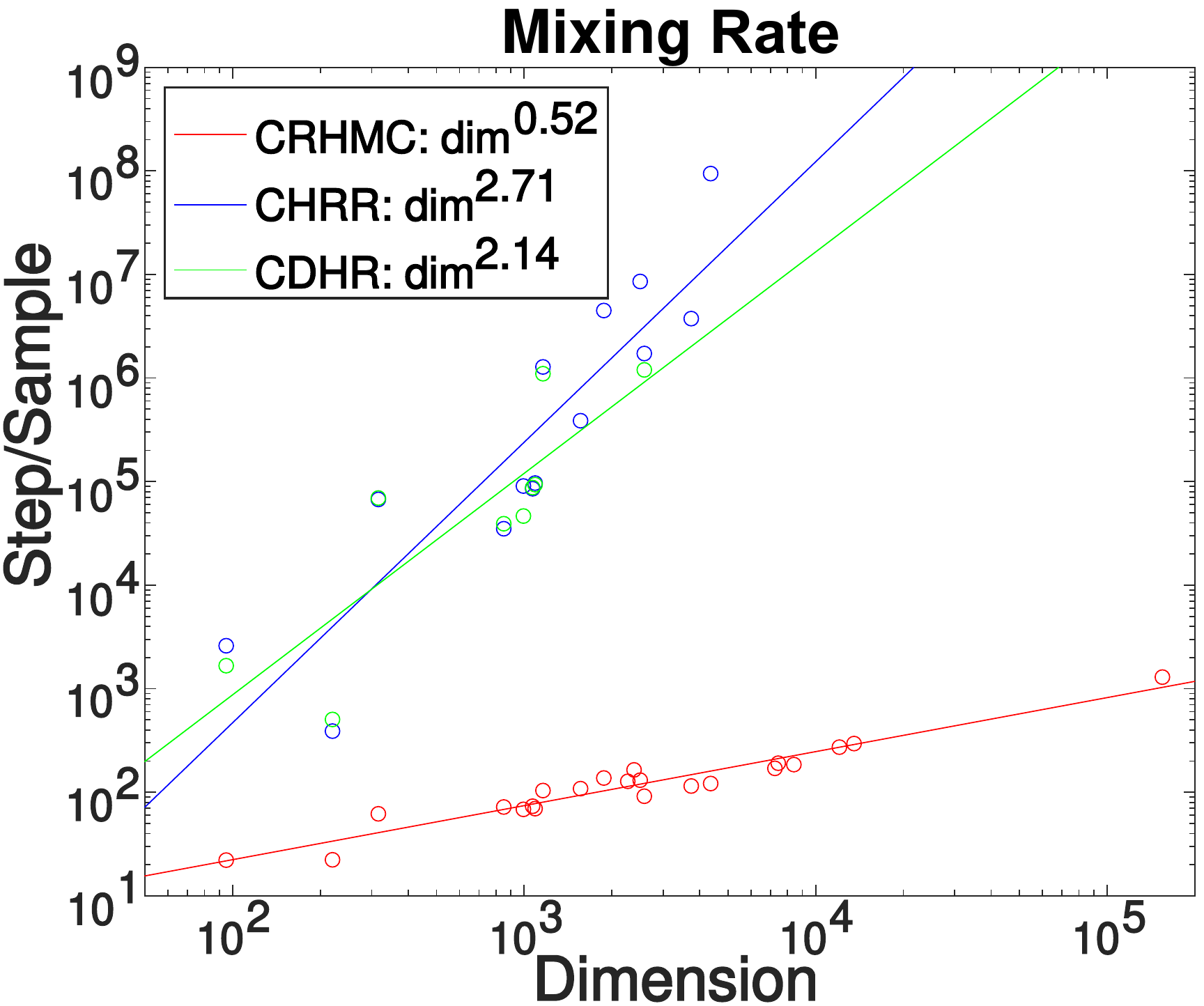} \\$\ $\includegraphics[scale=0.3]{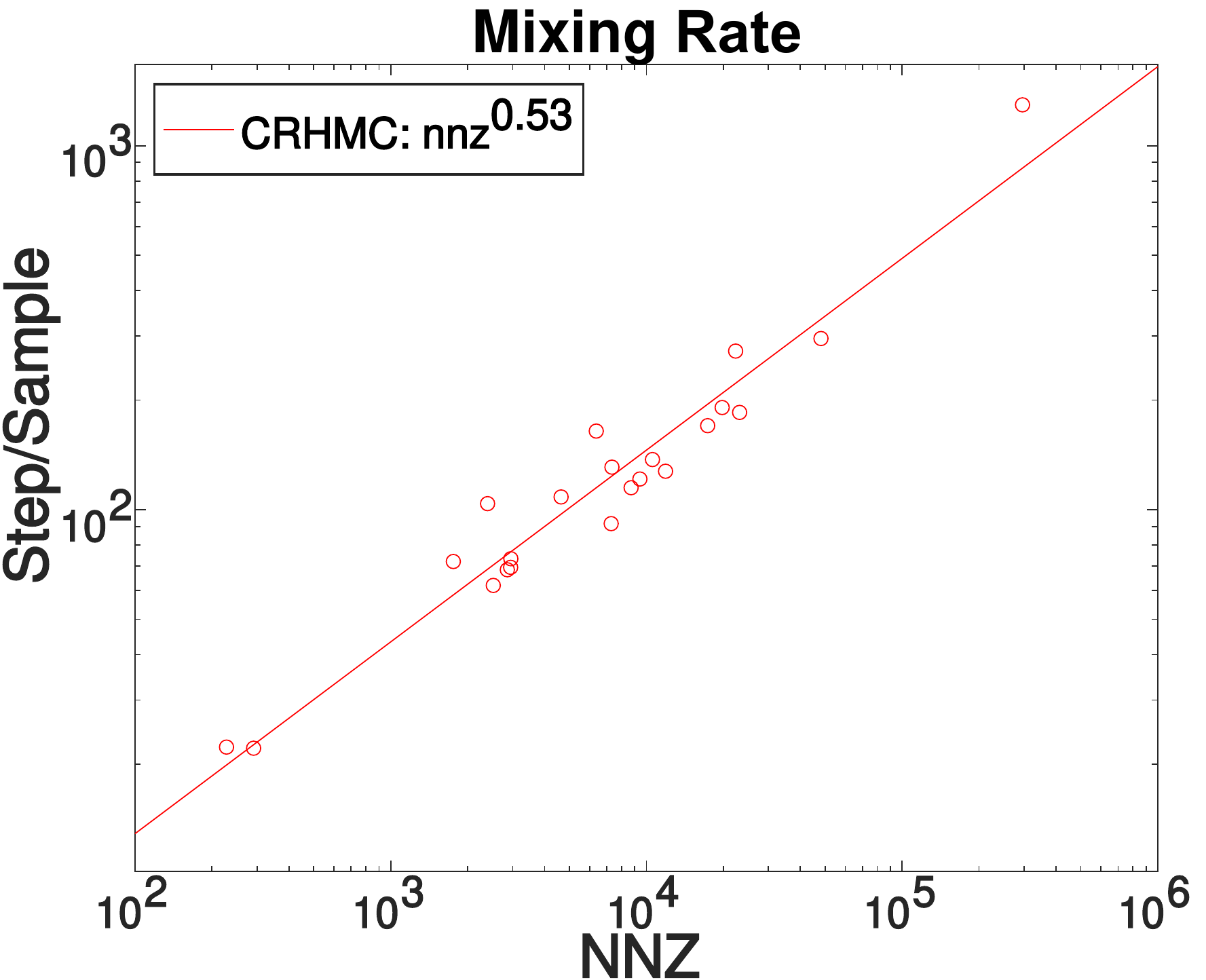}
\par\end{centering}
\caption{Mixing rate of CRHMC and the competitors. Mixing rate of CRHMC was
sub-linear in dimension and the nnz of a preprocessed matrix A in
a model, whereas the others needed quadratically many steps to converge
to uniform distribution. In particular for our dataset, CRHMC mixed
up to 6 orders of magnitude earlier than the others. Note that mixing
rate of CHAR was very close to quadratic growth when using the full-dimensional
scale (the first column in Table~\ref{tab:data}).}
\label{fig:mixing}
\vspace{-3mm}
\end{minipage}
\hfill
\begin{minipage}{0.52\textwidth}

\begin{centering}
		\includegraphics[scale=0.3]{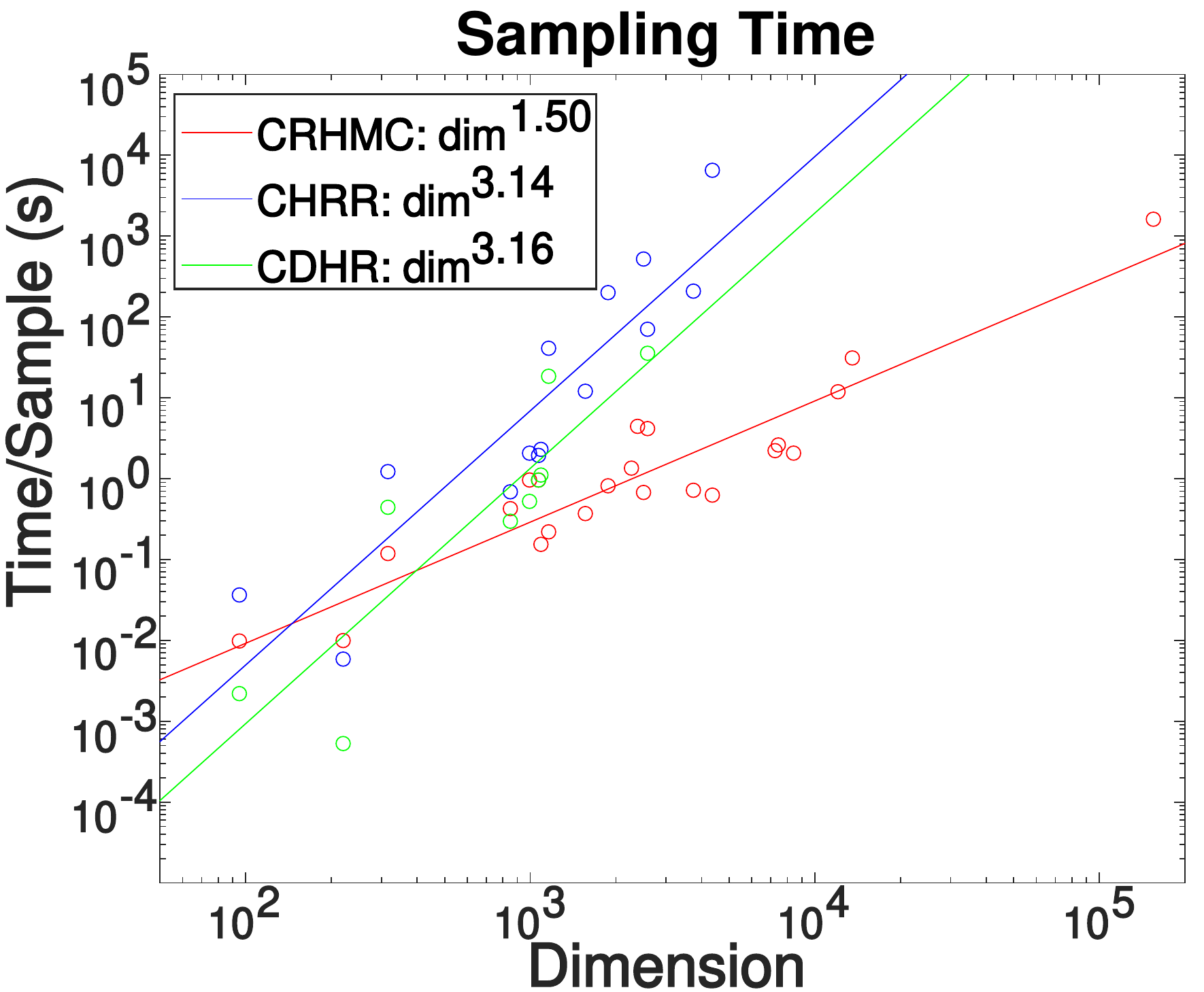}\\ \includegraphics[scale=0.3]{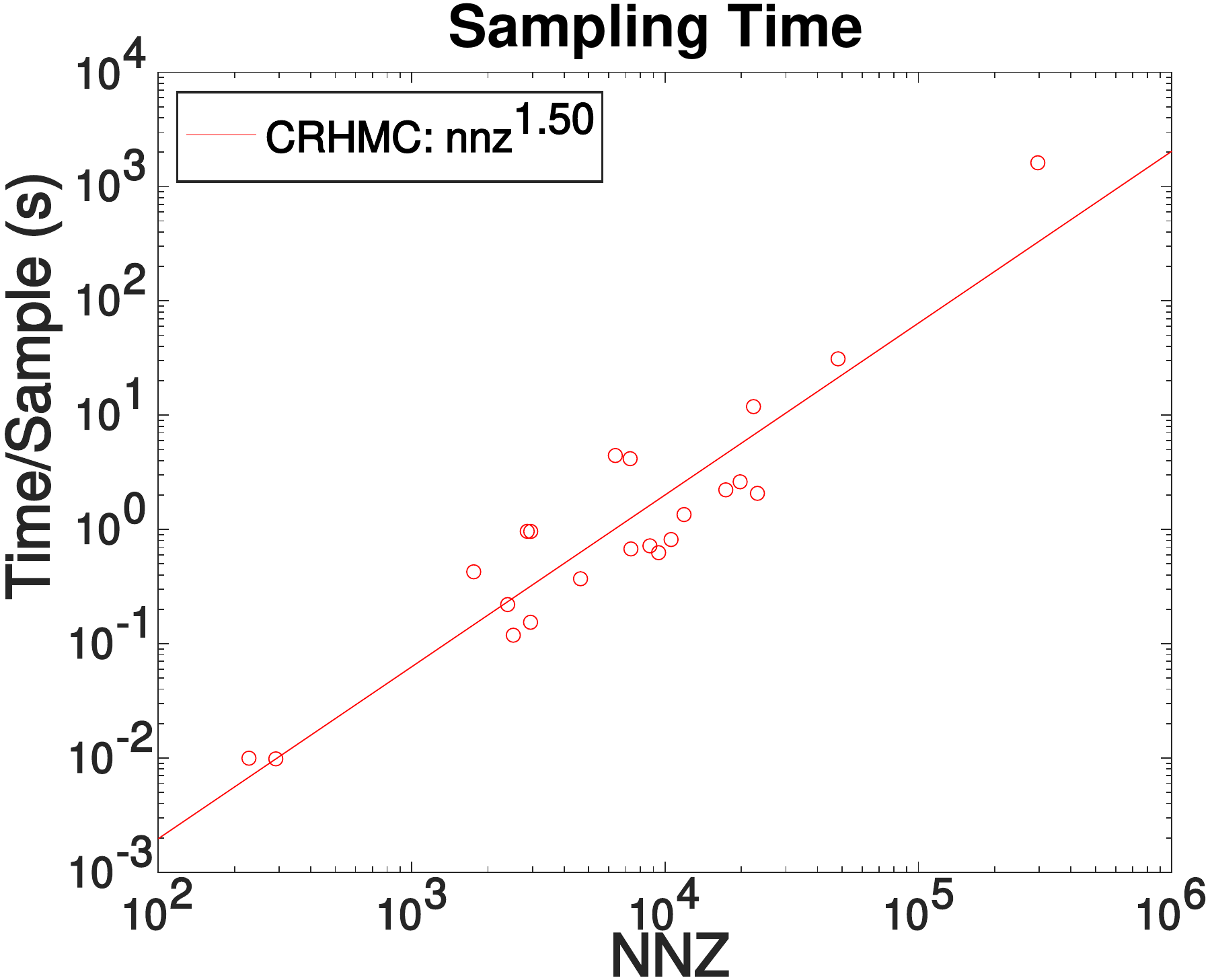}
\par\end{centering}
	\caption{Sampling time of CRHMC and the competitors. The sampling time per effective sample of CRHMC was sub-quadratic in dimension and the nnz of a preprocessed matrix A in a model, while the others indicates at least a cubic dependency on dimension. In particular for our dataset, CRHMC was able to obtain a statistically independent sample up to 4 orders of magnitude faster than the others. This benefit of speed-up was actually straightforward from the figure, since CHRR could not obtain enough samples from instances with more than 5000 variables until it ran out of time.}
	\label{fig:sample}
	\vspace{-3mm}
\end{minipage}
\end{figure*}

\paragraph*{Measurements.}

To evaluate the quality of sampling methods, we measured two quantities,
the \emph{number of steps per effective sample} (i.e., mixing
rate) and the\emph{ sampling time per effective sample, }$T_{s}$.
The \emph{effective sample size} (ESS)\footnote{We use the minimum of the ESS of each coordinate.} can be thought of as the number
of actual independent samples, taking into account correlation of
samples from a target distribution. 
Thus the number of steps per effective sample is estimated by the total number of steps divided by the ESS, and the sampling time $T_{s}$ is estimated as the total sampling time until termination divided by the ESS.

Each algorithm attempted to draw 1000 uniform samples, with limits on running
time set to 1 day (3 days for the largest instance \emph{ken\_18})
and memory usage to 6GB. If an algorithm passes either the time or
the memory limit, we stop the algorithm and measure the quantities
of interest based on samples drawn until that moment. After getting
uniform samples, we thinned the samples twice to ensure independence
of samples; first we computed the ESS of the samples, only kept ESS
many samples, and repeated this again. 
We estimated the above quantities only if the ESS
is more than 10 and an algorithm does not run into any error while
running\footnote{When running CDHR from the VolEsti package on some instances, we got an error message ``R session aborted and R encountered a fatal error''.}.

\subsection{ Mixing Rate and Sampling Time\label{subsec:Sub-linear-Mixing-Time}}

\paragraph{Sub-linear Mixing Rate.} We examined how the number of steps per effective sample grows with the number of nonzeros (nnz) of matrix $A$ (after preprocessing)
and the number of variables (dimension in the plots). To this end,
we counted the total number of steps taken until termination of algorithms
and divided it by the effective sample size of drawn samples. Note
that we thinned twice to ensure independence of samples used.

The mixing rate of CRHMC was sub-linear in both dimension and nnz,
whereas previous implementations based on CHAR required at least $n^{2}$
steps per sample as seen in Figure \ref{fig:mixing}. On the dataset,
mixing rate attained was up to 6 orders of magnitude faster for CRHMC
compared to CHAR, implying that CRHMC converged to uniform distribution
substantially faster than the other competitors. This gap in mixing
rate increased super-linearly in dimension, enabling CRHMC to run
on large instances of dimension up to 100000.

\begin{table}[t]
\renewcommand{\arraystretch}{2}
\begin{center}
\resizebox{\textwidth}{!}{
\begin{raggedright}
{\footnotesize{}}%
\begin{tabular}{>{\centering}p{2.7cm}||>{\centering}p{2.3cm}|>{\centering}p{2cm}||>{\centering}p{2.5cm}|>{\centering}p{2.5cm}|>{\centering}p{2.5cm}}
{\LARGE{}Bio Model} & {\LARGE{}Vars ($n$)} & {\LARGE{}nnz} & {\LARGE{}CRHMC} & {\LARGE{}CHRR} & {\LARGE{}CDHR}\tabularnewline
\hline 
\hline 
{\LARGE{}ecoli} & {\LARGE{}95} & {\LARGE{}291} & {\LARGE{}0.0098} & {\LARGE{}0.0365} & {\LARGE{}0.0022}\tabularnewline
\hline 
{\LARGE{}cardiac\_mit} & {\LARGE{}220} & {\LARGE{}228} & {\LARGE{}0.0100} & {\LARGE{}0.0059} & {\LARGE{}0.0005}\tabularnewline
\hline 
{\LARGE{}Aci\_D21} & {\LARGE{}851} & {\LARGE{}1758} & {\LARGE{}0.4257} & {\LARGE{}0.6884} & {\LARGE{}0.2974}\tabularnewline
\hline 
{\LARGE{}Aci\_MR95} & {\LARGE{}994} & {\LARGE{}2859} & {\LARGE{}0.9624} & {\LARGE{}2.0668} & {\LARGE{}0.5237}\tabularnewline
\hline 
{\LARGE{}Abi\_49176} & {\LARGE{}1069} & {\LARGE{}2951} & {\LARGE{}0.9608} & {\LARGE{}1.9395} & {\LARGE{}0.9622}\tabularnewline
\hline 
{\LARGE{}Aci\_20731} & {\LARGE{}1090} & {\LARGE{}2946} & {\LARGE{}0.1540} & {\LARGE{}2.3014} & {\LARGE{}1.1086}\tabularnewline
\hline 
{\LARGE{}Aci\_PHEA} & {\LARGE{}1561} & {\LARGE{}4640} & {\LARGE{}0.3701} & {\LARGE{}12.06} & {\LARGE{}-}\tabularnewline
\hline 
{\LARGE{}iAF1260} & {\LARGE{}2382} & {\LARGE{}6368} & {\LARGE{}4.4355} & {\LARGE{}3687.2} & {\LARGE{}-}\tabularnewline
\hline 
{\LARGE{}iJO1366} & {\LARGE{}2583} & {\LARGE{}7284} & {\LARGE{}4.1608} & {\LARGE{}70.5} & {\LARGE{}35.556}\tabularnewline
\hline 
{\LARGE{}Recon1} & {\LARGE{}3742} & {\LARGE{}8717} & {\LARGE{}0.7184} & {\LARGE{}208.5} & {\LARGE{}-}\tabularnewline
\hline 
{\LARGE{}Recon2} & {\LARGE{}7440} & {\LARGE{}19791} & {\LARGE{}2.6116} & {\LARGE{}10445{*}} & {\LARGE{}-}\tabularnewline
\hline 
{\LARGE{}Recon3} & {\LARGE{}13543} & {\LARGE{}48187} & {\LARGE{}31.114} & {\LARGE{}29211{*}} & {\LARGE{}-}\tabularnewline
\hline
\end{tabular}{\footnotesize{} }{\footnotesize\par}
\begin{tabular}{>{\centering}p{2.7cm}||>{\centering}p{2.3cm}|>{\centering}p{2cm}||>{\centering}p{2.5cm}|>{\centering}p{2.5cm}|>{\centering}p{2.5cm}}
{\LARGE{}LP Model} & {\LARGE{}Vars ($n$)} & {\LARGE{}nnz} & {\LARGE{}CRHMC} & {\LARGE{}CHRR} & {\LARGE{}CDHR}\tabularnewline
\hline 
\hline 
{\LARGE{}israel} & {\LARGE{}316} & {\LARGE{}2519} & {\LARGE{}0.1186} & {\LARGE{}1.2224} & {\LARGE{}0.4426}\tabularnewline
\hline 
{\LARGE{}gfrd\_pnc} & {\LARGE{}1160} & {\LARGE{}2393} & {\LARGE{}0.2199} & {\LARGE{}40.988} & {\LARGE{}18.468}\tabularnewline
\hline 
{\LARGE{}25fv47} & {\LARGE{}1876} & {\LARGE{}10566} & {\LARGE{}0.8159} & {\LARGE{}199.9} & {\LARGE{}-}\tabularnewline
\hline 
{\LARGE{}pilot\_ja} & {\LARGE{}2267} & {\LARGE{}11886} & {\LARGE{}1.3490} & {\LARGE{}5059{*}} & {\LARGE{}-}\tabularnewline
\hline 
{\LARGE{}sctap2} & {\LARGE{}2500} & {\LARGE{}7334} & {\LARGE{}0.6752} & {\LARGE{}520.2} & {\LARGE{}-}\tabularnewline
\hline 
{\LARGE{}ship08l} & {\LARGE{}4363} & {\LARGE{}9434} & {\LARGE{}0.6258} & {\LARGE{}6512} & {\LARGE{}-}\tabularnewline
\hline 
{\LARGE{}cre\_a} & {\LARGE{}7248} & {\LARGE{}17368} & {\LARGE{}2.2205} & {\LARGE{}30455{*}} & {\LARGE{}-}\tabularnewline
\hline 
{\LARGE{}woodw} & {\LARGE{}8418} & {\LARGE{}23158} & {\LARGE{}2.0689} & {\LARGE{}30307{*}} & {\LARGE{}-}\tabularnewline
\hline 
{\LARGE{}80bau3b} & {\LARGE{}12061} & {\LARGE{}22341} & {\LARGE{}11.881} & {\LARGE{}47432{*}} & {\LARGE{}-}\tabularnewline
\hline 
{\LARGE{}ken\_18} & {\LARGE{}154699} & {\LARGE{}295946} & {\LARGE{}1616.3} & {\LARGE{}-} & {\LARGE{}-}\tabularnewline
\hline 
\end{tabular}{\footnotesize{} }{\footnotesize\par}
\par\end{raggedright}
\raggedright{}}
\end{center}\caption{Sampling time per effective sample of CHRR and CRHMC. We note that
CRHMC is 1000 times faster than CHRR on the latest metabolic network
(Recon3). Sampling time with asterisk ({*}) indicates that the effective
sample size is less than 10.}
\label{tab:sampleTime}
\vspace{-6mm}
\end{table}

\paragraph{Sub-quadratic Sampling Time.} We next examined the sampling time $T_{s}$ in terms of both the nnz
of $A$ and the dimension of the instance. We computed the runtime
of algorithms until their termination divided by the effective sample
size of drawn samples, where we ignored the time it takes for preprocessing.
Note that the sampling time $T_{s}$ is essentially multiplication
of the mixing rate and the \emph{per-step complexity }(i.e., how much
time each step takes).

As shown in Figure~\ref{fig:sample} and Table~\ref{tab:sampleTime},
we found that the per-step complexity of CRHMC was small enough to make the
sampling time sub-quadratic in both dimension and nnz, whereas CHAR had at least a cubic dependency on dimension,
despite of a low per-step complexity. On our dataset, the sampling
time of CRHMC was up to 4 orders of magnitude less than that of CHRR
and CDHR. While CHRR can be used on dimension only up to a few thousands,
increasing benefits of sampling time in higher dimension allows CRHMC
to run on dimension up to 0.1 million.

\subsection{CRHMC on Structured Instances\label{subsec:CRHMC-on-Benchmark}}

To see the behavior of CRHMC on very large instances, we ran the algorithm on three families of structured polytopes -- hypercube, simplex, and Birkhoff polytope -- up to dimension half-million. 
We attempted to draw 500 uniform samples with a 1 day time limit (except for 2 days for half-million-dimensional Birkhoff polytope). 
The definitions of these polytopes are shown in Appendix~\ref{subsec:polydef}.

\begin{figure}[h]
\begin{centering}
\vspace{-1mm}
\includegraphics[scale=0.3]{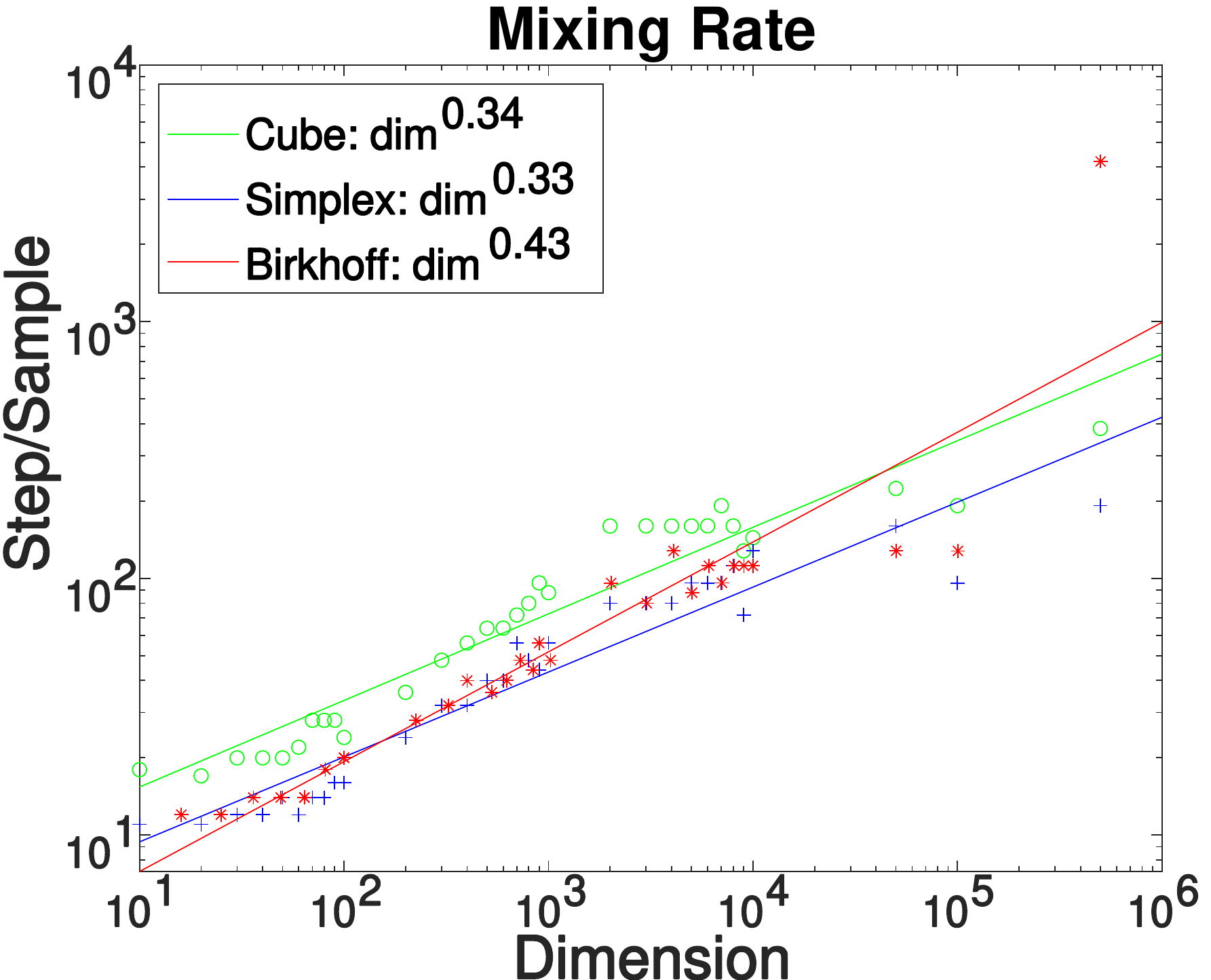}$\ $\includegraphics[scale=0.3]{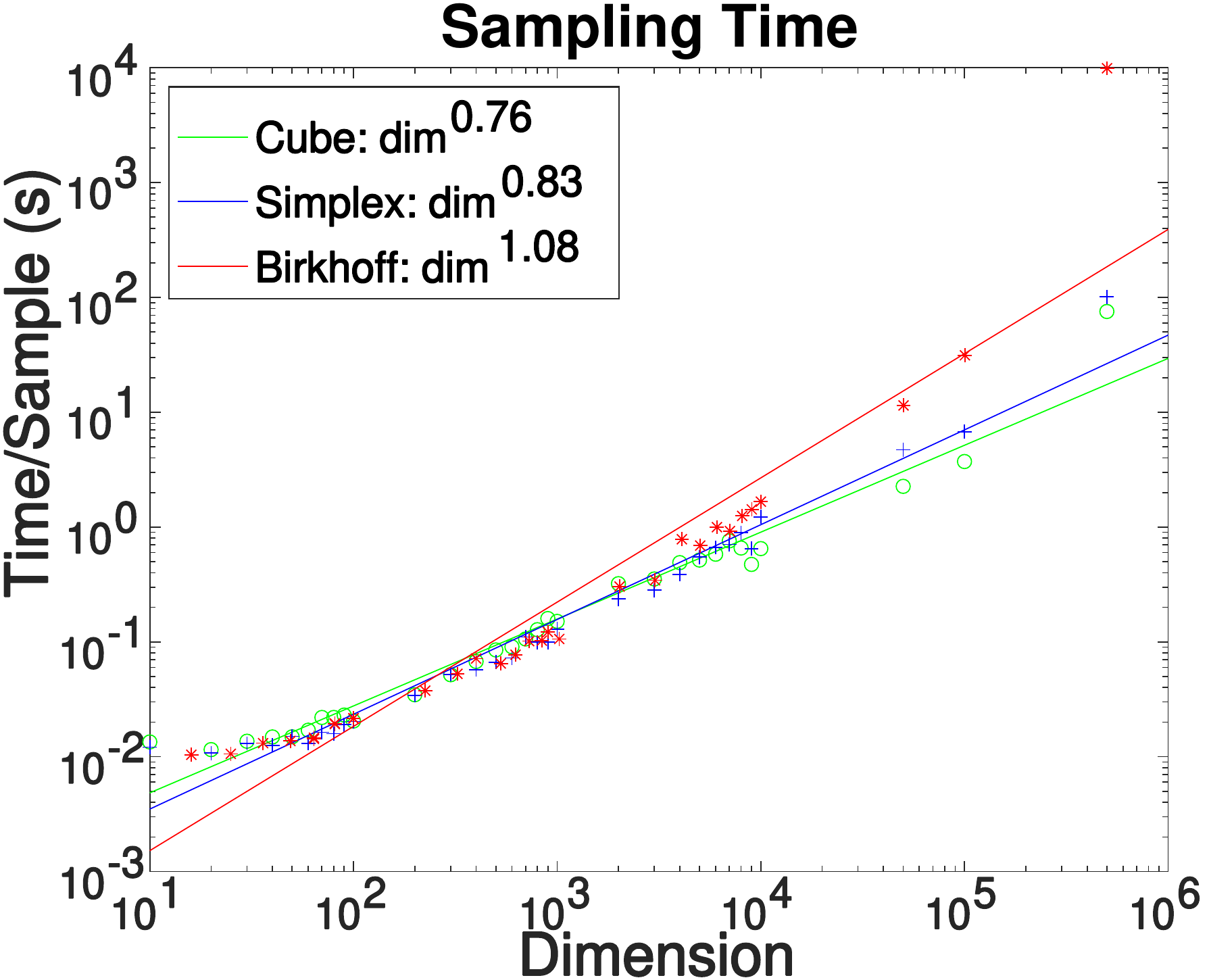}
\par\end{centering}
\vspace{-1mm}
\caption{Mixing rate and sampling time on structured polytopes including hybercubes,
simplices, and Birkhoff polytopes. CRHMC is scalable up to 0.5 million
dimension on hypercubes and simplices and up to 0.1 million dimension
on Birkhoff polytopes. We note that on the 0.5 million dimensional
Birkhoff polytope the ESS is only $16$, which is not reliable compared
to the ESS on the other instances.}
\label{fig:bench}
\vspace{-1mm}
\end{figure}

To the best of our knowledge, this is the first demonstration that
it is possible to sample such a large model. As seen in Figure~\ref{fig:bench},
CRHMC can scale smoothly up to half-million dimension on hypercubes
and simplices and up to dimension $10^{5}$ for Birkhoff polytopes
(we could not obtain a reliable estimate of mixing rate and sampling
time on the half-million dimensional Birkhoff polytope, as the ESS
is only $16$ after 2 days). However, we believe that one can find
room for further improvement of CRHMC by tuning parameters or leveraging
engineering techniques. We also expect that CRHMC enables us to estimate
the volume of \textbf{$B_{n}$ }for $n\geq20$, going well beyond the previously best possible dimension.

\subsection{Uniformity Test\label{subsec:unif-test}}

We used the following uniformity test to check whether samples from CRHMC form the uniform distribution over a polytope $P$: check that the fraction of the samples in the scaled set $x\cdot P$ is proportional to $x^{\textrm{dim}}$.
As seen in Figure~\ref{fig:unif_test}, the empirical CDFs of the radial distribution to the power of (1/dim) are close to the CDFs of the uniform distribution over those polytopes. 

\begin{figure}[h]
\begin{centering}
\vspace{-1mm}
\includegraphics[scale=0.3]{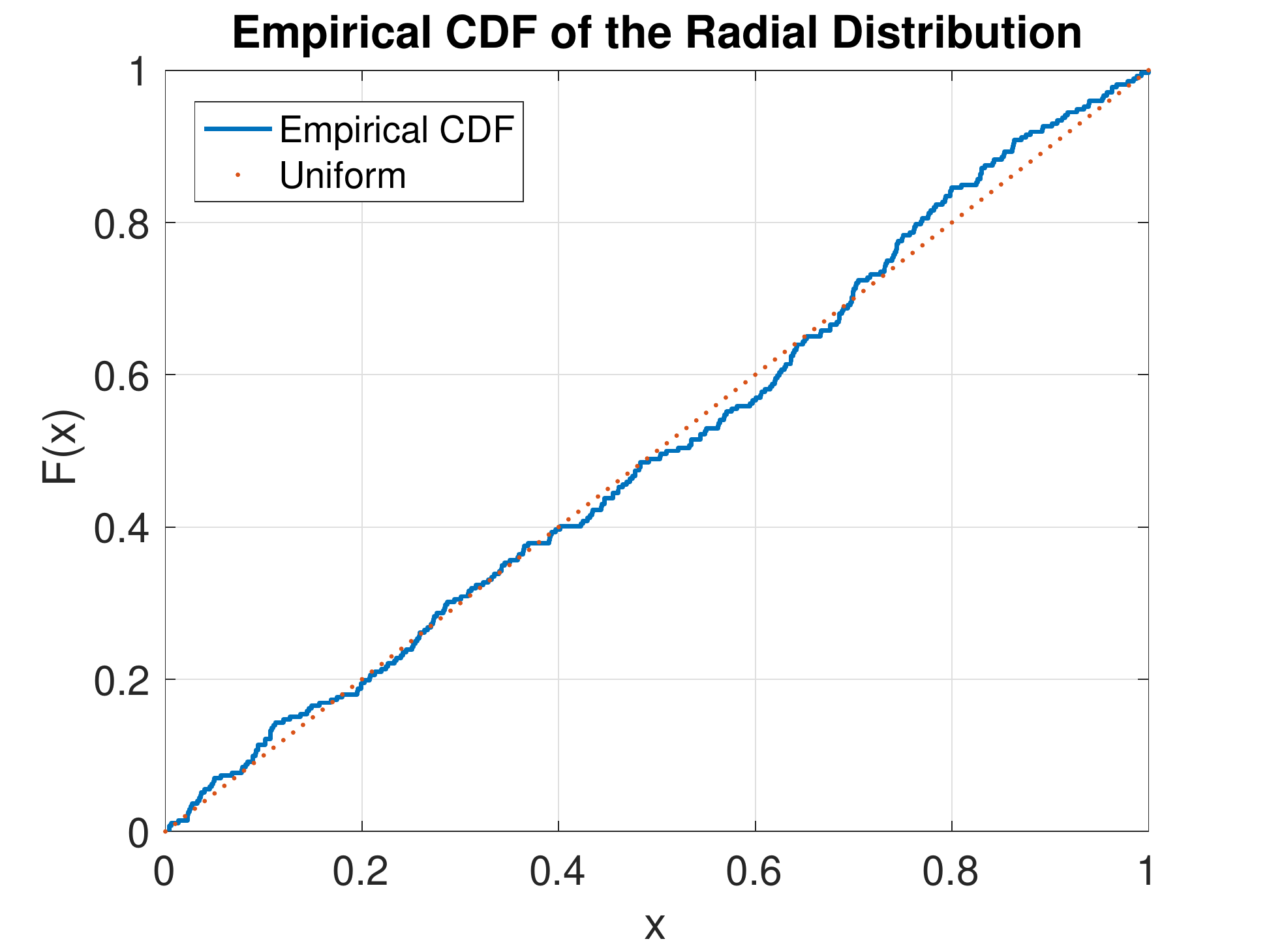}$\ $\includegraphics[scale=0.3]{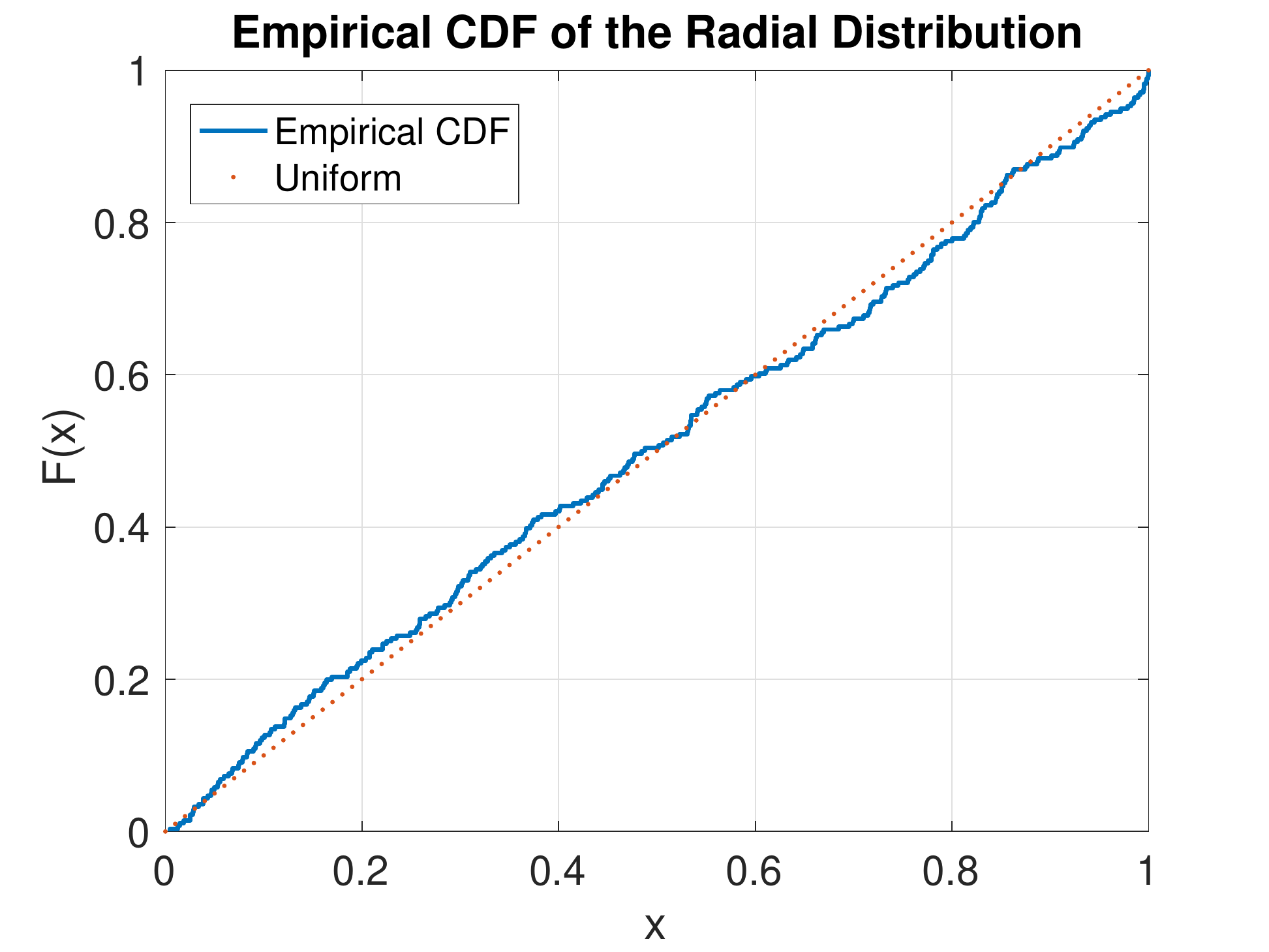}
\par\end{centering}
\vspace{-1mm}
\caption{We plot the empirical cumulative distribution function of the radial distribution to the power of $(1/\textrm{dim})$ with 1000 ESS obtained by running CRHMC on \emph{ATCC-49176} ($952 \times 1069$, left) and \emph{Aci-PHEA} ($1319\times 1561$, right), and in the plot $x$-axis is the scaling factor. 
We can observe the CDFs are very close to the CDFs of the uniform distribution over the polytopes defined by two instances.}
\label{fig:unif_test}
\vspace{-5mm}
\end{figure}

%% file: 93appendix_exp.tex
\newpage
\section{Additional Experiment Details\label{subsec:expDetails}}

\begin{figure}[t]  \label{app:outline}
\begin{tikzpicture}[>=stealth,every node/.style={shape=rectangle,draw,rounded corners, minimum width=4.5cm,},]
    \node[very thick, fill=orange!30, minimum width=3.1cm] (t1) {
    \begin{tabular}{c}
    	Experiments details\\
    	(App~\ref{subsec:expDetails})
    \end{tabular}};
    \node[very thick, fill=orange!30, minimum width=3.4cm] (l2) [below=of t1]{Theory};
    \node[very thick, fill=orange!30, minimum width=3.1cm] (l1) [below=of l2]{
    \begin{tabular}{c}
    	Notation, Definition\\
    	(App~\ref{sec:Missing-Definitions-and})
    \end{tabular}};
    \node[very thick, fill=blue!30, minimum width=3.3cm] (l4) [right=of t1]{
    \begin{tabular}{c}
      Continuous \\CRHMC \\ (App~\ref{subsec:contiCRHMC})
    \end{tabular}};
    \node[very thick, fill=blue!30, minimum width=3.3cm] (l5) [right=of l1]{
    \begin{tabular}{c}
      Discretized \\CRHMC \\ (App~\ref{subsec:Discretization})
    \end{tabular}};
    \node[very thick] (l7) [right=of l4]{App~\ref{subsec:postponed_sec_efficient}: Computational tricks};
    \node[very thick] (l6) [above=0.5cm of l7]{App~\ref{subsec:derivationH}: Derivation} ;
    \node[very thick] (l8) [below=0.5cm of l7]{App~\ref{subsec:Stationarity-of-CRHMC}: Correctness};

    \node[very thick] (l10) [right=of l5]{App~\ref{subsec:Discretized-CRHMC-based}: Implicit midpoint method};
    \node[very thick] (l11) [below=0.5cm of l10]{App~\ref{subsec:Theoretical-Guarantees}: Correctness \& Efficiency};
    \node[very thick] (l9) [above=0.5cm of l10]{
    \begin{tabular}{c}
    	App~\ref{subsec:Choosing}: Condition \# independence\\ 
    	(Mixing rate)
    \end{tabular}};

    \draw[->,  line width=.6mm] (l2) to[out=0,in=180] (l4);
    \draw[->,  line width=.6mm] (l2) to[out=0,in=180] (l5);

    \draw[->,  line width=.6mm] (l4) to[out=0,in=180] (l6);
    \draw[->,  line width=.6mm] (l4) to[out=0,in=180] (l7);
    \draw[->,  line width=.6mm] (l4) to[out=0,in=180] (l8);
    \draw[->,  line width=.6mm] (l4) to[out=0,in=180] (l9);

    \draw[->,  line width=.6mm] (l5) to[out=0,in=180] (l9);
    \draw[->,  line width=.6mm] (l5) to[out=0,in=180] (l10);
    \draw[->,  line width=.6mm] (l5) to[out=0,in=180] (l11);
\end{tikzpicture}
\end{figure}

\paragraph*{Dataset.}

We summarize in Table~\ref{tab:data} the dataset used in experiments.
If a model is unbounded, we make it bounded by setting $l=\max(l,-10^{7})$ and $u=\min(u,10^{7})$.
As existing packages require full-dimensional representations of polytopes
(i.e., $\{x:A'x\leq b'\}$), we transformed all constraint-based models
to prepare instances for them as follows: (1) first preprocess each
model by removing redundant constraints and appropriately scaling
it, (2) find its corresponding full-dimensional description, and (3)
round it via the maximum volume ellipsoid (MVE) algorithm making the
polytope more amenable to sampling. We note that a full-dimensional
polytope can be transformed into a constraint-based polytope and vice
versa, so CRHMC can be run on either representation.

\begin{table}[h]
	\begin{center}
	\resizebox{\textwidth}{!}{
	\begin{raggedright}
	{\footnotesize{}}
	\begin{tabular}{>{\centering}p{1.5cm}||>{\centering}p{1.2cm}|>{\centering}p{1.5cm}|>{\centering}p{1.2cm}|>{\centering}p{0.8cm}}
		{\footnotesize{}Bio Model} & {\footnotesize{}Full-dim} & {\footnotesize{}Consts ($m$)} & {\footnotesize{}Vars ($n$)} & {\footnotesize{}nnz}\tabularnewline
		\hline 
		\hline 
		{\footnotesize{}ecoli} & {\footnotesize{}24} & {\footnotesize{}72} & {\footnotesize{}95} & {\footnotesize{}291}\tabularnewline
		\hline 
		{\footnotesize{}cardiac\_mit} & {\footnotesize{}12} & {\footnotesize{}230} & {\footnotesize{}220} & {\footnotesize{}228}\tabularnewline
		\hline 
		{\footnotesize{}Aci\_D21} & {\footnotesize{}103} & {\footnotesize{}856} & {\footnotesize{}851} & {\footnotesize{}1758}\tabularnewline
		\hline 
		{\footnotesize{}Aci\_MR95} & {\footnotesize{}123} & {\footnotesize{}917} & {\footnotesize{}994} & {\footnotesize{}2859}\tabularnewline
		\hline 
		{\footnotesize{}Abi\_49176} & {\footnotesize{}157} & {\footnotesize{}952} & {\footnotesize{}1069} & {\footnotesize{}2951}\tabularnewline
		\hline 
		{\footnotesize{}Aci\_20731} & {\footnotesize{}164} & {\footnotesize{}1009} & {\footnotesize{}1090} & {\footnotesize{}2946}\tabularnewline
		\hline 
		{\footnotesize{}Aci\_PHEA} & {\footnotesize{}328} & {\footnotesize{}1319} & {\footnotesize{}1561} & {\footnotesize{}4640}\tabularnewline
		\hline 
		{\footnotesize{}iAF1260} & {\footnotesize{}572} & {\footnotesize{}1668} & {\footnotesize{}2382} & {\footnotesize{}6368}\tabularnewline
		\hline 
		{\footnotesize{}iJO1366} & {\footnotesize{}590} & {\footnotesize{}1805} & {\footnotesize{}2583} & {\footnotesize{}7284}\tabularnewline
		\hline 
		{\footnotesize{}Recon1} & {\footnotesize{}932} & {\footnotesize{}2766} & {\footnotesize{}3742} & {\footnotesize{}8717}\tabularnewline
		\hline 
		{\footnotesize{}Recon2} & {\footnotesize{}2430} & {\footnotesize{}5063} & {\footnotesize{}7440} & {\footnotesize{}19791}\tabularnewline
		\hline 
		{\footnotesize{}Recon3} & {\footnotesize{}5335} & {\footnotesize{}8399} & {\footnotesize{}13543} & {\footnotesize{}48187}\tabularnewline
		\hline 
	\end{tabular}
	$ $
	\centering{}{\footnotesize{}}%
	\begin{tabular}{>{\centering}p{1.4cm}||>{\centering}p{1.2cm}|>{\centering}p{1.5cm}|>{\centering}p{1.2cm}|>{\centering}p{0.8cm}}
		{\footnotesize{}LP Model} & {\footnotesize{}Full-dim} & {\footnotesize{}Consts ($m$)} & {\footnotesize{}Vars ($n$)} & {\footnotesize{}nnz}\tabularnewline
		\hline 
		\hline 
		{\footnotesize{}israel} & {\footnotesize{}142} & {\footnotesize{}174} & {\footnotesize{}316} & {\footnotesize{}2519}\tabularnewline
		\hline 
		{\footnotesize{}gfrd\_pnc} & {\footnotesize{}544} & {\footnotesize{}616} & {\footnotesize{}1160} & {\footnotesize{}2393}\tabularnewline
		\hline 
		{\footnotesize{}25fv47} & {\footnotesize{}1056} & {\footnotesize{}821} & {\footnotesize{}1876} & {\footnotesize{}10566}\tabularnewline
		\hline 
		{\footnotesize{}pilot\_ja} & {\footnotesize{}1002} & {\footnotesize{}940} & {\footnotesize{}2267} & {\footnotesize{}11886}\tabularnewline
		\hline 
		{\footnotesize{}sctap2} & {\footnotesize{}1410} & {\footnotesize{}1090} & {\footnotesize{}2500} & {\footnotesize{}7334}\tabularnewline
		\hline 
		{\footnotesize{}ship08l} & {\footnotesize{}2700} & {\footnotesize{}778} & {\footnotesize{}4363} & {\footnotesize{}9434}\tabularnewline
		\hline 
		{\footnotesize{}cre\_a} & {\footnotesize{}3703} & {\footnotesize{}3516} & {\footnotesize{}7248} & {\footnotesize{}17368}\tabularnewline
		\hline 
		{\footnotesize{}woodw} & {\footnotesize{}4656} & {\footnotesize{}1098} & {\footnotesize{}8418} & {\footnotesize{}23158}\tabularnewline
		\hline 
		{\footnotesize{}80bau3b} & {\footnotesize{}9233} & {\footnotesize{}2262} & {\footnotesize{}12061} & {\footnotesize{}22341}\tabularnewline
		\hline 
		{\footnotesize{}ken\_18} & {\footnotesize{}49896} & {\footnotesize{}105127} & {\footnotesize{}154699} & {\footnotesize{}295946}\tabularnewline
		\hline 
	\end{tabular}{\footnotesize{} }{\footnotesize\par}
	\par\end{raggedright}
	\raggedright{}}
	\end{center}
	\caption{Constraint-based models. Each constraint-based model has a form of
		$\{x\in\R^{n}:Ax=b,\,l\protect\leq x\protect\leq u\}$ for $A\in\R^{m\times n},b\in\R^{m}$
		and $l,u\in\R^{n}$, where the rows and columns correspond to constraints
		and variables respectively. The full-dimension of each model is obtained
		by transforming its degenerate subspace to a full dimensional representation
		(i.e., $A'x\protect\leq b'$), and we count the number of nonzero
		(nnz) entries of a preprocessed matrix $A$.}
	\label{tab:data}
\end{table}

\paragraph*{Preprocessing.}

We preprocessed each constrained-based model prior to sampling. This
preprocessing consists mainly of simplifying polytopes, scaling properly
for numerical stability, and finding a feasible starting point. To
simplify a given polytope, we check if $l_{i}=u_{i}$ for each $i\in[n]$
and then incorporate such variables $x_{i}$ into $Ax=b$. Any dense
column is split into several columns with less non-zero entries by
introducing additional variables. Then we remove dependent rows of
$A$ by the Cholesky decomposition. Then we find the Dikin ellipsoid of
the polytope. If the width along some axis is smaller than a preset
tolerance, then we fix variables in such directions, reducing columns
of $A$. Lastly, we run the primal-dual interior-point method with the log-barrier
to find an analytic center of the polytope, which will be used as
a starting point in sampling. When finding the analytic center of
the simplified polytope, if a coordinate of the analytic center is
too close to a boundary (to be precise, smaller than a preset tolerance
boundary $10^{-8}$), then we assume that the inequality constraint
(either $x_{i}\leq u_{i}$ or $l_{i}\leq x_{i}$) is tight, and we
collapse such a variable by moving it into the constraints $Ax=b$.
We go back to the step for removing dependent rows and repeat until
no more changes are made to $A$. Along with simplification, we keep
rescaling $A,b,l,u$ for numerical stability.

\paragraph*{Coordinate Hit-and-Run (CDHR).}
We briefly explain how CHRR works. First, rounding via the MVE algorithm
finds the maximum volume ellipsoid inscribed in the polytope and applies,
to the polytope, an affine transformation that makes this ellipsoid
a unit ball. This procedure puts a possibly highly-skewed polytope
into John's position, which guarantees that the polytope contains
a unit ball and is contained in a ball of radius $n.$ This position
still has a beneficial effect on sampling in practice in the sense
that the random walk can converge in fewer steps. After the transformation,
the random walk based on Coordinate Hit-and-Run (CHAR) chooses a random
coordinate and moves to a random point on the line through the current
point along the chosen coordinate.

When running CHRR and CDHR, we recorded a sample every $n^{2}$ steps.
The mixing rate (i.e., the number of steps required to get a sample
from a target distribution) of Hit-and-Run (HAR), a general version
of CHAR choosing a random direction (unit vector) instead of a random
coordinate, is $O^{*}(n^{2}R^{2})$ for a polytope $P$ with $B_{n}\subseteq P\subseteq R\cdot B_{n}$,
where $B_{n}$ is the unit ball in $\R^{n}$ \cite{lovasz2006hit}.
It was proved only recently that CHAR mixes in $O^{*}(n^{9}R^{2})$
steps on such a polytope \cite{laddha2020convergence,narayanan_srivastava_2021}.
Even though this bound is not as tight as the mixing-rate bound for
HAR, it was reported in \cite{haraldsdottir2017chrr} that CHRR mixes
in the same number of steps as HAR empirically. Moreover, the per-step
complexity of CHAR can be $n$ times faster than that of HAR, so CHAR
brings a significant speed-up in practice.

\paragraph*{Comparison Setup.}

We set the parameters of CRHMC to values in \texttt{default\_options.m} in the experiments.
For the competitors, we proceeded with the following additional steps for fair comparison.
First, as the VolEsti package does not support the MVE rounding, we
rounded each polytope by the MVE algorithm in the CHRR package and
then transformed the rounded polytope so that the R interface can
read the data file. Next, we limited all algorithms to a single core,
since the R interface uses a single core as a default whereas MATLAB
uses as many available cores as possible.
\subsection{Polytope Definition}
\label{subsec:polydef}

\paragraph{Hypercube.}
The $n$-dimensional hypercube is defined by $\{x\in\Rn:-\half\leq x_{i}\leq\half\textrm{ for all }i\in[n]\}$.
Note that it has no equality constraint and its full-dimension is
$n$.
\paragraph{Simplex.}
The $n$-dimensional simplex is defined by $\{x\in\Rn:0\leq x_{i}\textrm{ for all }i\in[n],\,\sum_{i=1}^{n}x_{i}=1\}$.
Note that its full-dimension is $n-1$.
\paragraph{Birkhoff Polytope.}
The $n^{th}$ Birkhoff polytope $B_{n}$ is the set of all doubly
stochastic $n\times n$ matrices (or the convex hull of all permutation
matrices), which is defined as
\[
B_{n}=\{(X_{ij})_{i,j\in[n]}:\sum_{j}X_{ij}=1\textrm{ for all }i\in[n],\,\sum_{i}X_{ij}=1\textrm{ for all }j\in[n],\textrm{ and }X_{ij}\geq0\}.
\]
Namely, $B_{n}$ is defined in a constrained $\R^{n^{2}}$-dimensional
space, and its full-dimension is $n^{2}-(2n-1)=(n-1)^{2}$. We ran
CRHMC on $B_{\sqrt{n}}$ to examine its efficiency on (roughly) $n$-dimensional
Birkhoff polytope.

%% file: 91appendix_ideal_detail.tex
\section{Deferred details of CRHMC\label{subsec:contiCRHMC}}

In this section, we present all technical details behind an \emph{idealized} version of our algorithm, CRHMC, together with correctness of CRHMC.
Subsequently in Appendix~\ref{subsec:Discretization}, we provide details on a \emph{discretized} version of CRHMC.

\subsection{Deferred details of Section~\ref{subsec:Basics-of-CRHMC}\label{subsec:derivationH}}

Recall that in Section~\ref{subsec:Basics-of-CRHMC} we mention that the following constrained Hamiltonian satisfies the Hamiltonian ODE $\left(\frac{dx}{dt}=\frac{\partial H(x,v)}{\partial v},\ \frac{dv}{dt}=-\frac{\partial H(x,v)}{\partial x}\right)$:
\[
H(x,v)=\overline{H}(x,v)+\lambda(x,v)^{\top}c(x)\quad\text{with}\quad\overline{H}(x,v)=f(x)+\frac{1}{2}v^{\top}M(x)^{\dagger}v+\log\pdet(M(x))
\]
where
\[
\lambda(x, v)
=
(Dc(x)Dc(x)^{\top})^{-1}\left(D^{2}c(x)[v,\frac{dx}{dt}]-Dc(x)\frac{\partial\overline{H}(x, v)}{\partial x}\right).
\]

\begin{restatable}{lem}{CHMCinva}
	\label{lem:CHMC_inva}Consider the constrained Hamiltonian defined
	by (\ref{eq:const_H}) with $\range(M(x))=\nulls(Dc(x))$ and 
	\[
	\lambda(x_{t},v_{t})=(Dc(x_{t})Dc(x_{t})^{\top})^{-1}\left(D^{2}c(x_{t})[v_{t},\frac{dx_{t}}{dt}]-Dc(x_{t})\frac{\partial\overline{H}(x_{t},v_{t})}{\partial x}\right).
	\]
	When the initial point satisfies $c(x_{0})=0$, the ODE solution of
	(\ref{eq:ham_ODE}) satisfies $c(x_{t})=0$ and $Dc(x_{t})v_{t}=Dc(x_{0})v_{0}$
	for all $t$.
\end{restatable}

\begin{proof}
	First we compute
	\begin{align*}
		\frac{d}{dt}c(x_{t}) & =Dc(x_{t})\cdot\frac{dx_{t}}{dt}=Dc(x_{t})\cdot\frac{\partial H(x_{t},v_{t})}{\partial v_{t}}\\
		& =Dc(x_{t})M(x_{t})^{\dagger}v+Dc(x_{t})D_{v}\lambda(x_{t},v_{t})^{\top}c(x_{t})\\
		& =Dc(x_{t})D_{v}\lambda(x_{t},v_{t})^{\top}c(x_{t})
	\end{align*}
	where we used $\range(M(x)^{\dagger})=\range(M(x))=\nulls(Dc(x))$.
	Since $c(x_{0})=0$, by the uniqueness of the ODE solution, we have
	that $c(x_{t})=0$ for all $t$. Next we compute
	
	\begin{align*}
		\frac{dv_{t}}{dt}= & -\frac{\partial H(x_{t},v_{t})}{\partial x}\\
		= & -\frac{\partial\overline{H}(x_{t},v_{t})}{\partial x}-Dc(x_{t})^{\top}\lambda(x_{t},v_{t})-D_{x}\lambda(x_{t},v_{t})^{\top}c(x_{t})\\
		= & -\frac{\partial\overline{H}(x_{t},v_{t})}{\partial x}-Dc(x_{t})^{\top}\lambda(x_{t},v_{t})
	\end{align*}
	where we used $c(x_{t})=0$. Hence, we have
	\begin{align*}
		\frac{d}{dt}Dc(x_{t})v_{t} & =D^{2}c(x_{t})[v_{t},\frac{dx_{t}}{dt}]+Dc(x_{t})\frac{dv_{t}}{dt}\\
		& =D^{2}c(x_{t})[v_{t},\frac{dx_{t}}{dt}]-Dc(x_{t})\frac{\partial\overline{H}(x_{t},v_{t})}{\partial x}-Dc(x_{t})Dc(x_{t})^{\top}\lambda(x_{t},v_{t}).
	\end{align*}
	By setting $\lambda(x_{t},v_{t})=(Dc(x_{t})Dc(x_{t})^{\top})^{-1}(D^{2}c(x_{t})[v_{t},\frac{dx_{t}}{dt}]-Dc(x_{t})\frac{\partial\overline{H}(x_{t},v_{t})}{\partial x})$,
	we have $\frac{d}{dt}Dc(x_{t})v_{t}=0$ and $Dc(x_{t})v_{t}=Dc(x_{0})v_{0}$
	for all $t$ (i.e., $v_{t}\in\nulls(Dc(x_{t}))$ during Step 2).
\end{proof}

\subsection{Deferred details of Section~\ref{subsec:Efficient} \label{subsec:postponed_sec_efficient}}

In Section~\ref{subsec:Efficient}, we mention that a naive algorithm computing $\partial H/\partial x$
and $\partial H/\partial v$ is bound to face the following challenges, especially in high-dimensional regime, and briefly explain how we address each of them.
In this section, we give full details on our computational tricks.

\begin{tcolorbox}
\begin{enumerate}\setlength{\leftmargin}{0pt}
	\item Computation of the pseudo-inverse and its derivatives takes $O(n^{3})$, except for very special matrices 
	$\Longrightarrow$ Find equivalent formulas (Appendix~\ref{subsec:Avoiding-pseudo-inverse}).
	\item The Lagrangian term in the constrained Hamiltonian entails extra computation such as $D^2c(x)$ 
	$\Longrightarrow$ Simplify the constrained Hamiltonian (Appendix~\ref{subsec:Simplification-for-subspace}).
	\item A naive approach to computing leverage scores in $\partial H/\partial x$ results in a very dense matrix
	$\Longrightarrow$ Track sparsity pattern (Appendix~\ref{subsec:Efficient-Computation-of}). 
\end{enumerate}	
\end{tcolorbox}

\subsubsection{Avoiding pseudo-inverse and pseudo-determinant\label{subsec:Avoiding-pseudo-inverse}}

We start with a formula for $M(x)^{\dagger}$.
\begin{lem}
	\label{lem:formula_inverse}Let $M(x)=Q(x)\cdot g(x)\cdot Q(x)$ where
	$Q(x)=I-Dc(x)^{\top}(Dc(x)\cdot Dc(x)^{\top})^{-1}Dc(x)$ is the orthogonal
	projection to the null space of $Dc(x)$. Then, $Dc(x)\cdot M(x)^{\dagger}=0$
	and $M(x)^{\dagger}=g(x)^{-\frac{1}{2}}\cdot (I-P(x))\cdot g(x)^{-\frac{1}{2}}$
	with 
	\[
	P(x)=g(x)^{-\frac{1}{2}}\cdot Dc(x)^{\top}(Dc(x)\cdot g(x)^{-1}\cdot Dc(x)^{\top})^{-1}Dc(x)\cdot g(x)^{-\frac{1}{2}}.
	\]
\end{lem}

\begin{proof}
	Recall that $\range(M(x)^{\dagger})=\range(M(x))$. Hence, for any
	$u\in\Rn$, we have that $M(x)^{\dagger}u\in\range(M(x))$. Since
	$\range(M(x))\subseteq\range(Q(x))$ and $\range(Q(x))=\nulls(Dc(x))$
	due to the definition of the orthogonal projection $Q(x)$, it follows
	that $Dc(x)\cdot M(x)^{\dagger}u=0$ for all $u$. 
	
	For the formula of $M(x)^{\dagger}$, we simplify the notation by
	ignoring the parameter $x$. Let $N=g^{-\frac{1}{2}}Pg^{-\frac{1}{2}}$
	and $J=Dc(x)$. The goal is to prove that $M^{\dagger}=N$. First,
	we show some basic identities about $Q$ and $N$:
	\begin{align}
		QN= & Qg^{-\frac{1}{2}}(I-g^{-\frac{1}{2}}J^{\top}(Jg^{-1}J^{\top})^{-1}Jg^{-\frac{1}{2}})g^{-\frac{1}{2}}\nonumber \\
		= & (I-J^{\top}(JJ^{\top})^{-1}J)(g^{-1}-g^{-1}J^{\top}(Jg^{-1}J^{\top})^{-1}Jg^{-1})\nonumber \\
		= & g^{-1}-J^{\top}(JJ^{\top})^{-1}Jg^{-1}\nonumber \\
		& -(g^{-1}J^{\top}(Jg^{-1}J^{\top})^{-1}Jg^{-1}-J^{\top}(JJ^{\top})^{-1}Jg^{-1}J^{\top}(Jg^{-1}J^{\top})^{-1}Jg^{-1})\nonumber \\
		= & N. \label{eq:QNN}
	\end{align}
	Similarly, we have $NQ=N$, $QgN=Q$, and $NgQ=Q$. To prove that
	$M^{\dagger}=N$, we need to check that $MN$ and $NM$ are symmetric,
	$MNM=M$, and $NMN=N$.
	
	For symmetry of $MN$ and $NM$, we note that $MN=QgQN=QgN=Q$ and
	$NM=NQgQ=NgQ=Q$. For the formula of $MNM$ and $NMN$, we note that
	that $Q$ is a projection matrix and hence 
	\begin{align*}
		MNM & =QM=QQgQ=QgQ=M,\\
		NMN & =QN=N.
	\end{align*}
	Therefore, we have $M^{\dagger}=N$.
\end{proof}
Another bottleneck of the algorithm is to compute $\log\pdet M(x)$.
The next lemma shows a simpler formula that can take advantage of
sparse Cholesky decomposition.
\begin{lem}
	\label{lem:formula_pdet}We have that
	\[
	\log\pdet(M(x))=\log\det g(x)+\log\det\left(Dc(x)\cdot g(x)^{-1}\cdot Dc(x)^{\top}\right)-\log\det\left(Dc(x)\cdot Dc(x)^{\top}\right).
	\]
\end{lem}

\begin{proof}
	We simplify the notation by ignoring the parameter $x$ and letting
	$J=Dc(x)$. Let 
	\begin{align*}
		f_{1}(g) & =\log\pdet(Q\cdot g\cdot Q), \\
		f_{2}(g) & =\log\det g+\log\det Jg^{-1}J^{\top}-\log\det JJ^{\top}.
	\end{align*}
	Clearly, $f_{1}(I)=f_{2}(I)=0$, and hence it suffices to prove that
	their derivatives are the same. 
	
	Note that $\text{Range}(Q\cdot g\cdot Q)=\text{Null}(J)$ and $\range(J^{\top})$
	is the orthogonal complement of $\nulls(J)$. Since $J^{\top}(JJ^{\top})^{-1}J$
	is the orthogonal projection to $\range(J^{\top})$, all of its eigenvectors
	in $\range(J^{\top})$ have eigenvalue $1$ and all the rest in $\nulls(J)$
	have eigenvalue $0$. Therefore, by padding eigenvalue $1$ on $\range(J^{\top})=\nulls(J)^{\perp}=\range(QgQ)^{\perp}$,
	we have 
	\begin{align*}
		\pdet(Q\cdot g\cdot Q) & =\det(Q\cdot g\cdot Q+J^{\top}(JJ^{\top})^{-1}J)\\
		& =\det(Q\cdot g\cdot Q+(I-Q)).
	\end{align*}
	Using $D\log\det A(g)[u]=\tr(A(g)^{-1}DA(g)[u]),$ the directional
	derivative of $f_{1}$ on direction $u$ is
	\[
	Df_{1}(g)[u]=\tr\text{\ensuremath{\left((Q\cdot g\cdot Q+(I-Q))^{-1}Q\cdot u\cdot Q\right)}}.
	\]
	Let $N=(Q\cdot g\cdot Q)^{\dagger}$. As shown in the proof of Lemma
	\ref{lem:formula_inverse}, we have $NQ=QN=N$ and $QgN=Q$. By using
	these identities, we can manually check that $(Q\cdot g\cdot Q+(I-Q))^{-1}=N+(I-Q)$.
	Hence, 
	\begin{align*}
		Df_{1}(g)[u] & =\tr\left((N+(I-Q))Q\cdot u\cdot Q\right)=\tr(NuQ)\\
		& =\tr(QNu)=\tr(Nu)
	\end{align*}
	where we used idempotence of the projection matrix $Q$ (i.e., $Q^{2}=Q$).
	
	On the other hand, we have
	\begin{align*}
		Df_{2}(g)[u] & =\tr(g^{-1}u)-\tr\left((Jg^{-1}J^{\top})^{-1}(Jg^{-1}ug^{-1}J^{\top})\right)\\
		& =\tr\left((g^{-1}-g^{-1}J^{\top}(Jg^{-1}J^{\top})^{-1}Jg^{-1})u\right)\\
		& =\tr(Nu)
	\end{align*}
	where we used the alternative formula of $N$ in Lemma \ref{lem:formula_inverse}.
	This shows that the derivative of $f_{1}$ equals to that of $f_{2}$
	at any point $g\succ0$. Since the set of positive definite matrices
	is connected and $f_{1}(I)=f_{2}(I)$, this implies that $f_{1}(g)=f_{2}(g)$
	for all $g\succ0$.
\end{proof}
Combining Lemma \ref{lem:formula_inverse} and Lemma \ref{lem:formula_pdet},
we have the following formula of the Hamiltonian.
\begin{align*}
	\overline{H}(x,v)= & H_0(x,v)+\lambda(x,v)^{\top}c(x),\\
	H_0(x,v)= & f(x)+\frac{1}{2}v^{\top}g(x)^{-\frac{1}{2}}\left(I-g(x)^{-\frac{1}{2}}\cdot Dc(x)^{\top}(Dc(x)\cdot g(x)^{-1}\cdot Dc(x)^{\top})^{-1}Dc(x)\cdot g(x)^{-\frac{1}{2}}\right)g(x)^{-\frac{1}{2}}v\\
	& +\frac{1}{2}\left(\log\det g(x)+\log\det\left(Dc(x)\cdot g(x)^{-1}\cdot Dc(x)^{\top}\right)-\log\det\left(Dc(x)\cdot Dc(x)^{\top}\right)\right).
\end{align*}

\subsubsection{Simplification for subspace constraints\label{subsec:Simplification-for-subspace}}

For the case $c(x)=Ax-b$, the constrained Hamiltonian is 
\begin{align}
	\overline{H}(x,v)= & f(x)+\frac{1}{2}v^{\top}g^{-\frac{1}{2}}\left(I-P\right)g^{-\frac{1}{2}}v+\frac{1}{2}\left(\log\det g+\log\det Ag^{-1}A^{\top}-\log\det AA^{\top}\right)+\lambda^{\top}c\label{eq:H_A}
\end{align}
where $P=g^{-\frac{1}{2}}A^{\top}(Ag^{-1}A^{\top})^{-1}Ag^{-\frac{1}{2}}$.
The following lemma shows that the dynamics corresponding to $\overline{H}$ above is equivalent to a simpler Hamiltonian. 
The key observation is that the algorithm only needs to know $x(h)$ in the HMC dynamics, and not $v(h)$. 
Thus we can replace $\overline{H}$ by any other $H$ that produces the same $x(h)$.
\begin{lem}
	\label{lem:ham_deri}The Hamiltonian dynamics of $x$ corresponding
	to (\ref{eq:H_A}) is same as the dynamics of $x$ corresponding to
	\begin{equation}
		H(x,v)=f(x)+\frac{1}{2}v^{\top}g^{-\frac{1}{2}}\left(I-P\right)g^{-\frac{1}{2}}v+\frac{1}{2}\left(\log\det g+\log\det Ag^{-1}A^{\top}\right)\label{eq:H_A_new}
	\end{equation}
	where $P=g^{-\frac{1}{2}}A^{\top}(Ag^{-1}A^{\top})^{-1}Ag^{-\frac{1}{2}}$.
	Furthermore, we have
	\begin{align}
		\frac{dx}{dt} & =g^{-\frac{1}{2}}\left(I-P\right)g^{-\frac{1}{2}}v,\label{eq:CHMC_dx}\\
		\frac{dv}{dt} & =-\nabla f(x)+\frac{1}{2}Dg\left[\frac{dx}{dt},\frac{dx}{dt}\right]-\frac{1}{2}\tr(g^{-\frac{1}{2}}\left(I-P\right)g^{-\frac{1}{2}}Dg).\label{eq:CHMC_dv}
	\end{align}
\end{lem}

\begin{proof}
	Note that the dynamics of $x$ corresponding to (\ref{eq:H_A}) is
	given by
	\begin{align}
		\frac{dx}{dt} & =\frac{\partial \overline{H}}{\partial v}=g^{-\frac{1}{2}}\left(I-P\right)g^{-\frac{1}{2}}v+(D_{v}\lambda)^{\top}c\nonumber \\
		& =g^{-\frac{1}{2}}\left(I-P\right)g^{-\frac{1}{2}}v\label{eq:CHMC_dxdt}
	\end{align}
	where we used that $c(x)=0$ (Lemma \ref{lem:CHMC_inva}). 
	
	Now let us compute the dynamics of $v$. Note that
	\[
	v^{\top}g^{-\frac{1}{2}}\left(I-P\right)g^{-\frac{1}{2}}v=v^{\top}g^{-1}v-v^{\top}g^{-1}A^{\top}(A\cdot g^{-1}\cdot A^{\top})^{-1}Ag^{-1}v.
	\]
	Hence, we have
	\begin{align*}
		& D_{x}\left(\frac{1}{2}v^{\top}g^{-\frac{1}{2}}\left(I-P\right)g^{-\frac{1}{2}}v\right) \\= & -\frac{1}{2}v^{\top}g^{-1}\cdot Dg\cdot g^{-1}v+v^{\top}g^{-1}\cdot Dg\cdot g^{-1}A^{\top}(A\cdot g^{-1}\cdot A^{\top})^{-1}Ag^{-1}v\\
		& -\frac{1}{2}v^{\top}g^{-1}A^{\top}(A\cdot g^{-1}\cdot A^{\top})^{-1}A\cdot g^{-1}\cdot Dg\cdot g^{-1}\cdot A^{\top}(A\cdot g^{-1}\cdot A^{\top})^{-1}Ag^{-1}v\\
		= & -\frac{1}{2}v^{\top}g^{-\frac{1}{2}}(I-P)g^{-\frac{1}{2}}\cdot Dg\cdot g^{-\frac{1}{2}}(I-P)g^{-\frac{1}{2}}v\\
		= & -\frac{1}{2}Dg\left[\frac{dx}{dt},\frac{dx}{dt}\right],
	\end{align*}
	where we used $\frac{dx}{dt}=g^{-\frac{1}{2}}(I-P)g^{-\frac{1}{2}}v$
	in (\ref{eq:CHMC_dxdt}). Therefore, it follows that
	\begin{align}
		\frac{dv}{dt} & =-\frac{\partial\overline{H}}{\partial x}-(D_{v}\lambda)^{\top}c-A^{\top}\lambda\label{eq:dvdt}\\
		& =-\nabla f(x)+\frac{1}{2}Dg\left[\frac{dx}{dt},\frac{dx}{dt}\right]-\frac{1}{2}\tr(g^{-1}Dg)\\&+\frac{1}{2}\tr\left((Ag(x)^{-1}A^{\top})^{-1}Ag(x)^{-1}\cdot Dg\cdot g(x)^{-1}A^{\top}\right)-A^{\top}\lambda\nonumber \\
		& =-\nabla f(x)+\frac{1}{2}Dg\left[\frac{dx}{dt},\frac{dx}{dt}\right]-\frac{1}{2}\tr(g^{-\frac{1}{2}}\left(I-P\right)g^{-\frac{1}{2}}Dg)-A^{\top}\lambda\nonumber 
	\end{align}
	where we used that $c=0$ again in the second equality.
	
	Recall that $\frac{dx}{dt}=g^{-\frac{1}{2}}\left(I-P\right)g^{-\frac{1}{2}}v$.
	In this formula, let us perturb $v$ by $A^{\top}y$ for any $y$
	as follows. 
	\begin{align*}
		\left(I-P\right)g^{-\frac{1}{2}}(v+A^{\top}y) & =\left(I-P\right)g^{-\frac{1}{2}}v+\left(I-g^{-\frac{1}{2}}A^{\top}(Ag^{-1}A^{\top})^{-1}Ag^{-\frac{1}{2}}\right)g^{-\frac{1}{2}}A^{\top}y\\
		& =\left(I-P\right)g^{-\frac{1}{2}}v+(g^{-\frac{1}{2}}A^{\top}y-g^{-\frac{1}{2}}A^{\top}(Ag^{-1}A^{\top})^{-1}(Ag^{-1}A^{\top})y)\\
		& =\left(I-P\right)g^{-\frac{1}{2}}v+(g^{-\frac{1}{2}}A^{\top}y-g^{-\frac{1}{2}}A^{\top}y)\\
		& =\left(I-P\right)g^{-\frac{1}{2}}v.
	\end{align*}
	Hence, removing $A^{\top}\lambda$ from $\frac{dv}{dt}$ in (\ref{eq:dvdt})
	does not change the dynamics of $x$, and thus we have the new dynamics
	given simply by (\ref{eq:CHMC_dx}) and (\ref{eq:CHMC_dv}). By repeating
	this proof, one can check that the simplified Hamiltonian (\ref{eq:H_A_new})
	also yields (\ref{eq:CHMC_dx}) and (\ref{eq:CHMC_dv}).
\end{proof}

\subsubsection{Efficient Computation of Leverage Score\label{subsec:Efficient-Computation-of}}

In this section, we discuss how we efficiently compute the diagonal
entries of $A^{\top}(Ag^{-1}A^{\top})^{-1}A$. Our idea is based on
the fact that certain entries of $(Ag^{-1}A^{\top})^{-1}$ can be
computed as fast as computing sparse Cholesky decomposition of $Ag^{-1}A^{\top}$
\cite{takahashi1973formation,campbell1995computing}, which can be
$O(n)$ time faster than computing $(Ag^{-1}A^{\top})^{-1}$ in many
settings.

For simplicity, we focus on the case $g(x)$ as a diagonal matrix,
since we use the log-barrier $\phi(x)=-\sum_{i=1}^{m}(\log(x_{i}-l_{i})+\log(u_{i}-x_{i}))$
in implementation. We first note that we maintain a ``sparsity pattern''
$\sp(M)$ of a sparse matrix $M$ so that we handle only these entries
in downstream tasks. The sparsity pattern indicates ``candidates''
of nonzero entries of a matrix (i.e., $\sp(M)\supseteq nnz(M)=\{(i,j):M_{ij}\neq0\}$).
For instance, it is obvious that $\sp(cc^{\top})=\{(i,j):c_{i}c_{j}\neq0\}=nnz(cc^{\top})$
for a column vector $c$ and that $\sp(Ag^{-1}A^{\top})=\bigcup_{i\in[n]}\sp(A_{i}A_{i}^{\top})$
follows from the equality $Ag^{-1}A^{\top}=\sum_{i=1}^{n}(Ag^{-\half})_{i}(Ag^{-\half})_{i}^{\top}$,
where $M_{i}$ denote the $i^{th}$ column of $M$ (See Theorem 2.1
in \cite{davis2006direct}). Then we compute the Cholesky decomposition
to obtain a sparse triangular matrix $L$ such that $LL^{\top}=Ag^{-1}A^{\top}$
with a property $\sp(Ag^{-1}A^{\top})\subseteq\sp(L^{\top})\cup\sp(L)$
(See Theorem 4.2 in \cite{davis2006direct}).

Once the sparsity pattern of $L$ is identified, we compute $S:=(Ag^{-1}A^{\top})^{-1}\vert_{\sp(L)}$,
the restriction of $S$ to $\sp(L)$, that is, the inverse matrix
$S$ is computed only for entries in $\sp(L)$. \cite{takahashi1973formation,campbell1995computing}
showed that this matrix $S$ can be computed as fast as the Cholesky
decomposition of $Ag^{-1}A^{\top}$.

For completeness, we explain how they compute $S$ efficiently. Let
$L_{0}DL_{0}^{\top}$ be the LDL decomposition of $Ag^{-1}A^{\top}$
such that the diagonals of $L_{0}$ is one and so $L=L_{0}D^{\half}$,
and it easily follows that
\[
S=D^{-1}L_{0}^{-1}+(I-L_{0}^{\top})S=D^{-\half}L^{-1}+(I-L^{\top}D^{-\half})^{-1}S.
\]
Since $D^{-1}L_{0}^{-1}$ is lower triangular and $I-L_{0}^{\top}$
is strictly upper triangular, symmetry of $S$ implies that $S$ can
be computed from the bottom row to the top row one by one. We note
that the computation of $S$ on any entry in $\sp(L)$ only requires
previously computed $S$ on entries in $\sp(L)$, due to the sparsity
pattern of $I-L^{\top}D^{-\half}$. \cite{takahashi1973formation,campbell1995computing}
showed that the total cost of computing $S$ is $O(\sum_{i=1}^{n}n_{i}^{2})$
for backward substitution, where $n_{i}$ is the number of nonzeros
in the $i^{th}$ column of $L$. This exactly matches the cost of
computing $L$. In our experiments, for many sparse matrices $A$,
we found that $O(\sum_{i=1}^{n}n_{i}^{2})$ is roughly $O(n^{1.5})$
and it is much faster than dense matrix inverse.

We have presented methods to save computational cost, avoiding full
computation of the inverse $(Ag^{-1}A^{\top})^{-1}$. This attempt
is justified by the fact that only entries of $Ag^{-1}A^{\top}$ in
$\sp(L)\cup\sp(L^{\top})$ matter in computing $\diag(A^{\top}SA)=\diag(A^{\top}(Ag^{-1}A^{\top})^{-1}A)$.
\begin{lem}
	Computation of $\diag(A^{\top}(Ag^{-1}A^{\top})^{-1}A)$ involves
	accessing only entries of $(Ag^{-1}A^{\top})^{-1}$ in $\sp(Ag^{-1}A^{\top})$.
\end{lem}

\begin{proof}
	Let $M:=(Ag^{-1}A^{\top})^{-1}\in\R^{m\times m}$, $\sigma_{i}:=(A^{\top}(Ag^{-1}A^{\top})^{-1}A)_{ii}$
	for $i\in[n]$, and $a_{i}$ be the $i^{th}$ column of $A$. Observe
	that
	\[
	\sigma_{i}=a_{i}^{\top}(Ag^{-1}A^{\top})^{-1}a_{i}=\tr(a_{i}^{\top}Ma_{i})=\tr(Ma_{i}a_{i}^{\top}).
	\]
	As the entries of $M$ only in $\sp(a_{i}a_{i}^{\top})$ matter when
	computing the trace, we have that all the entries of $M$ used for
	computing $\sigma_{i}$ for all $i\in[n]$ are included in $\bigcup_{i=1}^{n}\sp(a_{i}a_{i}^{\top})=\sp(Ag^{-1}A^{\top})$.
\end{proof}
Now let us divide the diagonals of $S$ by $2$. Then we have $(Ag^{-1}A^{\top})^{-1}\vert_{\sp(L)\cup\sp(L^{\top})}=S+S^{\top}$
and thus 
\begin{align*}
	\diag(A^{\top}(Ag^{-1}A^{\top})^{-1}A)=\diag(A^{\top}(Ag^{-1}A^{\top})^{-1}\vert_{\sp(L)\cup\sp(L^{\top})}A)\\
	=\diag(A^{\top}SA+A^{\top}S^{\top}A)=2\cdot\diag(A^{\top}SA)
\end{align*}
and the last term can be computed efficiently using $S$. In our experiment,
the cost of computing leverage score is roughly twice the cost of
computing Cholesky decomposition in all datasets.

Finally, we discuss another approach to compute leverage score with
the same asymptotic complexity. We consider the function
\[
V(g)=\log\det Ag^{-1}A^{\top}
\]
where $g$ is a sparse matrix $g\in\mathbb{R}^{\sp(g)}$ and $V$
is defined only on $\mathbb{R}^{\sp(g)}$. Note that $V(g)$ can be
computed using Cholesky decomposition of $A^{\top}g^{-1}A^{\top}$
and multiplying the diagonal of the decomposition. Next, we note that
\[
\nabla V(g)=-(g^{-1}A^{\top}(Ag^{-1}A^{\top})^{-1}Ag^{-1})|_{\sp(g)}.
\]
Hence, we can compute leverage score by first computing $\nabla V(g)$
via automatic differentiation, and the time complexity of computing
$\nabla V$ is only a small constant factor more than the time complexity
of computing $V$ \cite{griewank2008evaluating}. The only problem
with this approach is that the Cholesky decomposition algorithm is
an algorithm involving a large loop and sparse operations and existing
automatic differentiation packages are not efficient to differentiate
such functions.

\subsection{Stationarity of CRHMC\label{subsec:Stationarity-of-CRHMC}}

Now, the ideal CRHMC (or the continuous CRHMC) is the same as Algorithm~\ref{alg:HMC} with the simplified constrained Hamiltonian $H$ in place of the unconstrained Hamiltonian.
In this section, we prove that the Markov chain defined by the ideal CRHMC projected
to $x$ satisfies detailed balance with respect to its target distribution
proportional to $e^{-f(x)}$ subject to $c(x)=0$, leading to the
target distribution being stationary.

To this end, we introduce a few notations here. 
Let $\mathcal{M}=\{x\in\Rn:c(x)=0\}$ be a manifold in $\Rn$ and $\pi(x)$ be a desired distribution on $\mathcal{M}$ proportional to $e^{-f(x)}$ satisfying $\int_{\mathcal{M}}\pi(x)dx=1$ (to be precise, the Radon-Nikodym derivative of $\pi$ w.r.t. the Hausdorff measure on the manifold $\mathcal{M}$ is proportional to $e^{-f(x)}$).
We denote the set of velocity $v$ at $x\in\mathcal{M}$ (i.e., cotangent space) by $\tcal_{x}\mcal=\nulls(Dc(x))=\{v\in\Rn:Dc(x)M(x)^{\dagger}v=0\}$.
Let $T_{h}$ be the map sending $(x,v)$ to $(x',v')=(x(h),y(h))$ in the Hamiltonian ODE (Step 2 of Algorithm~\ref{alg:HMC}) and define $F_{x,h}(v):=(\pi_{1}\circ T_{h})(x,v)=x'$, where $\pi_{1}(x,v):=x$ is the projection to the position space $x$.
For a matrix $A$, we denote by $|A|$ the absolute value of its determinant $|\det(A)|$.

Note that we check the detailed balance of the induced chain on the “original $(x)$” space without moving to the “phase $(x,v)$” space, unlike Brubaker’s proof \cite{brubaker2012family}.

\begin{thm}
	For $x,x'\in\mcal$, let $\P_{x}(x')$ be the probability density
	of the one-step distribution to $x'$ starting at $x$ in CRHMC (i.e.,
	transition kernel from $x$ to $x'$). It satisfies detailed balance
	with respect to the desired distribution $\pi$ (i.e., $\pi(x)\P_{x}(x')=\pi(x')\P_{x'}(x)$).
\end{thm}

\begin{proof}
	Fix $x$ and $x'$ in $\mcal$. Let $C_{1}$ be the normalization
	constant of $e^{-f(x)}$ (i.e., $\pi(x)=C_{1}e^{-f(x)}$). The transition
	kernel $\P_{x}(x')$ is characterized as the pushforward by $F_{x,h}$
	of the probability measure $v\sim\mathcal{N}(0,M(x))$ on $\tcal_{x}\mcal$,
	so it follows that
	\[
	\P_{x}(x')=C_{2}\int_{V_{x}}\frac{e^{-\frac{1}{2}\log\pdet(M(x))-\frac{1}{2}v^{\top}M(x)^{\dagger}v}}{|DF_{x,h}(v)|}dv,
	\]
	where $C_{2}$ is the normalization constant of $e^{-\frac{1}{2}\log\pdet(M(x))-\frac{1}{2}v^{\top}M(x)^{\dagger}v}$
	and $V_{x}=\{v\in\tcal_{x}\mcal:F_{x,h}(v)=x'\}$ is the set of velocity
	in cotangent space at $x$ such that the Hamiltonian ODE with step
	size $h$ sends $(x,v)$ to $(x',v')$. 
	(Further details for deducing the $1$-step distribution can be found in Lemma 10 of \cite{lee2018convergence})
	As $c(x)=0$ for $x\in\mcal$,
	it follows that
	
	\begin{align*}
		&\pi(x)\P_{x}(x')\\&=C_{1}C_{2}\int_{V_{x}}\frac{e^{-f(x)-\frac{1}{2}\log\pdet(M(x))-\frac{1}{2}v^{\top}M(x)^{\dagger}v-\lambda(x,v)^{\top}c(x)}}{|DF_{x,h}(v)|}dv=C_{1}C_{2}\int_{V_{x}}\frac{e^{-H(x,v)}}{|DF_{x,h}(v)|}dv.
	\end{align*}
	
	Going forward, we use three important properties of the Hamiltonian
	dynamics including reversibility, Hamiltonian preservation, and volume
	preservation, which still hold for the constrained Hamiltonian $H$.
	Due to reversibility $T_{-h}(x',v')=(x,v)$, we can write
	
	\[
	\pi(x')\P_{x'}(x)=C_{1}C_{2}\int_{V_{x'}}\frac{e^{-H(x',v')}}{|DF_{x',-h}(v')|}dv',
	\]
	where $V_{x'}=\{v'\in\tcal_{x'}\mcal:F_{x',-h}(v')=x\}$ is the counterpart
	of $V_{x}$. From reversibility $T_{-h}\circ T_{h}=I$, the inverse
	function theorem implies $DT_{-h}=(DT_{h})^{-1}$. Now let us denote 
	
	\[
	DT_{h}(x,v)=\left[\begin{array}{cc}
		A & B\\
		C & D
	\end{array}\right]\quad\&\quad DT_{-h}(x',v')=\left[\begin{array}{cc}
		A' & B'\\
		C' & D'
	\end{array}\right],
	\]
	where each entry is a block matrix with the same size. Note that $DF_{x,h}(v)=B$
	and $DF_{x',-h}(v')=B'$ hold by the definition of Jacobian. Together
	with $DT_{-h}=(DT_{h})^{-1}$, a formula for the inverse of a block
	matrix results in 
	
	\[
	|DF_{x',-h}(v')|=|B'|=\frac{|B|}{|D||A-BD^{-1}C|}=\frac{|B|}{|DT_{h}(x,v)|}=|B|=|DF_{x,h}(v)|,
	\]
	where we use the property of volume preservation in the fourth equality
	(i.e., $|DT_{h}(x,v)|=1$). Finally, the property of Hamiltonian preservation
	implies $H(x,v)=H(x',v')$ and thus 
	
	\[
	\int_{V_{x'}}\frac{e^{-H(x',v')}}{|DF_{x',-h}(v')|}dv'=\int_{V_{x}}\frac{e^{-H(x,v)}}{|DF_{x,h}(v)|}dv.
	\]
	Therefore, $\pi(x)\P_{x}(x')=\pi(x')\P_{x'}(x)$ holds.
\end{proof}

Similar reasoning as Theorem 3 and Lemma 1 in \cite{brubaker2012family} gives $\pi$-irreducibility and aperiodicity of the process, so CRHMC converges
to the unique stationary distribution $\pi\propto e^{-f(x)}$.

%% file: 92appendix_discretization.tex
\section{Discretization\label{subsec:Discretization}}

We discuss how to implement our Hamiltonian dynamics using the implicit
midpoint method in Section \ref{subsec:Discretized-CRHMC-based} and
present theoretical guarantees of correctness and efficiency of the
discretized CRHMC in Section \ref{subsec:Theoretical-Guarantees}.

\subsection{Discretized CRHMC based on Implicit Midpoint Integrator\label{subsec:Discretized-CRHMC-based}}

In our algorithm, we discretize the Hamiltonian process into steps
of step size $h$ and run the process for $T$ iterations (see Algorithm
\ref{alg:CHMC_JL}). Rather than resampling the velocity at every step, we may change the
velocity more gradually, using the following update:
\[
v'\leftarrow\sqrt{\beta}v+\sqrt{1-\beta}z
\]
where $z\sim\mathcal{N}(0,M(x))$ and $\beta$ is a parameter. We
note that this step is time-reversible, i.e., $\P(v|x)\P(v\rightarrow v')=\P(v'|x)\P(v'\rightarrow v)$
(see Theorem \ref{thm:balance}). Starting from $(x^{(0)},v^{(0)})$, let $(x^{(t)},v^{(t)})$
be the point obtained after iteration $t$. In the beginning of each
iteration, we compute the Cholesky decomposition of $Ag(x)^{-1}A^{\top}$
for later use and resample the velocity with momentum. As noted previously
in Lemma \ref{lem:ham_deri}, for $c(x)=Ax-b$ we can just use the
simplified Hamiltonian in (\ref{eq:H_A_new}),
\[
H(x,v)=f(x)+\frac{1}{2}v^{\top}g(x)^{-\frac{1}{2}}\left(I-P(x)\right)g(x)^{-\frac{1}{2}}v+\frac{1}{2}\left(\log\det g(x)+\log\det Ag(x)^{-1}A^{\top}\right)
\]
instead of the constrained Hamiltonian $H + \lambda^{\top}c$.
We solve the Hamiltonian dynamics for $H$ by the implicit midpoint
method, which we will discuss below, and then use a Metropolis filter
on $H$ to ensure the distribution is correct. 

\paragraph{Implicit Midpoint Method.}

For general Riemannian manifolds, explicit integrators such as the leapfrog method (LM) are not symplectic, unlike IMM. 
LM is symplectic when the Hamiltonian equations are separable (i.e., each of $dx/dt$ and $dv/dt$ is a function of either x or v only).
However, in the general Riemannian manifold setting, where $dx/dt$ depends on position $x$ due to mass matrices (which is $g(x)$ in our paper) as well as velocity $v$, the Hamiltonian is no longer separable, which prevents us from using LM. 
We refer interested readers to Section 3 and Section 4.1 in \cite{cobb2019introducing}.

We now elaborate on how the implicit midpoint integrator works (see Algorithm
\ref{alg:IMM}), which is symplectic (so measure-preserving)
and reversible \cite{hairer2006geometric}. 
Let us write $H(x,v)=\overline{H}_{1}(x,v)+\overline{H}_{2}(x,v)$,
where
\begin{align*}
	\overline{H}_{1}(x,v) & =f(x)+\frac{1}{2}\left(\log\det g(x)+\log\det Ag(x)^{-1}A^{\top}\right),\\
	\overline{H}_{2}(x,v) & =\frac{1}{2}v^{\top}g(x)^{-\frac{1}{2}}\left(I-P(x)\right)g(x)^{-\frac{1}{2}}v.
\end{align*}
Starting from $(x_{0},v_{0})$, in the first step of the integrator,
we run the process on the Hamiltonian $\overline{H}_{1}$ with step
size $\frac{h}{2}$ to get $(x_{1/3},v_{1/3})$, and this discretization
leads to $x_{1/3}=x_{0}+\frac{h}{2}\frac{\partial\overline{H}_{1}}{\partial v}(x_{0},v_{0})$
and $v_{1/3}=v_{0}-\frac{h}{2}\frac{\partial\overline{H}_{1}}{\partial x}(x_{0},v_{0})$.
Note that $x_{1/3}=x_{0}$ due to $\frac{\partial\overline{H}_{1}}{\partial v}=0$.
In the second step of the integrator, we run the process on $\overline{H}_{2}$
with step size $h$ by solving
\begin{align*}
	x_{\frac{2}{3}} & =x_{\frac{1}{3}}+h\frac{\partial\overline{H}_{2}}{\partial v}\left(\frac{x_{\frac{1}{3}}+x_{\frac{2}{3}}}{2},\frac{v_{\frac{1}{3}}+v_{\frac{2}{3}}}{2}\right),\\
	v_{\frac{2}{3}} & =v_{\frac{1}{3}}-h\frac{\partial\overline{H}_{2}}{\partial x}\left(\frac{x_{\frac{1}{3}}+x_{\frac{2}{3}}}{2},\frac{v_{\frac{1}{3}}+v_{\frac{2}{3}}}{2}\right).
\end{align*}
To this end, starting from $x_{2/3}=x_{1/3}$ and $v_{2/3}=v_{1/3}$,
we apply $x_{2/3}\gets x_{1/3}+h\frac{\partial\overline{H}_{2}}{\partial v}\left(\frac{x_{1/3}+x_{2/3}}{2},\frac{v_{1/3}+v_{2/3}}{2}\right)$
and $v_{2/3}\gets v_{1/3}-h\frac{\partial\overline{H}_{2}}{\partial x}\left(\frac{x_{1/3}+x_{2/3}}{2},\frac{v_{1/3}+v_{2/3}}{2}\right)$
iteratively with the following subroutine for computing $\frac{\partial\overline{H}_{2}}{\partial v}$
and $\frac{\partial\overline{H}_{2}}{\partial x}$. According to Lemma
\ref{lem:ham_deri}, this computation involves solving $g(x)^{-1}A^{\top}\left(Ag(x)^{-1}A^{\top}\right)^{-1}Ag(x)^{-1}v$
for some $v$ and $x$. To compute $\left(Ag(x)^{-1}A^{\top}\right)^{-1}Ag(x)^{-1}v$,
we use the Newton's method, which iteratively computes $\nu\gets\nu+M^{-1}Ag(x)^{-1}\left(v-A^{\top}\nu\right)$
for some $M$. Note that the Newton's method guarantees that $\nu$
converges to $M^{-1}Ag(x)^{-1}v$ if $M$ is invertible. Here, we
choose $M=Ag(x^{(t)})^{-1}A^{\top}$ to ensure fast convergence. Since
we have already computed the Cholesky decomposition of $M$ in the
beginning, $M^{-1}Ag(x)^{-1}\left(v-A^{\top}\nu\right)$ can be computed
efficiently by backward and forward substitution. In the third step
of the integrator, we run the process on the Hamiltonian $\overline{H}_{1}$
with step size $\frac{h}{2}$ again to get $(x_{1},v_{1})$, which
results in $x_{1}=x_{2/3}$ and $v_{1}=v_{2/3}-\frac{h}{2}\frac{\partial\overline{H}_{1}}{\partial x}(x_{1}, v_{2/3})$.

We note that CRHMC is affine-invariant and provably independent of condition number (Theorem~\ref{thm:disc-mixing}), and thus the step size and momentum only need to depend on the dimension. 
In practice, we set the momentum to roughly $1-h$, and for the step size $h$, we decrease it until the acceptance probability is close enough to 1 during the warm-up phase. 
Empirically, we found that the step size stays between $0.05$ and $0.2$ in practice even for high dimensional ill-conditioned polytopes. 
This step size is remarkable, given that for these instances a standard package like STAN ends up selecting a small step size like $10^{-8}$ and thus fails to converge.

Putting Algorithm \ref{alg:CHMC_JL} and Algorithm \ref{alg:IMM}
together, we obtain discretization of constrained Riemannian Hamiltonian
Monte Carlo algorithm.

\begin{algorithm2e}[h!]
	
	\caption{\small $\texttt{Discretized Constrained Riemannian Hamiltonian Monte Carlo with Momentum}$}\label{alg:CHMC_JL}
	
	\SetAlgoLined
	
	\textbf{Input:} Initial point $x^{(0)}$, velocity $v^{(0)}$, record
	frequency $T$, step size $h$, ODE steps $K$
	
	\For{$t=1,2,\cdots,T$}{
		
		Let $\overline{v}=v^{(t-1)}$ and $x=x^{(t-1)}.$ 
		
		\tcp{Step 1: Resample $v$ with momentum}
		
		Let $z\sim\mathcal{N}(0,M(x))$. Update $\overline{v}$:
		
		\[
		v\leftarrow\sqrt{\beta}\overline{v}+\sqrt{1-\beta}z.
		\]
		
		\ 
		
		\tcp{Step 2: Solve $\frac{dx}{dt}=\frac{\partial H(x,v)}{\partial v},\frac{dv}{dt}=-\frac{\partial H(x,v)}{\partial x}$
			via the implicit midpoint method}
		
		Use $\texttt{Implicit Midpoint Method}(x,v,h,K)$ to find $(x',v')$
		such that 
		
		\begin{align}
			v_{\frac{1}{3}} & =v-\frac{h}{2}\frac{\partial\overline{H}_{1}(x,v)}{\partial x},\label{eq:mid_sol}\\
			x' & =x+h\frac{\partial\overline{H}_{2}(\frac{x+x'}{2},\frac{v_{1/3}+v_{2/3}}{2})}{\partial v},\ v_{\frac 2 3}=v_{\frac 1 3}-h\frac{\partial\overline{H}_{2}(\frac{x+x'}{2},\frac{v_{1/3}+v_{2/3}}{2})}{\partial x},\nonumber \\
			v' & =v_{\frac{2}{3}}-\frac{h}{2}\frac{\partial\overline{H}_{1}(x',v_{\frac 2 3})}{\partial x}.\nonumber 
		\end{align}
		
		\ 
		
		\tcp{Step 3: Filter}
		
		With probability $\min\left\{ 1,\frac{e^{-H(x',v')}}{e^{-H(x,v)}}\right\} $,
		set $x^{(t)}\leftarrow x'$ and $v^{(t)}\leftarrow v'$.
		
		Otherwise, set $x^{(t)}\leftarrow x$ and $v^{(t)}\leftarrow-v$.
		
	}
	
	\textbf{Output:} $x^{(T)}$
	
\end{algorithm2e}

\begin{algorithm2e}[h!]
	
	\caption{$\texttt{Implicit Midpoint Method}$}\label{alg:IMM}
	
	\SetAlgoLined
	
	\textbf{Input:} Initial point $x$, velocity $v$, step size $h$,
	ODE steps $K$
	
	\tcp{Step 1: Solve $\frac{dx}{dt}=\frac{\partial\overline{H}_{1}(x,v)}{\partial v},\frac{dv}{dt}=-\frac{\partial\overline{H}_{1}(x,v)}{\partial x}$
	}
	
	Set $x_{\frac{1}{3}}\gets x$ and $v_{\frac{1}{3}}\gets v-\frac{h}{2}\frac{\partial\overline{H}_{1}(x,v)}{\partial x}$. 
	
	\ 
	
	\tcp{Step 2: Solve $\frac{dx}{dt}=\frac{\partial\overline{H}_{2}(x,v)}{\partial v},\frac{dv}{dt}=-\frac{\partial\overline{H}_{2}(x,v)}{\partial x}$
		via implicit midpoint}
	
	Set $\nu\gets0$.
	
	\For{$k=1,2,\cdots,K$}{
		
		Let $\xmid\gets\frac{1}{2}\left(x_{\frac{1}{3}}+x_{\frac{2}{3}}\right)$
		and $\vmid\gets\frac{1}{2}\left(v_{\frac{1}{3}}+v_{\frac{2}{3}}\right)$
		
		Set $\nu\gets\nu+\left(LL^{\top}\right)^{-1}Ag(\xmid)^{-1}\left(\vmid-A^{\top}\nu\right)$
		
		Set $x_{\frac{2}{3}}\gets x_{\frac{1}{3}}+hg(\xmid)^{-1}\left(\vmid-A^{\top}\nu\right)$
		
		and $v_{\frac{2}{3}}\gets v_{\frac{1}{3}}+\frac{h}{2}Dg(\xmid)\left[g(\xmid)^{-1}\left(\vmid-A^{\top}\nu\right),g(\xmid)^{-1}\left(\vmid-A^{\top}\nu\right)\right]$
		
	}
	
	\ 
	
	\tcp{Step 3: Solve $\frac{dx}{dt}=\frac{\partial\overline{H}_{1}(x,v)}{\partial v},\frac{dv}{dt}=-\frac{\partial\overline{H}_{1}(x,v)}{\partial x}$
	}
	
	Set $x_{1}\gets x_{\frac{2}{3}}$ and $v_{1}\gets v_{\frac{2}{3}}-\frac{h}{2}\frac{\partial\overline{H}_{1}}{\partial x}(x_{\frac{2}{3}},v_{\frac{2}{3}})$.
	
	\textbf{Output:} $x_{1,}v_{1}$
	
\end{algorithm2e}

\subsection{Theoretical Guarantees\label{subsec:Theoretical-Guarantees}}

In terms of efficiency, we first show that one iteration of Algorithm~\ref{alg:CHMC_JL} incurs the cost of solving a few Cholesky decomposition and $O(K)$ sparse triangular systems.
We also show in Lemma~\ref{lem:imm_correctness} that the implicit midpoint integrator converges to the solution of Eq.~(\ref{eq:mid_sol}) in logarithmically many iterations.
Regarding correctness, Theorem~\ref{thm:balance} and Lemma~\ref{lem:imm_correctness} together show that the discretized CRHMC (Algorithm~\ref{alg:CHMC_JL}) converges to the stationary distribution indeed (see Remark~\ref{rem:correctness}).
\begin{thm}
	The cost of each iteration of Algorithm \ref{alg:CHMC_JL} is solving
	$O(1)$ Cholesky decomposition and $O(K)$ triangular systems, where
	$K$ is the number of iterations in Algorithm \ref{alg:IMM}.
\end{thm}

\begin{proof}
	We first solve the Cholesky decomposition to get $L_{t-1}L_{t-1}^{\top}=Ag(x^{(t-1)})^{-1}A^{\top}$
	at the beginning of iteration. Recall that 
	\begin{align*}
		H(x,v) & =\overline{H}_{1}(x,v)+\overline{H}_{2}(x,v)\\
		& =\left(f(x)+\frac{1}{2}(\log\det g(x)+\log\det Ag(x)^{-1}A^{\top})\right)\\
		& \quad\ +\left(\frac{1}{2}v^{\top}g(x)^{-\frac{1}{2}}\left(I-g(x)^{-\frac{1}{2}}A^{\top}(Ag(x)^{-1}A^{\top})^{-1}Ag(x)^{-\frac{1}{2}}\right)g(x)^{-\frac{1}{2}}v\right).
	\end{align*}
	The value of $H(x^{(t-1)},v^{(t-1)})$ should be computed later for
	the filter step and can be efficiently computed by the given $L_{t-1}L_{t-1}^{\top}=Ag(x^{(t-1)})^{-1}A^{\top}$and
	solving two sparse triangular systems (i.e., $L_{t-1}^{-\top}(L_{t-1}^{-1}(Ag(x)^{-\half}))$).
	We need the same cost (i.e., Cholesky decomposition and solving two
	triangular systems) for the value of $H(x',v')$, where $(x',v')$
	is the output of Algorithm \ref{alg:IMM}. We note that $L$ inherits
	sparsity of $A$ and thus each triangular system can be solved efficiently
	by backward and forward substitution.
	
	In the implicit midpoint method, one main component is computation
	of $\frac{\partial\overline{H}_{1}(x,v)}{\partial x}$ in Step 1 and
	$\frac{\partial\overline{H}_{1}}{\partial x}(x_{\frac{2}{3}},v_{\frac{2}{3}})$
	in Step 3 due to leverage scores. As seen in Section \ref{subsec:Efficient-Computation-of},
	the cost for these computations is within a constant factor of solving
	the Cholesky decomposition for $Ag(x^{(t-1)})^{-1}A^{\top}$ and $Ag(x_{\frac{2}{3}})^{-1}A^{\top}$.
	Another component is solving $O(K)$ triangular systems to update
	$\nu$ in Step 2.
	
	Adding up all these costs, each iteration of Algorithm \ref{alg:CHMC_JL}
	only requires solving $O(1)$ Cholesky decomposition and $O(K)$ sparse
	triangular systems.
\end{proof}
\begin{thm}
	\label{thm:balance}The Markov chain defined by Algorithm \ref{alg:CHMC_JL}
	projected to $x$ has a stationary density proportional to $\exp(-f(x))$, and is irreducible and aperiodic.
	Therefore, this Markov chain converges to the stationary distribution.
\end{thm}

\begin{proof}
	Each iteration consists of two stages: resampling velocity with momentum
	in Step 1 (i.e., $(x,\overline{v})$ to $(x,v)$) and solving ODE
	followed by the filter in Step 2 and 3 (i.e., $(x,v)$ to $(x',v')$).
	To prove the claim, we show that Step 1 is time-reversible with respect
	to the conditional distribution $\pi(v|x)$ and that Step 2 followed
	by Step 3 is also time-reversible with respect to $\pi(x,v)$.
	
	We begin with the first part. We have $\pi(\overline{v}|x)=\mathcal{N}(0,M(x))$
	due to the definition of $H$. Since $\overline{v}|x\sim\mathcal{N}(0,M(x))$
	and $z\sim\mathcal{N}(0,M(x))$ are independent Gaussians, the update
	rule $v=\sqrt{\beta}\overline{v}+\sqrt{1-\beta}z$ implies $\pi(v|x)=\mathcal{N}(0,M(x))$.
	Let $\P(z)$ be the probability density and $C$ be the normalization
	constant for Gaussian $\mathcal{N}(0,M(x))$. Then, the time-reversibility
	w.r.t. $\pi(v|x)$ is immediate from the following computation:
	\begin{align*}
		\pi(\overline{v}|x)\P(\overline{v}\to v) & =C^{2}\exp(-\frac{1}{2}\overline{v}^{\top}M^{\dagger}\overline{v})\cdot\exp(-\frac{1}{2}\frac{(v-\sqrt{\beta}\overline{v})^{\top}M^{\dagger}(v-\sqrt{\beta}\overline{v})}{1-\beta})\\
		& =C^{2}\exp\text{\ensuremath{\left(-\frac{1}{2}\left(\overline{v}^{\top}M^{\dagger}\overline{v}+\frac{v^{\top}M^{\dagger}v}{1-\beta}+\frac{\beta\overline{v}^{\top}M^{\dagger}\overline{v}}{1-\beta}-\frac{\sqrt{\beta}}{1-\beta}(\overline{v}^{\top}M^{\dagger}v+v^{\top}M^{\dagger}\overline{v})\right)\right)}}\\
		& =C^{2}\exp\left(-\frac{1}{2}\left(\frac{v^{\top}M^{\dagger}v}{1-\beta}+\frac{\overline{v}^{\top}M^{\dagger}\overline{v}}{1-\beta}-\frac{\sqrt{\beta}}{1-\beta}(\overline{v}^{\top}M^{\dagger}v+v^{\top}M^{\dagger}\overline{v})\right)\right),\\
		\pi(v|x)\P(v\to\overline{v}) & =C^{2}\exp(-\frac{1}{2}v^{\top}M^{\dagger}v)\cdot\exp(-\frac{1}{2}\frac{(\overline{v}-\sqrt{\beta}v)^{\top}M^{\dagger}(\overline{v}-\sqrt{\beta}v)}{1-\beta})\\
		& =C^{2}\exp\text{\ensuremath{\left(-\frac{1}{2}\left(v^{\top}M^{\dagger}v+\frac{\overline{v}^{\top}M^{\dagger}\overline{v}}{1-\beta}+\frac{\beta v^{\top}M^{\dagger}v}{1-\beta}-\frac{\sqrt{\beta}}{1-\beta}(\overline{v}^{\top}M^{\dagger}v+v^{\top}M^{\dagger}\overline{v})\right)\right)}}\\
		& =C^{2}\exp\left(-\frac{1}{2}\left(\frac{v^{\top}M^{\dagger}v}{1-\beta}+\frac{\overline{v}^{\top}M^{\dagger}\overline{v}}{1-\beta}-\frac{\sqrt{\beta}}{1-\beta}(\overline{v}^{\top}M^{\dagger}v+v^{\top}M^{\dagger}\overline{v})\right)\right)\\
		\Longrightarrow & \quad\pi(\overline{v}|x)\P(\overline{v}\to v)=\pi(v|x)\P(v\to\overline{v}).
	\end{align*}
	
	The second part follows from a stronger statement due to symmetry
	of $v$ in $H(x,v)$: In the space where $(x,v)$ and $(x,-v)$ are
	identified, the Markov chain defined by Step 2 and 3 satisfies detailed
	balance with respect the density $\pi([x,v])$ proportional to $\exp(-H(x,v))$,
	where $[x,v]$ denotes the identified point for $(x,v)$ and $(x,-v)$.
	Consider the pairs $[x,v]=\left\{ (x,v),(x,-v)\right\} $ and $[x',v']=\left\{ (x',v'),(x',-v')\right\} $
	where in Step 2 $(x,v)$ goes to $(x',v')$ and $(x',-v')$ goes to
	$(x,-v)$ due to reversibility of the implicit midpoint method. We
	now verify that the filtering probability is the same in either direction,
	using the measure-preserving property of Step 2
	\begin{align*}
		\pi(x,v)\P\left((x,v)\rightarrow(x',v')\right) & =\pi(x,v)\min\left\{ 1,\frac{\pi(x',v')}{\pi(x,v)}\right\} \\
		& =\min\left\{ \pi(x,v),\pi(x',v')\right\} \\
		& =\min\left\{ \pi(x,-v),\pi(x',-v')\right\} \\
		& =\pi(x',-v')\min\left\{ 1,\frac{\pi(x,-v)}{\pi(x',-v')}\right\} \\
		& =\pi(x',-v')\P\left((x',-v')\rightarrow(x,-v)\right).
	\end{align*}
	Therefore, for any two pairs $[x,v]$ and $[x',v']$, we have $\pi([x,v])\P\left([x,v]\to[x',v']\right)=\pi([x',v']\P\left([x',v']\to[x,v]\right)$, and thus this detailed balance implies that the target density is stationary.
	
	Its irreducibility is implied by the non-zero lower bound on the conductance of the discretized CRHMC (Theorem~\ref{thm:disc-mixing}).
	To see this, let $A$ and $B$ be two subsets of positive measure such that one subset is not reachable from another in infinitely many steps. 
	Take the set $R$ of reachable points from $A$ via running the Markov chain, and note that $R$ and $R^c(\supseteq B)$ have non-zero measures. However, the non-zero conductance, meaning that there must be a positive probability of stepping out of $R$, which contradicts the definition of $R$. 
	Now for aperiodicity, as assumed at the beginning of the mixing rate proof (Appendix~\ref{subsec:Choosing}), we consider a lazy version of the discretized CRHMC instead, which makes the chain stay where it is at with probability $1/2$ at each iteration, which prevents potential periodicity of the process. 
	Note that this modification worsens the mixing rate only by a factor of 2.
	
	Putting these three together, we can show that the discretized CRHMC converges to the target distribution.
\end{proof}

Now we show in Lemma \ref{lem:imm_correctness} that the implicit midpoint method (Algorithm~\ref{alg:IMM}) converges to the solution of (\ref{eq:mid_sol}) in logarithmically many iterations. 
To show the convergence of Algorithm \ref{alg:IMM}, we denote by $\T$ the map induced by one iteration of Step 2.
\begin{definition}
	Let
	\[
	\T(x,v,\nu)=\left(\begin{array}{c}
		x_{\frac{1}{3}}+hg(\xmid)^{-1}(\vmid-A^{\top}\lambda_{1})\\
		v_{\frac{1}{3}} + \frac{h}{2}Dg(\xmid)[g(\xmid)^{-1}(\vmid-A^{\top}\lambda_{1}),g(\xmid)^{-1}(\vmid-A^{\top}\lambda_{1})]\\
		\lambda_{1}
	\end{array}\right),
	\]
	where $\xmid=\frac{1}{2}(x_{\frac{1}{3}}+x)$, $\vmid=\frac{1}{2}(v_{\frac{1}{3}}+v)$,
	and $\lambda_{1}=\nu+(LL^{\top})^{-1}Ag(\xmid)^{-1}\left(\vmid-A^{\top}\nu\right)$.
	Let $(x_{\frac{2}{3}}^{*},v_{\frac{2}{3}}^{*},\nu^{*})$ be the fixed
	point of $\T$.
\end{definition}

We assume that $g$ is given by the Hessian of a highly self-concordant barrier $\phi$ (see \ref{def:selfcon}).
Note that the log-barrier is highly self-concordant. 
We can show that for small enough step size $h$, Algorithm \ref{alg:IMM} can solve (\ref{eq:mid_sol}) to $\delta$-accuracy in logarithmically many iterations. 

\begin{lem}
	\label{lem:imm_correctness}Suppose $g(x)=\nabla^{2}\phi(x)$ for
	some highly self-concordant barrier $\phi$. For any input $(x_{\frac{1}{3}},v_{\frac{1}{3}})$,
	let $(x_{\frac{2}{3}}^{(k)},v_{\frac{2}{3}}^{(k)},\nu^{(k)})$ be
	points obtained after $k$ iterations in Step 2 of Algorithm $\text{\ref{alg:IMM}}$.
	Let $(\tx_{\frac{2}{3}},\tv_{\frac{2}{3}})$ be the solution for $(x_{\frac{2}{3}},v_{\frac{2}{3}})$
	in the following equation
	\[
	x_{\frac{2}{3}}=x_{\frac{1}{3}}+h\frac{\partial\overline{H}_{2}}{\partial v}\left(\frac{x_{\frac{1}{3}}+x_{\frac{2}{3}}}{2},\frac{v_{\frac{1}{3}}+v_{\frac{2}{3}}}{2}\right),\ v_{\frac{2}{3}}=v_{\frac{1}{3}}-h\frac{\partial\overline{H}_{2}}{\partial x}\left(\frac{x_{\frac{1}{3}}+x_{\frac{2}{3}}}{2},\frac{v_{\frac{1}{3}}+v_{\frac{2}{3}}}{2}\right).
	\]
	Let $\norm x_{A}:=\sqrt{x^{\top}Ax}$ for a matrix $A$. For any $(x,v,\nu)$,
	define the norm 
	\[
	\|(x,v,\lambda)\|:=\|x\|_{g(x_{\frac{1}{3}})}+\|v\|_{g(x_{\frac{1}{3}})^{-1}}+h\|A^{\top}\nu\|_{g(x_{\frac{1}{3}})^{-1}}.
	\]
	If $\left\Vert (x_{\frac{2}{3}}^{(0)},v_{\frac{2}{3}}^{(0)},\nu^{(0)})-(\tx_{\frac{2}{3}},\tv_{\frac{2}{3}},\nu^{*})\right\Vert \leq r$
	with $h\leq r\leq\min(\frac{1}{10},\frac{\sqrt{h}}{4},\frac{\|v^{*}\|_{g(x_{0})^{-1}}}{4})$,
	then 
	
	\[
	\left\Vert (x_{\frac{2}{3}}^{(L)},v_{\frac{2}{3}}^{(L)},\nu^{(L)})-(\tx_{\frac{2}{3}},\tv_{\frac{2}{3}},\nu^{*})\right\Vert \leq\delta
	\]
	for some $L=O\left(\log_{1/C}\frac{r}{\delta}\right)$, where $C=O_{n}(h)$
	is the Lipschitz constant of the map $\T$.
\end{lem}

\begin{proof}
	Since $(x_{\frac{2}{3}}^{*},v_{\frac{2}{3}}^{*},\nu^{*})$ is the
	fixed point of $\T$ (i.e., $\nu^{*}=\lambda_{1}$), we have 
	\[
	\nu^{*}=\nu^{*}+(LL^{\top})^{-1}Ag(\xmid)^{-1}\left(\vmid-A^{\top}\nu^{*}\right)
	\]
	and thus $Ag(\xmid)^{-1}\vmid=Ag(\xmid)^{-1}A^{\top}\nu^{*}$. For
	invertible $Ag(\xmid)^{-1}A^{\top}$, we have 
	\[
	\nu^{*}=\left(Ag(\xmid)^{-1}A^{\top}\right)^{-1}Ag(\xmid)^{-1}\vmid.
	\]
	Similarly by using the definition of the fixed point and this new
	formula for $\nu^{*}$,
	\begin{align*}
		x_{\frac{2}{3}}^{*} & =x_{\frac{1}{3}}+hg(\xmid)^{-1}\vmid-hg(\xmid)^{-1}A^{\top}\nu^{*}\\
		& =x_{\frac{1}{3}}+hg(\xmid)^{-1}\vmid-hg(\xmid)^{-1}A^{\top}\left(Ag(\xmid)^{-1}A^{\top}\right)^{-1}Ag(\xmid)^{-1}\vmid\\
		& =x_{\frac{1}{3}}+h\frac{\partial\overline{H}_{2}}{\partial v}\left(\xmid,\vmid\right)
	\end{align*}
	and
	\begin{align*}
		v_{\frac{2}{3}}^{*} & =v_{\frac{1}{3}} + \frac{h}{2}Dg(\xmid)[g(\xmid)^{-1}(\vmid-A^{\top}\nu^{*}),g(\xmid)^{-1}(\vmid-A^{\top}\nu^{*})]\\
		& =v_{\frac{1}{3}}-h\frac{\partial\overline{H}_{2}}{\partial x}\left(\xmid,\vmid\right)
	\end{align*}
	which shows that $(x_{\frac{2}{3}}^{*},v_{\frac{2}{3}}^{*})$ is
	exactly the solution for $(x,v)$ in the equation 
	\[
	x=x_{\frac{1}{3}}+h\frac{\partial\overline{H}_{2}}{\partial v}\left(\frac{x_{\frac{1}{3}}+x}{2},\frac{v_{\frac{1}{3}}+v}{2}\right),\ v=v_{\frac{1}{3}}-h\frac{\partial\overline{H}_{2}}{\partial x}\left(\frac{x_{\frac{1}{3}}+x}{2},\frac{v_{\frac{1}{3}}+v}{2}\right).
	\]
	Next, we show that the iterations in Step 2 converges to $(x_{\frac{2}{3}}^{*},v_{\frac{2}{3}}^{*},\nu^{*})$.
	If $\left\Vert (x_{\frac{2}{3}}^{(0)},v_{\frac{2}{3}}^{(0)},\nu^{(0)})-(x_{\frac{2}{3}}^{*},v_{\frac{2}{3}}^{*},\nu^{*})\right\Vert \leq r$
	for some $C=O_{n}(h)$, we have
	\begin{align*}
		\left\Vert (x_{\frac{2}{3}}^{(\ell)},v_{\frac{2}{3}}^{(\ell)},\nu^{(\ell)})-(x_{\frac{2}{3}}^{*},v_{\frac{2}{3}}^{*},\nu^{*})\right\Vert  & =\left\Vert \T(x_{\frac{2}{3}}^{(\ell-1)},v_{\frac{2}{3}}^{(\ell-1)},\nu^{(\ell)})-\T(x_{\frac{2}{3}}^{*},v_{\frac{2}{3}}^{*},\nu^{*})\right\Vert \\
		& \leq C\left\Vert (x_{\frac{2}{3}}^{(\ell-1)},v_{\frac{2}{3}}^{(\ell-1)},\nu^{(\ell-1)})-(x_{\frac{2}{3}}^{*},v_{\frac{2}{3}}^{*},\nu^{*})\right\Vert \\
		& \leq C^{\ell}\left\Vert (x_{\frac{2}{3}}^{(0)},v_{\frac{2}{3}}^{(0)},\nu^{(0)})-(x_{\frac{2}{3}}^{*},v_{\frac{2}{3}}^{*},\nu^{*})\right\Vert ,
	\end{align*}
	where the first equality follows from $(x_{\frac{2}{3}}^{*},v_{\frac{2}{3}}^{*},\nu^{*})$
	is the fixed point of $\T$ and the second inequality follows from
	Lemma \ref{lem:IMM_contraction}. Therefore, we have $\left\Vert (x_{\frac{2}{3}}^{(L)},v_{\frac{2}{3}}^{(L)},\nu^{(L)})-(x_{\frac{2}{3}}^{*},v_{\frac{2}{3}}^{*},\nu^{*})\right\Vert \leq\delta$
	for $L=O\left(\log_{C}\frac{r}{\delta}\right).$
\end{proof}

\begin{rem}
	\label{rem:correctness}Lemma \ref{lem:imm_correctness} shows that
	Algorithm \ref{alg:IMM} converges to the solution of (\ref{eq:mid_sol})
	in logarithmically many iterations for small enough step size $h$.
	In Step 1 of Algorithm \ref{alg:CHMC_JL}, $v$ is resampled so that
	every iteration of Algorithm \ref{alg:CHMC_JL} is a non-degenerate
	map. Then, the total variation distance between the distributions
	generated by solving (\ref{eq:mid_sol}) using Algorithm \ref{alg:IMM}
	and solving (\ref{eq:mid_sol}) exactly in one iteration of Algorithm
	\ref{alg:CHMC_JL} can be bounded by error due to Algorithm \ref{alg:IMM}.
	Theorem \ref{thm:balance} shows that the process will converge
	to the exact stationary distribution. Therefore, in order for the
	accumulated error of Algorithm \ref{alg:CHMC_JL} to remain bounded
	for polynomially many steps, it suffices to run logarithmically many
	iterations in Algorithm \ref{alg:IMM}. Any small bias due to the
	numerical error in the ODE computation is corrected by the filter,
	and maintaining as small error as possible is important to keep the
	acceptance probability high.
\end{rem}

\subsection{Deferred Proof}

\begin{lem}[\cite{kook2022conditionnumberindependent}, Lemma 28]
\label{lem:sc_facts}Suppose $g(x)=\nabla^{2}\phi(x)$ for some highly
self-concordance barrier $\phi$. Then, we have that
\end{lem}

\begin{itemize}
\item $(1-\|y-x\|_{g(x)})^{2}g(x)\preceq g(y)\preceq\frac{1}{(1-\|y-x\|_{g(x)})^{2}}g(x)$.
\item $\|Dg(x)[v,v]\|_{g(x)^{-1}}\leq2\|v\|_{g(x)}^{2}$.
\item $\|Dg(x)[v,v]-Dg(y)[v,v]\|_{g(x)^{-1}}\leq\frac{6}{(1-\|y-x\|_{g(x)})^{3}}\|v\|_{g(x)}^{2}\|y-x\|_{g(x)}$.
\item $\norm{Dg(x)[v,v]-Dg(x)[w,w]}_{g(x)^{-1}} \leq 2\norm{v-w}_{g(x)}\norm{v+w}_{g(x)}$.
\end{itemize}

\begin{lem}
\label{lem:IMM_contraction} Let $g(x)=\nabla^{2}\phi(x)$ for some
highly self-concordance barrier $\phi$. Given $x_{0},v_{0}$ and
$L$ such that $LL^{\top}=Ag(x_{0})^{-1}A^{\top}$, consider the map
\[
\T(x,v,\lambda)=\left(\begin{array}{c}
x_{0}+hg(x_{1/2})^{-1}(v_{1/2}-A^{\top}\lambda_{1})\\
v_{0}+\frac{h}{2}Dg(x_{1/2})[g(x_{1/2})^{-1}(v_{1/2}-A^{\top}\lambda_{1}),g(x_{1/2})^{-1}(v_{1/2}-A^{\top}\lambda_{1})]\\
\lambda_{1}
\end{array}\right)
\]
where $x_{1/2}=(x_{0}+x)/2$, $v_{1/2}=(v_{0}+v)/2$ and $\lambda_{1}=\lambda+(LL^{\top})^{-1}Ag(x_{1/2})^{-1}\left(v_{1/2}-A^{\top}\lambda\right)$.
Let $(x^{*},v^{*},\lambda^{*})$ be a fixed point of $\T$. For any
$x,v,\lambda$, we define the norm
\[
\|(x,v,\lambda)\|=\|x\|_{g(x_{0})}+\|v\|_{g(x_{0})^{-1}}+h\|A^{\top}\lambda\|_{g(x_{0})^{-1}}.
\]
Let $\Omega=\{(x,v,\lambda):\|(x,v,\lambda)-(x^{*},v^{*},\lambda^{*})\|\leq r\}$
with $h\leq r\leq\min(\frac{1}{10},\frac{\sqrt{h}}{4},\frac{\|v^{*}\|_{g(x_{0})^{-1}}}{4})$.
Suppose that $(x_{0},v_{0},0)\in\Omega$. Then, for any $(x,v,\lambda),(\overline{x},\overline{v},\overline{\lambda})\in\Omega$,
we have
\[
\|\tcal(x,v,\lambda)-\tcal(\overline{x},\overline{v},\overline{\lambda})\|\leq C\|(x,v,\lambda)-(\overline{x},\overline{v},\overline{\lambda})\|
\]
where $C=(\frac{3r}{h}+\|v^{*}\|_{g(x_{0})^{-1}})(400r+18h\|v^{*}\|_{g(x_{0})^{-1}})$.
\end{lem}

\begin{rem}
Note that we should think $r=\Theta_{n}(h)$ because that is the distance
between $(x_{0},v_{0},0)$ and $(x^{*},v^{*},\lambda^{*})$. In that
case, the Lipschitz constant of $\T$ is $O_{n}(h\|v^{*}\|_{g(x_{0})^{-1}}^{2})=O_{n}(h)$.
Hence, if the step size $h$ is small enough, then $\T$ is a contractive
mapping. In practice, we can take $h$ close to a constant because
$g$ is decomposable into barriers in each dimension and the bound
can be improved using this.
\end{rem}

\begin{proof}
We use $\T(x,v,\lambda)_{x}$ to denote the $x$ component of $\T(x,v,\lambda)$
and similarly for $\T(x,v,\lambda)_{v}$ and $\T(x,v,\lambda)_{\lambda}$.
For simplicity, we write $g_{0}=g(x_{0})$, $g_{1/2}=g(x_{1/2})$
and $\overline{g}_{1/2}=g(\overline{x_{1/2}})$. By the assumption,
we have that
\[
\|x-x_{0}\|_{g_{0}}\leq\|x-x^{*}\|_{g_{0}}+\|x^{*}-x_{0}\|_{g_{0}}\leq2r.
\]
Similarly, $\|\overline{x}-x_{0}\|_{g_{0}}\leq2r$.

We first bound $\T(x,v,\lambda)_{\lambda}$. Note that
\[
\T(\overline{x},\overline{v},\overline{\lambda})_{\lambda}-\T(x,v,\lambda)_{\lambda}=\alpha_{1}+\alpha_{2}+\alpha_{3}+\alpha_{4}
\]
where 
\begin{align*}
\alpha_{1} & =(I-(LL^{\top})^{-1}Ag_{0}^{-1}A^{\top})(\overline{\lambda}-\lambda),\\
\alpha_{2} & =(LL^{\top})^{-1}Ag_{0}^{-1}(\overline{v_{1/2}}-v_{1/2}),\\
\alpha_{3} & =(LL^{\top})^{-1}A(g_{1/2}^{-1}-g_{0}^{-1})((\overline{v_{1/2}}-A^{\top}\overline{\lambda})-(v_{1/2}-A^{\top}\lambda)),\\
\alpha_{4} & =(LL^{\top})^{-1}A(\overline{g}_{1/2}^{-1}-g_{1/2}^{-1})(\overline{v_{1/2}}-A^{\top}\overline{\lambda}).
\end{align*}
Using that $LL^{\top}=Ag(x_{0})^{-1}A^{\top}$, we have $\alpha_{1}=0$.
For $\alpha_{2}$, we have
\begin{align*}
\|A^{\top}\alpha_{2}\|_{g_{0}^{-1}}^{2} & =(\overline{v_{1/2}}-v_{1/2})^{\top}g_{0}^{-1}A^{\top}(LL^{\top})^{-1}Ag_{0}^{-1}A^{\top}(L^{\top}L)^{-1}Ag_{0}^{-1}(\overline{v_{1/2}}-v_{1/2})\\
 & =(\overline{v_{1/2}}-v_{1/2})^{\top}g_{0}^{-1}A^{\top}(Ag_{0}^{-1}A^{\top})^{-1}Ag_{0}^{-1}(\overline{v_{1/2}}-v_{1/2})\\
 & \leq(\overline{v_{1/2}}-v_{1/2})^{\top}g_{0}^{-1}(\overline{v_{1/2}}-v_{1/2})\\
 & =\frac{1}{4}\|\overline{v}-v\|_{g_{0}^{-1}}^{2}
\end{align*}
where we use $LL^{\top}=Ag(x_{0})^{-1}A^{\top}$ and $g_{0}^{-1/2}A^{\top}(Ag_{0}^{-1}A^{\top})^{-1}Ag_{0}^{-1/2}=B^{\top}(BB^{\top})^{-1}B\preceq I$
for $B=Ag_{0}^{-1/2}$. For $\alpha_{3}$, by self-concordance of
$g$ (Lemma \ref{lem:sc_facts}) and $\|x-x_{0}\|_{g_{0}}\leq2r$,
we have
\begin{equation}
(1-r)^{2}g_{0}\preceq g_{1/2}\preceq\frac{1}{(1-r)^{2}}g_{0}\label{eq:g_bound}
\end{equation}
and hence $(g_{0}^{1/2}(g_{1/2}^{-1}-g_{0}^{-1})g_{0}^{1/2})^{2}\preceq((1-r)^{-2}-1)^{2}I.$
Using this and $P=g_{0}^{-1/2}A^{\top}(Ag_{0}^{-1}A^{\top})^{-1}Ag_{0}^{-1/2}\preceq I$,
we have
\begin{align*}
\|A^{\top}\alpha_{3}\|_{g_{0}^{-1}} & =\|g_{0}^{1/2}(g_{1/2}^{-1}-g_{0}^{-1})((\overline{v_{1/2}}-A^{\top}\overline{\lambda})-(v_{1/2}-A^{\top}\lambda))\|_{P}\\
 & \leq\|g_{0}^{1/2}(g_{1/2}^{-1}-g_{0}^{-1})((\overline{v_{1/2}}-A^{\top}\overline{\lambda})-(v_{1/2}-A^{\top}\lambda))\|_{2}\\
 & \leq((1-r)^{-2}-1)\|g_{0}^{-1/2}((\overline{v_{1/2}}-A^{\top}\overline{\lambda})-(v_{1/2}-A^{\top}\lambda))\|_{2}\\
 & \leq((1-r)^{-2}-1)(\frac{1}{2}\|\overline{v}-v\|_{g_{0}^{-1}}+\|A^{\top}(\overline{\lambda}-\lambda)\|_{g_{0}^{-1}}).
\end{align*}
Using $r\leq1/10$, we have 
\[
\|A^{\top}\alpha_{3}\|_{g_{0}^{-1}}\leq1.2r\|\overline{v}-v\|_{g_{0}^{-1}}+2.4r\|A^{\top}(\overline{\lambda}-\lambda)\|_{g_{0}^{-1}}.
\]
For $\alpha_{4}$, similarly, we have
\begin{align*}
\|A^{\top}\alpha_{4}\|_{g_{0}^{-1}} & \leq((1-0.5\|\overline{x}-x\|_{g_{1/2}})^{-2}-1)\|\overline{v_{1/2}}-A^{\top}\overline{\lambda}\|_{g_{0}^{-1}}\\
 & \leq((1-0.6\|\overline{x}-x\|_{g_{0}})^{-2}-1)\|\overline{v_{1/2}}-A^{\top}\overline{\lambda}\|_{g_{0}^{-1}}\\
 & \leq1.5\|\overline{x}-x\|_{g_{0}}\|\overline{v_{1/2}}-A^{\top}\overline{\lambda}\|_{g_{0}^{-1}}
\end{align*}
where we used $g_{1/2}\preceq1.2g_{0}$ (by (\ref{eq:g_bound})) in
the second inequality and $\|\overline{x}-x\|_{g_{0}}\leq\|\overline{x}-x^{*}\|_{g_{0}}+\|x-x^{*}\|_{g_{0}}\leq\frac{1}{5}$
at the end. Combining everything, we have
\begin{align}
 & \|A^{\top}(\T(\overline{x},\overline{v},\overline{\lambda})_{\lambda}-\T(x,v,\lambda)_{\lambda})\|_{g_{0}^{-1}}=\|A^{\top}(\overline{\lambda}_{1}-\lambda_{1})\|_{g_{0}^{-1}}\nonumber \\
 & \qquad\leq0.7\|\overline{v}-v\|_{g_{0}^{-1}}+2.4r\|A^{\top}(\overline{\lambda}-\lambda)\|_{g_{0}^{-1}}+1.5\|\overline{x}-x\|_{g_{0}}\|\overline{v_{1/2}}-A^{\top}\overline{\lambda}\|_{g_{0}^{-1}}.\label{eq:contractive_Alambda1}
\end{align}

Now we bound $\tcal(x,v,\lambda)_{x}$. Note that
\[
\T(\overline{x},\overline{v},\overline{\lambda})_{x}-\T(x,v,\lambda)_{x}=h\beta_{1}+h\beta_{2}
\]
where
\begin{align*}
\beta_{1}= & g_{1/2}^{-1}((\overline{v}_{1/2}-A^{\top}\overline{\lambda}_{1})-(v_{1/2}-A^{\top}\lambda_{1})),\\
\beta_{2}= & (\overline{g}_{1/2}^{-1}-g_{1/2}^{-1})(\overline{v}_{1/2}-A^{\top}\overline{\lambda}_{1}).
\end{align*}
By a proof similar to above, we have
\begin{align*}
\|\beta_{1}\|_{g_{0}} & \leq1.2(\|\overline{v}_{1/2}-v_{1/2}\|_{g_{0}^{-1}}+\|A^{\top}(\overline{\lambda}_{1}-\lambda_{1})\|_{g_{0}^{-1}}),\\
\|\beta_{2}\|_{g_{0}} & \leq0.6\|\overline{x}-x\|_{g_{0}}\|\overline{v}_{1/2}-A^{\top}\overline{\lambda}_{1}\|_{g_{0}^{-1}}.
\end{align*}
and thus
\begin{align*}
 & \|\tcal(\overline{x},\overline{v},\overline{\lambda})_{x}-\tcal(x,v,\lambda)_{x}\|_{g_{0}}\\
 & \qquad\leq0.6h\|\overline{v}-v\|_{g_{0}^{-1}}+1.2h\|A^{\top}(\overline{\lambda}_{1}-\lambda_{1})\|_{g_{0}^{-1}}+0.6h\|\overline{x}-x\|_{g_{0}}\|\overline{v}_{1/2}-A^{\top}\overline{\lambda}_{1}\|_{g_{0}^{-1}}.
\end{align*}

Finally, we bound $\tcal(x,v,\lambda)_{v}$. We split the term
\[
\tcal(\overline{x},\overline{v},\overline{\lambda})_{v}-\tcal(x,v,\lambda)_{v}=\frac{h}{2}\gamma_{1}+\frac{h}{2}\gamma_{2}
\]
where 
\begin{align*}
\gamma_{1}= & Dg(x_{1/2})[\overline{g}_{1/2}^{-1}(\overline{v}_{1/2}-A^{\top}\overline{\lambda}_{1}),\overline{g}_{1/2}^{-1}(\overline{v}_{1/2}-A^{\top}\overline{\lambda}_{1})]\\
 & -Dg(x_{1/2})[g_{1/2}^{-1}(v_{1/2}-A^{\top}\lambda_{1}),g_{1/2}^{-1}(v_{1/2}-A^{\top}\lambda_{1})],\\
\gamma_{2}= & Dg(\overline{x}_{1/2})[\overline{g}_{1/2}^{-1}(\overline{v}_{1/2}-A^{\top}\overline{\lambda}_{1}),\overline{g}_{1/2}^{-1}(\overline{v}_{1/2}-A^{\top}\overline{\lambda}_{1})]\\
 & -Dg(x_{1/2})[\overline{g}_{1/2}^{-1}(\overline{v}_{1/2}-A^{\top}\overline{\lambda}_{1}),\overline{g}_{1/2}^{-1}(\overline{v}_{1/2}-A^{\top}\overline{\lambda}_{1})].
\end{align*}
Let $\overline{\eta}=\overline{g}_{1/2}^{-1}(\overline{v}_{1/2}-A^{\top}\overline{\lambda}_{1})$
and $\eta=g_{1/2}^{-1}(v_{1/2}-A^{\top}\lambda_{1})$. For $\gamma_{1}$,
we have that
\begin{align*}
 & \|Dg(x_{1/2})[\overline{\eta},\overline{\eta}]-Dg(x_{1/2})[\eta,\eta]\|_{g_{1/2}^{-1}}\\
 & \quad\leq2\|Dg(x_{1/2})[\overline{\eta}-\eta,\overline{\eta}]\|_{g_{1/2}^{-1}}+\|Dg(x_{1/2})[\overline{\eta}-\eta,\overline{\eta}-\eta]\|_{g_{1/2}^{-1}}\\
 & \quad\leq4\|\overline{\eta}-\eta\|_{g_{1/2}}\|\overline{\eta}\|_{g_{1/2}}+2\|\overline{\eta}-\eta\|_{g_{1/2}}^{2}
\end{align*}
where we use Lemma \ref{lem:sc_facts}. Using $g_{1/2}\preceq1.2g_{0}$
(by (\ref{eq:g_bound})),
\begin{align*}
\|\gamma_{1}\|_{g_{0}^{-1}}\leq & 4\|(\overline{v}_{1/2}-A^{\top}\overline{\lambda}_{1})-(v_{1/2}-A^{\top}\lambda_{1})\|_{g_{0}^{-1}}\|\overline{v}_{1/2}-A^{\top}\overline{\lambda}_{1}\|_{g_{0}^{-1}}\\
 & +2\|(\overline{v}_{1/2}-A^{\top}\overline{\lambda}_{1})-(v_{1/2}-A^{\top}\lambda_{1})\|_{g_{0}^{-1}}^{2}.
\end{align*}
For $\gamma_{2}$, we use Lemma \ref{lem:sc_facts} and get
\begin{align*}
\|\gamma_{2}\|_{g_{0}^{-1}} & \leq\frac{4}{(1-0.6\|\ox-x\|_{g_{0}})^{3}}\|\overline{v}_{1/2}-A^{\top}\overline{\lambda}_{1}\|_{g_{0}^{-1}}^{2}\|\ox-x\|_{g_{0}}\\
 & \leq6\|\overline{v}_{1/2}-A^{\top}\overline{\lambda}_{1}\|_{g_{0}^{-1}}^{2}\|\ox-x\|_{g_{0}}.
\end{align*}
Combining everything, we have
\begin{align*}
 & \|\T(\overline{x},\overline{v},\overline{\lambda})_{v}-\T(x,v,\lambda)_{v}\|_{g_{0}^{-1}}\\
 & \quad\leq2h\|(\overline{v}_{1/2}-A^{\top}\overline{\lambda}_{1})-(v_{1/2}-A^{\top}\lambda_{1})\|_{g_{0}^{-1}}\|\overline{v}_{1/2}-A^{\top}\overline{\lambda}_{1}\|_{g_{0}^{-1}}\\
 & \qquad+h\|(\overline{v}_{1/2}-A^{\top}\overline{\lambda}_{1})-(v_{1/2}-A^{\top}\lambda_{1})\|_{g_{0}^{-1}}^{2}\\
 & \qquad+3h\|\overline{v}_{1/2}-A^{\top}\overline{\lambda}_{1}\|_{g_{0}^{-1}}^{2}\|\ox-x\|_{g_{0}}
\end{align*}
\end{proof}
Combining the bounds for $\T_{\lambda},\T_{x},\T_{v}$, we have
\begin{align*}
 & \|\T(x,v,\lambda)-\T(\overline{x},\overline{v},\overline{\lambda})\|\\
 & \leq0.7h\|\overline{v}-v\|_{g_{0}^{-1}}+2.4rh\|A^{\top}(\overline{\lambda}-\lambda)\|_{g_{0}^{-1}}+1.5h\|\overline{x}-x\|_{g_{0}}\|\overline{v_{1/2}}-A^{\top}\overline{\lambda}\|_{g_{0}^{-1}}\\
 & \quad+0.6h\|\overline{v}-v\|_{g_{0}^{-1}}+1.2h\|A^{\top}(\overline{\lambda}_{1}-\lambda_{1})\|_{g_{0}^{-1}}+0.6h\|\overline{x}-x\|_{g_{0}}\|\overline{v}_{1/2}-A^{\top}\overline{\lambda}_{1}\|_{g_{0}^{-1}}\\
 & \quad+2h\|(\overline{v}_{1/2}-A^{\top}\overline{\lambda}_{1})-(v_{1/2}-A^{\top}\lambda_{1})\|_{g_{0}^{-1}}\|\overline{v}_{1/2}-A^{\top}\overline{\lambda}_{1}\|_{g_{0}^{-1}}\\
 & \quad+h\|(\overline{v}_{1/2}-A^{\top}\overline{\lambda}_{1})-(v_{1/2}-A^{\top}\lambda_{1})\|_{g_{0}^{-1}}^{2}\\
 & \quad+3h\|\overline{v}_{1/2}-A^{\top}\overline{\lambda}_{1}\|_{g_{0}^{-1}}^{2}\|\ox-x\|_{g_{0}}.
\end{align*}
To simplify the terms, we note that
\begin{align*}
\|\overline{v}_{1/2}-A^{\top}\overline{\lambda}_{1}\|_{g_{0}^{-1}}\leq & \|\overline{v}_{1/2}-A^{\top}\overline{\lambda}\|_{g_{0}^{-1}}+\|A^{\top}(LL^{\top})^{-1}A\overline{g}_{1/2}^{-1}\left(\overline{v}_{1/2}-A^{\top}\overline{\lambda}\right)\|_{g_{0}^{-1}}\\
= & \|\overline{v}_{1/2}-A^{\top}\overline{\lambda}\|_{g_{0}^{-1}}+\|g_{0}^{1/2}\overline{g}_{1/2}^{-1}\left(\overline{v}_{1/2}-A^{\top}\overline{\lambda}\right)\|_{P}\\
\leq & \|\overline{v}_{1/2}-A^{\top}\overline{\lambda}\|_{g_{0}^{-1}}+\|g_{0}^{1/2}\overline{g}_{1/2}^{-1}\left(\overline{v}_{1/2}-A^{\top}\overline{\lambda}\right)\|_{2}\\
\leq & 3\|\overline{v}_{1/2}-A^{\top}\overline{\lambda}\|_{g_{0}^{-1}}.
\end{align*}
Using this and simplifying, we have
\begin{align*}
 & \|\T(x,v,\lambda)-\T(\overline{x},\overline{v},\overline{\lambda})\|\\
\leq & 1.3h\|\overline{v}-v\|_{g_{0}^{-1}}+2.4rh\|A^{\top}(\overline{\lambda}-\lambda)\|_{g_{0}^{-1}}+3.3h\|\overline{x}-x\|_{g_{0}}\|\overline{v_{1/2}}-A^{\top}\overline{\lambda}\|_{g_{0}^{-1}}\\
 & +1.2h\|A^{\top}(\overline{\lambda}_{1}-\lambda_{1})\|_{g_{0}^{-1}}\\
 & +6h(\frac{1}{2}\|\overline{v}-v\|_{g_{0}^{-1}}+\|A^{\top}(\overline{\lambda}_{1}-\lambda_{1})\|_{g_{0}^{-1}})\|\overline{v}_{1/2}-A^{\top}\overline{\lambda}\|_{g_{0}^{-1}}\\
 & +h\|\overline{v}-v\|_{g_{0}^{-1}}^{2}+2h\|A^{\top}(\overline{\lambda}_{1}-\lambda_{1})\|_{g_{0}^{-1}}^{2}\\
 & +27h\|\overline{v}_{1/2}-A^{\top}\overline{\lambda}\|_{g_{0}^{-1}}^{2}\|\ox-x\|_{g_{0}}.
\end{align*}
Next, we note that
\begin{align*}
\|\overline{v_{1/2}}-A^{\top}\overline{\lambda}\|_{g_{0}^{-1}}\leq & \frac{1}{2}\|\overline{v}-v^{*}\|_{g_{0}^{-1}}+\frac{1}{2}\|v_{0}-v^{*}\|_{g_{0}^{-1}}+\|v^{*}\|_{g_{0}^{-1}}\\
 & +\frac{1}{2}\|A^{\top}\overline{\lambda}-A^{\top}\lambda^{*}\|_{g_{0}^{-1}}+\frac{1}{2}\|A^{\top}\overline{\lambda}-A^{\top}\lambda^{*}\|_{g_{0}^{-1}}+\|A^{\top}\lambda^{*}\|_{g_{0}^{-1}}\\
\leq & \frac{1}{2}r+\frac{1}{2}r+\|v^{*}\|_{g_{0}^{-1}}+\frac{r}{2h}+\frac{r}{2h}+\frac{r}{h}\leq\frac{3r}{h}+\|v^{*}\|_{g_{0}^{-1}}
\end{align*}
Using this, (\ref{eq:contractive_Alambda1}), $h\leq r$, $r^{2}\leq\frac{h}{16}$,
$r\leq\|v^{*}\|_{g_{0}^{-1}}/4$, we have
\begin{align*}
\|A^{\top}(\overline{\lambda}_{1}-\lambda_{1})\|_{g_{0}^{-1}} & \leq\|\overline{v}-v\|_{g_{0}^{-1}}+3r\|A^{\top}(\overline{\lambda}-\lambda)\|_{g_{0}^{-1}}+(\frac{5r}{h}+2\|v^{*}\|_{g_{0}^{-1}})\|\overline{x}-x\|_{g_{0}}\\
 & \leq r+\frac{3r^{2}}{h}+\frac{5r^{2}}{h}+2r\|v^{*}\|_{g_{0}^{-1}}\leq\frac{8r^{2}}{h}+2r\|v^{*}\|_{g_{0}^{-1}}\leq1
\end{align*}
Hence, we can further simplify it to 
\begin{align*}
 & \|\T(x,v,\lambda)-\T(\overline{x},\overline{v},\overline{\lambda})\|\\
\leq & 2.3h\|\overline{v}-v\|_{g_{0}^{-1}}+2.4rh\|A^{\top}(\overline{\lambda}-\lambda)\|_{g_{0}^{-1}}+3.3h\|\overline{x}-x\|_{g_{0}}\|\overline{v_{1/2}}-A^{\top}\overline{\lambda}\|_{g_{0}^{-1}}\\
 & +3.2h\|A^{\top}(\overline{\lambda}_{1}-\lambda_{1})\|_{g_{0}^{-1}}\\
 & +6h(\frac{1}{2}\|\overline{v}-v\|_{g_{0}^{-1}}+\|A^{\top}(\overline{\lambda}_{1}-\lambda_{1})\|_{g_{0}^{-1}})\|\overline{v}_{1/2}-A^{\top}\overline{\lambda}\|_{g_{0}^{-1}}\\
 & +27h\|\overline{v}_{1/2}-A^{\top}\overline{\lambda}\|_{g_{0}^{-1}}^{2}\|\ox-x\|_{g_{0}}\\
\leq & (\frac{3r}{h}+\|v^{*}\|_{g_{0}^{-1}})(6h\|\overline{v}-v\|_{g_{0}^{-1}}+9h\|A^{\top}(\overline{\lambda}_{1}-\lambda_{1})\|_{g_{0}^{-1}}+31h\|\ox-x\|_{g_{0}})\\
 & +2.4rh\|A^{\top}(\overline{\lambda}-\lambda)\|_{g_{0}^{-1}}
\end{align*}
where we used $\|\overline{v}_{1/2}-A^{\top}\overline{\lambda}\|_{g_{0}^{-1}}\leq\frac{3r}{h}+\|v^{*}\|_{g_{0}^{-1}}$
and $r\geq h$. Using the bound on $\|A^{\top}(\overline{\lambda}_{1}-\lambda_{1})\|_{g_{0}^{-1}}$,
we have
\begin{align*}
 & \|\T(x,v,\lambda)-\T(\overline{x},\overline{v},\overline{\lambda})\|\\
\leq & (\frac{3r}{h}+\|v^{*}\|_{g_{0}^{-1}})(15h\|\overline{v}-v\|_{g_{0}^{-1}}+27rh\|A^{\top}(\overline{\lambda}-\lambda)\|_{g_{0}^{-1}}+9h(\frac{36r}{h}+2\|v^{*}\|_{g_{0}^{-1}})\|\overline{x}-x\|_{g_{0}})\\
 & +2.4rh\|A^{\top}(\overline{\lambda}-\lambda)\|_{g_{0}^{-1}}\\
\leq & (\frac{3r}{h}+\frac{1}{4r})(15h\|\overline{v}-v\|_{g_{0}^{-1}}+30rh\|A^{\top}(\overline{\lambda}-\lambda)\|_{g_{0}^{-1}}+9h(\frac{36r}{h}+2\|v^{*}\|_{g_{0}^{-1}})\|\overline{x}-x\|_{g_{0}})\\
\leq & (\frac{3r}{h}+\|v^{*}\|_{g_{0}^{-1}})(400r+18h\|v^{*}\|_{g_{0}^{-1}})\|(x,v,\lambda)-(\overline{x},\overline{v},\overline{\lambda})\|.
\end{align*}

%% file: 90appendix_algo.tex
\section{Condition Number Independence via Self-concordant Barrier\label{subsec:Choosing}}

\begin{figure}[t]
\centering
\begin{tikzpicture}
    \node[rectangle,draw,rounded corners, very thick, fill=orange!30, minimum width=3.3cm] at (0, 0) (t1) {
    \begin{tabular}{c}
    	Ideal RHMC \\
    	$O(mn^{7/6})$
    \end{tabular}};
    \node[rectangle,draw,rounded corners, very thick, fill=orange!30, minimum width=3.3cm] at (7, 0) (t2) {\begin{tabular}{c}
    	Discretized RHMC \\
    	$O(mn^{3})$
    \end{tabular}};
    
    \node[rectangle,draw,rounded corners, very thick, fill=blue!30, minimum width=3.3cm] at (0, -3) (l1) {
    \begin{tabular}{c}
      Ideal CRHMC (\ref{subsec:ideal-mixing})\\ 
      $O(mk^{7/6})$
    \end{tabular}};
    \node[rectangle,draw,rounded corners, very thick, fill=blue!30, minimum width=3.3cm] at (7, -3) (l2) {
    \begin{tabular}{c}
      Discretized CRHMC (\ref{subsec:disc-mixing})\\ 
      $O(mk^3)$
    \end{tabular}};

    \draw[orange, thick, dotted, line width=1mm]    ($(t1.north west)+(-0.5,0.3)$) rectangle ($(t2.south east)+(0.5,-0.3)$) 
    node[black, right] at (2.5, 1.2) {Proven in \cite{kook2022conditionnumberindependent}};
    \draw[->,  line width=.6mm] (l1) to[out=90,in=270] node[right]{Reduce} (t1);
    \draw[->,  line width=.6mm] (l2) to[out=90,in=270] node[left]{Reduce} (t2);
\end{tikzpicture}
 \caption{Proof outline for the mixing rates of CRHMC\label{app:proofoutline}}
\end{figure}
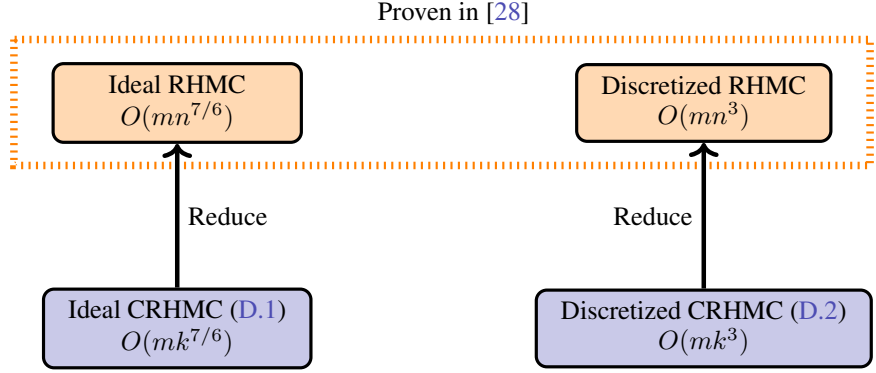

In this section, we analyze the convergence rates of the ideal CRHMC and discretized RHMC in our setting respectively, showing that both are independent of condition numbers.
We only show the case when {\bf $f$ is linear},
\begin{equation*}\label{eq:flinear}  \pi(x) \propto e^{-f(x)}=e^{-\alpha^{\top}x},  \text{ for some } \alpha\in \Rn. \end{equation*}
However, recall that {\bf all logconcave densities} can be reduced to this linear case (see \eqref{eq:problem-2}). 
We also focus on when a manifold $\mcal$ is a polytope in the form of $\{x\in \Rn : A'x \geq b',\, Ax=0\}$ for full-rank $A'\in \R^{m\times n},  A\in \R^{p\times n}$ and $b'\in \R^m$, with the Riemannian metric induced by the Hessian of the logarithmic barrier of the polytope.
For simplicity, we consider the case when there is no momentum (i.e., $\beta = 0$) in Algorithm~\ref{alg:CHMC_JL}. 
In addition, we consider a \emph{lazy} version of Algorithm~\ref{alg:CHMC_JL} to avoid a uniqueness issue of the stationary distribution of the Markov chain.
The lazy version of the Markov chain, at each step, does nothing with probability $\frac 1 2$ (in other words, stays at where it is and does not move). 
Note that this change for the purpose of proof  worsens the mixing rate only by a factor of $2$.

In this setting, we show that the mixing rates of the ideal CRHMC and the discretized CRHMC (Algorithm~\ref{alg:CHMC_JL}) are 
$O\Par{mk^{7/6} \log^2 \frac {\Lambda}{\epsilon}}$ and $O\Par{mk^3\log^3 \frac {\Lambda}{\epsilon}}$,
where $\Lambda$ is a warmness parameter and $k$ is the dimension of the constrained space defined by $\{x\in \Rn: Ax= 0\}$.
We remark that our algorithm is actually independent of condition number (i.e., no dependency on $\norm{\alpha}_2$ and the geometry of polytope).
This is the key reason that our sampler is much more efficient for skewed instances than previous samplers.

We first shed light on how RHMC and CRHMC can be related (see Figure~\ref{app:proofoutline}), establishing a correspondence between RHMC (full-dimensional space) and CRHMC (constrained space).
This connection enables us to refer to the mixing rates of the ideal RHMC and discretized RHMC proven in \cite{kook2022conditionnumberindependent}. 
To be precise, we prove in Section~\ref{subsec:ideal-mixing} the following theorem on the mixing rate of the ideal CRHMC, which can solve the Hamiltonian equations accurately without any error.

\begin{thm}[Mixing rate of ideal CRHMC]\label{thm:ideal-mixing}
	Let $\pi_T$ be the distribution obtained after $T$ iterations of the ideal CRHMC on a convex body 
	$\mcal=\{x\in \Rn: A'x\geq b',\, Ax=0\}$. 
	Let $\Lambda= \sup_{S\subseteq \mcal} \frac {\pi_0(S)}{\pi(S)}$ be the warmness of the initial distribution $\pi_0$. 
	For any $\epsilon>0$, there exists $T = O\Par{mk^{7/6}\log^2 \frac {\Lambda}{\epsilon}}$ such that $\norm{\pi_T- \pi}_{\TV} \leq \epsilon$, where $k$ is the dimension of the constrained space defined by $\{x\in \Rn: Ax = 0\}$.
\end{thm}

We then prove in Section~\ref{subsec:disc-mixing} the convergence rate of the discretized RHMC (Algorithm~\ref{alg:CHMC_JL}).

\begin{thm}[Mixing rate of discretized CRHMC]\label{thm:disc-mixing}
	Let $\pi_T$ be the distribution obtained after $T$ iterations of Algorithm~\ref{alg:CHMC_JL} on a convex body $\mcal=\{x\in \Rn: A'x\geq b',\, Ax=0\}$.
	Let $\Lambda= \sup_{S\subseteq \mcal} \frac {\pi_0(S)}{\pi(S)}$ be the warmness of the initial distribution $\pi_0$. 
	For any $\epsilon>0$, there exists $T = O\Par{mk^3 \log^3 \frac {\Lambda}{\epsilon}}$ such that $\norm{\pi_T- \pi}_{\TV} \leq \epsilon$, where $k$ is the dimension of the constrained space defined by $\{x\in \Rn: Ax = 0\}$.
\end{thm}

We believe there is room for improvement on the $n$-dependence via a more careful analysis.

\subsection{Convergence rate of ideal CRHMC}\label{subsec:ideal-mixing}

Lee and Vempala \cite{lee2018convergence} first analyzed Riemannian Hamiltonian Monte Carlo (RHMC) on $n$-dimensional polytopes embedded in $\Rn$, with an invertible metric induced by the Hessian of the logarithmic barrier of the polytopes.
They bounded the mixing rate in terms of \emph{smoothness} parameters that depend on the manifold.
In particular for uniform sampling, they showed that the convergence rate of RHMC is $O(mn^{2/3})$.
Subsequently, \cite{kook2022conditionnumberindependent} extended their analysis to exponential densities and further analyzed the convergence rate of RHMC discretized by the implicit midpoint method, showing that the mixing rates are $O(mn^{7/6})$ and $O(mn^3)$, respectively.

However, our metric $M(x)$ defined for the constrained space could be singular in the underlying space $\Rn$, so we cannot directly refer to any theoretical results from \cite{lee2018convergence, kook2022conditionnumberindependent}.
To address this challenge, we establish a formalism that allows us to reduce the ideal CRHMC to the ideal RHMC, obtaining the mixing rate through this reduction.

Even though our convex body $\mcal= \{x\in \Rn: A'x\geq b',\, Ax = 0 \}$ of dimension $k$ is embedded in $\Rn$, we can handle it with an invertible metric $\bg$ on $\mcal$ properly defined as if it is embedded in $\R^k$.
To this end, we use $\{u_1,...,u_k\}$ to denote an orthonormal basis of the constrained space (which is the null space of $A$) and extend it to an orthonormal basis of $\Rn$ denoted by $\{u_1,...,u_k,...,u_n\}$.
We also define two matrices $U_k \in\R^{n\times k}$ and $U\in \R^{n \times n}$ by
\[
U_{k}=\left[\begin{array}{ccc}
	u_{1} & \cdots & u_{k}\end{array}\right]\quad\&\quad U=\left[\begin{array}{cccc}
	U_{k} & u_{k+1} & \cdots & u_{n}\end{array}\right].
\]

Using this orthonormal basis $\{u_1,...,u_k\}$, we can consider a new coordinate system $y=(y_1,...,y_k)\in \R^k$ on the $k$-dimensional manifold $\mcal$.
Moreover, there exists one-to-one correspondence between $y$ and $x$; for any $x\in \mcal$ there is a unique $y$ such that $x=U_ky$, and we can recover this $y$ by multiplying $U_k^\top$ (i.e., $y = U_k^\top x$).

Let us define the invertible local metric $\bg$ at $y \in \mcal$ by 
\[
\bg(y)(u_i,u_j)\defeq g(x)(u_i, u_j )\quad \text{for }i,j\leq k.
\]
With abuse of notations, we also use $\bg(y)$ to denote the $k\times k$ matrix with its $(i,j)$-entry being  $\bg(y)(u_i,u_j)$.
We first establish relationships between $\bg(y)$ (and its inverse $\bg^{-1}$) and $M(x)$ (and its pseudoinverse $W\defeq M(x)^{\dagger}$).
We recall that for the orthogonal projection $Q$ to the null space of $A$
\begin{align*}
	M(x) = Q^\top g(x) Q,\,\,
	W(x) = g(x)^{-\half}(I-P(x)) g(x)^{-\half}.
\end{align*}

\begin{lem}\label{lem:bg_relationship}
	We have $\bg(y) = U_k^\top M(x) U_k =  U_k^\top g(x) U_k$  and $\bg(y)^{-1} = U_k^\top W(x)U_k$.
\end{lem}

\begin{proof}
{}	It is immediate from the definition of $\bg$ that $\bg(y) = U_k^\top g(x) U_k$. 
	Since the quadratic forms of $M(x)$ and $g(x)$ agree on the constrained space, we also have $\bg(y) = U_k^\top M(x) U_k$.

	For $\bg^{-1}$,
	we define two matrices $P_k \in \R^{n\times k}$ and $P_r\in \R^{n\times (n-k)}$ by
	\[
		P_{k}=\left[\begin{array}{c}
	I_{k}\\
	0_{(n-k)\times k}
	\end{array}\right],\,\,
	P_{r}=\left[\begin{array}{c}
0_{k\times(n-k)}\\
I_{n-k}
\end{array}\right]
	\]
	where $0_{(n-k)\times k}$ is the zero matrix of size $(n-k)\times k$, $I_k$ is the identity matrix of size $k\times k$ and so on.
	Due to $U_k = U P_k$, the upper-left $k\times k$ submatrix of $g'(x):= U^\top g(x)U$ is exactly $\bg(y)$.
	Let us represent the inverse of $g'$ in the form of block matrix:
	\[
	g'(x)^{-1}=\left[\begin{array}{cc}
	B_{1} & B_{2}\\
	B_{2}^{\top} & B_{3}
	\end{array}\right],
	\]
	for $B_1\in\R^{k\times k}$, $B_2 \in \R^{k\times (n-k)}$ and $B_3 \in \R^{(n-k)\times (n-k)}$.
	Using the formula of the inverse of block matrices (see App.~\ref{subsec:Details}),
	\[
		\bg(y)^{-1} = B_1 - B_2 B_3^{-1} B_2^\top.
	\]
	It is straightforward to check
	\begin{align*}
		B_1 
		&=
		P_k^\top g'(x)^{-1} P_k
		= P_k^\top U^\top g(x)^{-1} U P_k\\
		&= U_k^\top g(x)^{-1} U_k,\\
		B_2
		&= P_k^\top g'(x)^{-1}P_r
		= U_k^\top g(x)^{-1} U_r,\\
		B_3
		&= P_r^\top g'(x)^{-1} P_r
		= U_r^\top g(x)^{-1} U_r,
	\end{align*}
	for 
	$U_{r}
	=
	\left[
	\begin{array}{ccc} u_{k+1} & \cdots & u_{n}
	\end{array}
	\right] \in \R^{n\times(n-k)}$.
	Therefore,
	\begin{align*}
		\bg(y)^{-1}
		&=
		U_k^\top g(x)^{-1} U_k - 
		U_k^\top g(x)^{-1} U_r
		(U_r^\top g(x)^{-1} U_r)^{-1}
		U_r^\top g(x)^{-1} U_k\\
		&=
		U_k^\top\Par{g(x)^{-1} - g(x)^{-1}U_r(U_r^\top g(x)^{-1} U_r)^{-1}U_r^\top g(x)^{-1}}U_k.
	\end{align*}
	Since $g(x)^{-\half}U_r(U_r^\top g(x)^{-1}U_r)^{-1}U_r^\top g(x)^{-\half}$ is the orthogonal projection to the row space of $U_r^\top g(x)^{-\half}$ and this row space is the same with the row space of $Ag(x)^{-\half}$, the uniqueness of orthogonal projection matrices implies 
	\[
		g(x)^{-\half}A^\top(Ag(x)^{-1}A^\top)^{-1}Ag(x)^{-\half}
		= 
		g(x)^{-\half}U_r(U_r^\top g(x)^{-1}U_r)^{-1}U_r^\top g(x)^{-\half}.
	\]
	Therefore,
	\begin{align*}
		\bg(y)^{-1}
		&=
		U_k^\top\Par{g(x)^{-1} - g(x)^{-1}
		A^\top(Ag(x)^{-1}A^\top)^{-1}A
		g(x)^{-1}}U_k\\
		&=
		U_k^\top \Par{g(x)^{-\half} \Par{I - g(x)^{-\half} A^\top(Ag(x)^{-1}A^\top)^{-1}A g(x)^{-\half} } g(x)^{-\half}} U_k\\
		&= U_k^\top \Par{g(x)^{-\half} (I - P(x)) g(x)^{-\half}} U_k\\
		&= U_k^\top W(x) U_k.
	\end{align*}
\end{proof}

We can now view the ideal CRHMC with the metric $M(x)$ as the ideal RHMC with the metric $\bg$ on the $k$-dimensional manifold.
Note that we have to ensure that the local metric $\bg$ is also induced by the Hessian of a logarithmic barrier, in order to refer to results from it.

\begin{lem}
	Let $\overline{A'} = A'U_k$ and $\psi(y)$ be the logarithmic barrier of the $k$-dimensional polytope defined by $\{y\in \R^k : \overline{A'}y\leq b'\}$.
	Then $\nabla^2_y \psi(y) = \bg(y)$.
\end{lem}

\begin{proof}
	Observe that $\{y\in \R^k : \overline{A'}y \geq b'\}$  is the new representation of $\mcal = \{x\in \Rn: A'x\geq b', Ax=0\}$ in the $y$-coordinate system.
	Due to $\overline{A'}y = Ax$, we have $S_x = \Diag(A'x-b') =  \Diag(\overline{A'}y-b) = S_y$.
	For the logarithmic barrier $\phi(x)$ of $\{x\in \Rn : A'x \geq b\}$, 
	direct computation results in 
	\begin{align*}
	\nabla^2_y \psi(y) 
	&= 
	\overline{A'}^\top S_y^{-2} \overline{A'} 
	= U_k^\top A'^\top S_x^{-2} A' U_k 
	= U_k^\top \nabla^2_x \phi(x)U_k\\ 
	&= U_k^\top g(x) U_k = \bg(y),
	\end{align*}
	where we used $\nabla_x^2\phi(x) = A'^\top S_x^{-2}A'$ in the third equality and Lemma~\ref{lem:bg_relationship} in the last equality.
\end{proof}

Most importantly, we prove that this ideal RHMC on the $k$-dimensional manifold with the metric $\bg(y)$ is equivalent to the ideal CRHMC with the metric $M(x)$.

\begin{lem}\label{lem:equiv}
	The dynamics $(x,v)$ and $(y,u)$ of the ideal CRHMC in $\Rn$ and the ideal RHMC in $\R^k$ are equivalent in a sense that the Hamiltonian equations for $(x,v)$ can be obtained by lifting up the Hamiltonian equations for $(y,u)$  from $\R^k$ to $\Rn$ via multiplying $U_k$.
	That is, when we lift up the dynamics $(y, u)$ in $\R^k$ to the dynamics $(\bx, \bv)$ in $\Rn$ defined by $\bx = U_ky$ and $\bv = U_ku$, it follows that
	\begin{align*}
		\frac{dx}{dt} = \frac{d\bx}{dt},\,
		\frac{dv}{dt} = \frac{d\bv}{dt}
		\text{ and }
		v, \bv \sim \N(0, M(x)). 
	\end{align*} 
\end{lem}

\begin{proof}
	We first recall from the proof of Lemma~\ref{lem:ham_deri} that the Hamiltonian equations of $(x, v)$ are 
	\begin{align*}
 		\frac{dx}{dt}
 		&=
 		W(x)v,
 		\\
 		\frac{dv}{dt}
 		&=
 		- \nabla_x f(x) 
 		- \half \tr\Brack{W(x)Dg(x)}
 		+ \half Dg(x)\Brack{\frac{dx}{dt}, \frac{dx}{dt}}
 		- A^\top \lambda(x,v)\\
 		&= 
 		- \nabla_x f(x) 
 		- \half \tr\Brack{W(x)Dg(x)}
 		+ \half Dg(x)\Brack{\frac{dx}{dt}, \frac{dx}{dt}}\\
 		&~- A^\top 
 		\Par{AA^\top}^{-1}A
 		\Par{- \nabla_x f(x) 
 		- \half \tr\Brack{W(x)Dg(x)}
 		+ \half Dg(x)\Brack{\frac{dx}{dt}, \frac{dx}{dt}}}\\
 		&=
 		(I- A^\top \Par{AA^\top}^{-1}A)
 		\Par{- \nabla_x f(x) 
 		- \half \tr\Brack{W(x)Dg(x)}
 		+ \half Dg(x)\Brack{\frac{dx}{dt}, \frac{dx}{dt}}}\\
 		&=
 		U_k U_k^\top
 		\Par{- \nabla_x f(x) 
 		- \half \tr\Brack{W(x)Dg(x)}
 		+ \half Dg(x)\Brack{\frac{dx}{dt}, \frac{dx}{dt}}},
 	\end{align*}
 	where the last equality follows from that $U_kU_k^\top$ is the orthogonal projection to the null space of $A$.

 	From Lemma 7 of \cite{lee2018convergence}, the Hamiltonian equations of $(y,u)$ are
 	\begin{align*}
 		\frac{dy}{dt}
 		&=
 		\bg(y)^{-1}u,
 		\\
 		\frac{du}{dt}
 		&=
 		-\nabla_y f(U_ky) - \half \tr\Brack{\bg(y)^{-1}D\bg(y)}
 		+\half D\bg(y)\Brack{\frac{dy}{dt}, \frac{dy}{dt}}.
 	\end{align*}
 	From the definitions of $(\bx, \bv)$, we have 
 	\begin{align*}
 		\frac{d\bx}{dt}
 		&= U_k \bg(y)^{-1}u
 		= U_k U_k^\top W(\bx) U_k u  
 		= U_k U_k^\top W(\bx) \bv\\
 		&\underset{(*)}{=} W(\bx)\bv,\\
 		\frac{d\bv}{dt}
 		&= U_k
 		\Par{-\nabla_y f(U_ky) - \half \tr\Brack{\bg(y)^{-1}D\bg(y)}
 		+\half D\bg(y)\Brack{\frac{dy}{dt}, \frac{dy}{dt}}},
 	\end{align*}
 	where $(*)$ follows from that $W(x)\bv$ is already in the constrained space (i.e., the null space of $A$).
 	Let us examine each term in $d\bv/dt$ separately.
 	\begin{align*}
 		\nabla_y f(U_ky) 
 		&= 
 		U_k^\top \nabla_{\bx} f(\bx),\\
 		\tr\Brack{\bg(y)^{-1} D_y\bg(y)}
 		&= 
 		\tr\Brack{U_k^\top W(\bx) U_k \cdot D_y\Par{U_k^\top g(U_ky)U_k} }\\
 		&=
 		\tr\Brack{U_kU_k^\top W(\bx)U_k U_k^\top (U_k^\top D_{\bx} g(\bx))}\\
 		&=
 		\tr\Brack{W(\bx)U_kU_k^\top (U_k^\top D_{\bx} g(\bx))}\\
 		&=
 		\tr\Brack{(U_kU_k^\top W(\bx))^\top (U_k^\top D_{\bx} g(\bx))}\\
 		&=
 		\tr\Brack{W(\bx) (U_k^\top D_{\bx} g(\bx))}\\
 		&=
 		U_k^\top \tr\Brack{W(\bx)D_{\bx} g(\bx)},\\
 		D_y\bg(y) \Brack{\frac{dy}{dt}, \frac{dy}{dt}}
 		&= U_k^\top (U_k^\top D_{\bx} g(\bx)) U_k \Brack{\frac{dy}{dt}, \frac{dy}{dt}}\\
 		&= \Par{\bv^\top W(\bx) U_k} U_k^\top (U_k^\top D_{\bx} g(\bx)) U_k \Par{U_k^\top W(\bx)\bv}\\
 		&= \bv^\top W(\bx) (U_k^\top D_{\bx} g(\bx)) W(\bx) \bv\\
 		&= U_k^\top D_{\bx}g(\bx)\Brack{\frac{d\bx}{dt}, \frac{d\bx}{dt}}.
 	\end{align*}
 	Putting all these together, the Hamiltonian equations of $(\bx, \bv)$ can be written as
 	\begin{align*}
 		\frac{d\bx}{dt}
 		&= W(\bx)\bv,\\
 		\frac{d\bv}{dt}
 		&= U_kU_k^\top
 		\Par{- \nabla f(\bx) 
 		- \half \tr\Brack{W(\bx)Dg(\bx)}
 		+ \half Dg(\bx)\Brack{\frac{d\bx}{dt}, \frac{d\bx}{dt}}}.
 	\end{align*}
 	Therefore, the Hamiltonian equations of $(x,v)$ and $(\bx, \bv)$ are exactly the same.
 	In addition, $\bv = U_ku$ leads to $\bv\sim \N(0, U_k \bg(y) U_k^\top) = \N(0, M(x))$.
\end{proof}

Using these three lemmas, we conclude that the dynamics of the ideal CRHMC on the constrained space is equivalent to that of the ideal RHMC on the corresponding $k$-dimensional polytope.
Therefore, Theorem~\ref{thm:ideal-mixing} immediately follows from Corollary 3 in \cite{kook2022conditionnumberindependent}.

\begin{corollary}
Let $\pi$ be a target distribution on a polytope with $m$ constraints
in $\Rn$ such that $\frac{d\pi}{dx}\sim e^{-\alpha^{\top}x}$ for
$\alpha\in\Rn$. Let $\mcal$ be the Hessian manifold of the polytope
induced by the logarithmic barrier of the polytope. Let $\Lambda=\sup_{S\subset\mcal}\frac{\pi_{0}(S)}{\pi(S)}$
be the warmness of the initial distribution $\pi_{0}$. Let $\pi_{T}$
be the distribution obtained after $T$ iterations of the ideal RHMC
on $\mcal$. For any $\varepsilon>0$ and step size $h=O\Par{\frac{1}{n^{7/12}\log^{1/2}\frac{\Lambda}{\varepsilon}}}$,
there exists $T=O\Par{mn^{7/6}\log^{2}\frac{\Lambda}{\varepsilon}}$ such
that $\norm{\pi_T - \pi}_{TV} \leq \varepsilon$.
\end{corollary}

\subsection{Convergence rate of discretized CRHMC}\label{subsec:disc-mixing}
We attempt to demonstrate a similar reduction of the discretized CRHMC.
However, it is trickier than that of the ideal RHMC, since Algorithm~\ref{alg:CHMC_JL} uses the simplified Hamiltonian, which omits the Lagrangian term $c(x)^\top \lambda(x,v)$, in place of the full Hamiltonian.

We look into the reduction in two steps.
First of all, we show in Section~\ref{subsubsec:disc-corr1} that the dynamics of $x$ is the same under the discretized CRHMC via IMM with the simplified Hamiltonian, and the discretized CRHMC via IMM with the full Hamiltonian, and that the acceptance probabilities are the same as well.
Next in Section~\ref{subsubsec:disc-corr2}, we show a correspondence between the discretized CRHMC via IMM and the discretized RHMC via IMM, just as we did for the ideal case in Section~\ref{subsec:ideal-mixing}.

\subsubsection{Simplified Hamiltonian and full Hamiltonian in constrained space}\label{subsubsec:disc-corr1}

We recall that the simplified Hamiltonian $\overline{H}(x,v) $ is the sum of two parts defined by 
\begin{align*}
	\overline{H}_1(x,v)
	&= f(x) + \half \log \pdet M(x)\\
	&= f(x) + \half\Par{\log\det g(x) + \log \det Ag(x)^{-1}A^\top - \log\det AA^\top},\\
	\overline{H}_2(x,v)
	&=\half v^\top W(x)v.
\end{align*}
The full Hamiltonian $H(x,v)$ with the Lagrangian term $c(x)^\top \lambda(x,v)$ can be also written as the sum of two parts defined by
\begin{align*}
	H_1(x,v)
	&= f(x) + \half\Par{\log\det g(x) + \log\det Ag(x)^{-1}A^\top - \log\det AA^\top}\\
	&\quad- x^\top A^\top(AA^\top)^{-1}A \Par{\nabla f(x) + \half \tr\Brack{W(x)Dg(x)}},\\
	H_2(x,v)
	&=\half v^\top W(x)v + \half x^\top A^\top(AA^\top)^{-1}A Dg(x) \Brack{\frac{dx}{dt},\frac{dx}{dt}},
\end{align*}
and IMM with this Hamiltonian is implemented as in (\ref{eq:mid_sol}).
We note from the proof of Lemma~\ref{lem:equiv} that
\begin{align*}
	\frac{\del H_1}{\del x}(x,v) = U_kU_k^\top \frac{\del \bh_1}{\del x}(x,v),\\
	\frac{\del H_2}{\del x}(x,v) = U_kU_k^\top \frac{\del \bh_2}{\del x}(x,v),\\
	\frac{\del H_2}{\del v}(x,v) = U_kU_k^\top \frac{\del \bh_2}{\del v}(x,v).
\end{align*}

\begin{lem}
	For step size $h$, let $(\bx, \bv)$ and $(x,v)$ be the outputs of IMM with the simplified Hamiltonian and with the full Hamiltonian starting from $(x_0,v_0)$, respectively.
	Then $\bx = x$, and $\frac{e^{-\bh(\bx, \bv)} }{e^{-\bh(x_0, v_0)}} = \frac{e^{-H(x,v)}}{e^{-H(x_0, v_0)}}$.
\end{lem}

\begin{proof}
	We use $x_{\ot}(=x_0), x_{\tt}(=x)$ and $v_{\ot},v_{\tt},v$ to denote the points obtained during one step of IMM with the full Hamiltonian. 
	We similarly define $\bx_i$ and $\bv_i$ for $i=\ot, \tt$.
	As $(x_0,v_0)$ is a starting point, $U_kU_k^\top x_0 = x_0$ and $U_kU_k^\top v_0 = v_0$.
	Due to $\nulls(W(x)) = \text{row}(A)$, we have that $W(z)U_kU_k^\top w = W(z)w$.
	By comparing the first step of IMM for each Hamiltonian,
	\[
		U_kU_k^\top \bv_{\ot} = v_0 - \frac{h}2 U_kU_k^\top \frac{\del \bh_1}{\del x}(x_0, v_0)
		=v_0 - \frac{h} 2 \frac{\del H_1}{\del x}(x_0, v_0)
		= v_{\ot},
	\]
	and thus $U_kU_k^\top \bv_{\ot} = v_{\ot}$.

	From the second step of IMM, $\bx_{\tt}$ is already in the null space of $A$.
	For $x_{\ot} = \bx_{\ot} =  x_0, \bx_{\mid} = (\bx_{\ot} + \bx_{\tt})/2$ and $\bv_{\mid} = (\bv_{\ot} + \bv_{\tt})/2$, the second step of IMM with the simplified Hamiltonian is 
	\begin{align*}
		\bx_{\tt} 
		&= x_{\ot} + h U_kU_k^\top \frac{\del \bh_2}{\del v}(\bx_{\mid}, \bv_{\mid})
		= x_{\ot} + h U_kU_k^\top W(\bx_{\mid})\bv_{\mid}\\
		&= x_{\ot} + h U_kU_k^\top W(\bx_{\mid})U_kU_k^\top\bv_{\mid}
		= x_{\ot} + h U_kU_k^\top \frac{\del \bh_2}{\del v}(\bx_{\mid}, U_kU_k^\top\bv_{\mid})\\
		&= x_{\ot} + h \frac{\del H_2}{\del v}(\bx_{\mid}, U_kU_k^\top\bv_{\mid}),\\
		U_kU_k^\top\bv_{\tt} 
		&= v_{\ot} + h U_kU_k^\top \frac{\del \bh_2}{\del x}(\bx_{\mid}, \bv_{\mid})\\
		&= v_{\ot} + h U_kU_k^\top \Par{-\half Dg(\bx_{\mid})\Brack{W(\bx_{\mid})\bv_{\mid}, W(\bx_{\mid})\bv_{\mid}}}\\
		&= v_{\ot} + h U_kU_k^\top \Par{-\half Dg(\bx_{\mid})\Brack{W(\bx_{\mid})U_kU_k^\top\bv_{\mid}, W(\bx_{\mid})U_kU_k^\top\bv_{\mid}}}\\
		&= v_{\ot} + h U_kU_k^\top \frac{\del \bh_2}{\del x}(\bx_{\mid}, U_kU_k^\top\bv_{\mid})\\
		&=v_{\ot} + h\frac{\del H_2}{\del x}(\bx_{\mid}, U_kU_k^\top \bv_{\mid}).
	\end{align*}
	Note that $U_kU_k^\top\bv_{\mid} = (U_kU_k^\top\bv_{\tt}  + v_{\ot})/2$ and $\bx_{\mid} = (\bx_{\tt} + x_{\ot})/2$.
	Since the solution of this second step is characterized as a unique fixed-point, it follows that $(\bx_{\tt}, U_kU_k^\top \bv_{\tt}) = (x_{\tt}, v_{\tt})$ and so $\bx = x$.
	In the same way we analyzed $\bv_{\tt}$, we can obtain that $U_kU_k^\top \bv = v$.

	We now compare the acceptance probabilities. 
	We clearly have $\bh(x_0,v_0) = H(x_0, v_0)$ due to $c(x_0)=0$ and
	have $\bh_1(\bx, \bv) = H_1(x,v)$ due to $\bx = x$.
	For $\bh_2$, 
	\begin{align*}
		\bv^\top W(\bx) \bv 
		=  \bv^\top U_k U_k^\top W(x) U_k U_k^\top \bv
		= v^\top W(x) v,
	\end{align*}
	and so $\bh_2(\bx, \bv) = H_2(x,v)$.
\end{proof}

\subsubsection{CRHMC and RHMC discretized by IMM}\label{subsubsec:disc-corr2}

In this section, we show that there is a correspondence between the dynamics of CRHMC discretized by IMM  and that of RHMC discretized by IMM.

\begin{lem}
The discretized CRHMC via IMM in $\Rn$ and the discretized RHMC via IMM in $\R^k$ are equivalent. 
That is, the output $(x_1, v_1)$ given by the discretized CRHMC starting from $(x, v)$ is the same with $(U_ky_1, U_k u_1)$, where $(y_1, u_1)$ is the output of the discretized RHMC starting from $(y,u)$ satisfying $(x, v) = (U_k y , U_ku)$.
Moreover, the acceptance probabilities are the same due to
\[
	\frac{e^{-H^c(x_1, v_1)}}{e^{-H^c(x, v)}}
	=
	\frac{e^{-H^r(y_1, u_1)}}{e^{-H^r(y, u)}},
\]
where $H^c(x,v)$ and $H^r(y,u)$ are the Hamiltonians of CRHMC and RHMC respectively.
\end{lem}

\begin{proof}
	We first recall that $H^c(x,v)$ can be rewritten as the sum of two parts defined by
	\begin{align*}
	H^c_1(x, v)
	&= f(x) + \half\log\pdet M(x)\\
	&\quad- x^\top A^\top(AA^\top)^{-1}A \Par{\nabla f(x) + \half \tr\Brack{W(x)Dg(x)}},\\
	H^c_2(x, v)
	&=\half v^\top W(x)v + \half x^\top A^\top(AA^\top)^{-1}A Dg(x) \Brack{\frac{dx}{dt},\frac{dx}{dt}}.
	\end{align*}
	Similarly for $H^r$, we can represent it by the sum of two parts defined by
	\begin{align*}
	H^r_1(y, u)
	&= f(U_ky) + \half\log\det \bg(y),\\
	H^r_2(y, u)
	&= \half u^\top \bg(y)^{-1}u.
\end{align*}
	For the first claim, we need to show that each step of IMM for RHMC and CRHMC is equivalent, thus it suffices to check that for any $(y, u)\in \R^k \times \R^k$
	\begin{align*}
		\frac{\del H^c_1(U_ky, U_ku)}{\del x}
		&= U_k \frac{\del H^r_1(y, u)}{\del y},\\
		\frac{\del H^c_2(U_ky, U_ku)}{\del x}
		&= U_k \frac{\del H^r_2(y, u)}{\del y},\\
		\frac{\del H^c_2(U_ky, U_ku)}{\del v}
		&= U_k \frac{\del H^r_2(y, u)}{\del u}.
	\end{align*}
	These computations were already checked in the proof of Lemma~\ref{lem:equiv}.

	For the second claim, we note that the Lagrangian term vanishes due to $c(x) = c(U_ky) = 0$.
	Then the second claim follows from
	\begin{align*}
		\log\det \bg(y')
		&= \log\det U_k^\top M(U_ky')U_k \quad(\text{Lemma~\ref{lem:bg_relationship}})\\
		&= \log\pdet M(U_ky'),\\
		u'^\top \bg(y')^{-1} u'
		&= u'^\top U_k^\top W(U_ky') U_k u'. \quad (\text{Lemma~\ref{lem:bg_relationship}})
	\end{align*}
\end{proof}

The previous two lemmas imply that the dynamics of the discretized CRHMC via IMM on the constrained space is equivalent to that of the discretized RHMC via IMM on the corresponding k-dimensional polytope.
Therefore, Theorem~\ref{thm:disc-mixing} follows from Corollary 4 in \cite{kook2022conditionnumberindependent}.

\begin{corollary}
Let $\pi$ be a target distribution on a polytope with $m$ constraints
in $\Rn$ such that $\frac{d\pi}{dx}\sim e^{-\alpha^{\top}x}$ for
$\alpha\in\Rn$. Let $\mathcal{M}$ be the Hessian manifold of the polytope
induced by the logarithmic barrier of the polytope. Let $\Lambda=\sup_{S\subset\mathcal{M}}\frac{\pi_{0}(S)}{\pi(S)}$
be the warmness of the initial distribution $\pi_{0}$. Let $\pi_{T}$
be the distribution obtained after $T$ iterations of RHMC discretized
by IMM on $\mathcal{M}$. For any $\varepsilon>0$ and step size $h=O\Par{\frac{1}{n^{3/2}\log\frac{\Lambda}{\varepsilon}}}$,
there exists $T=O\Par{mn^{3}\log^{3}\frac{\Lambda}{\varepsilon}}$ such
that $\norm{\pi_{T} - \pi}_{TV} \leq \varepsilon$.
\end{corollary}

%% file: 94appendix_prelim.tex
\section{Missing Notations and Definitions\label{sec:Missing-Definitions-and}}

\subsection{Notations}

\begin{itemize}
	\item We use $\mathcal{N}(\mu, \Sigma)$ to denote Gaussian distribution with mean $\mu$ and covariance $\Sigma$.
	\item We use $\text{Null}(A)$ and $\range(A)$ to denote the null space and image space of a matrix or linear operator $A$.
	\item We use $\nabla^2 f\in \R^{n\times n}$ to denote the Hessian of a function $f:\Rn \to \R$. 
	\item We use $\norm{\cdot}$ to denote $\ell_2$-norm unless specified otherwise, and define $\norm{x}_A := \sqrt{x^{\top}Ax}$ for a vector $x\in \Rn$ and a matrix $A\in \R^{n\times n}$.
	\item We use $\partial K$ to denote the boundary of the set $K$.
	\item For a matrix $g(x)$ with $x\in \Rn$, we use $Dg(x)$ to denote the derivative of $g(x)$ with respect to $x$. 
	This can be thought of as the $n\times n \times n$ tensor such that $(Dg(x))(i,j,k) = \frac{\partial (g(x))_{ij}}{\partial x_k}$.
	In other words, $(Dg(x))(\cdot, \cdot, k)$ is the matrix, each of entries is the derivative of $g(x)$ with respect to $x_k$.
	In addition, for a vector $v\in \Rn$, $Dg(x)[v,v]$ is a vector in $\Rn$ such that $(Dg(x)[v,v])_i = v^\top Dg(x)(\cdot, \cdot, i)v$.
	\item For a matrix $A$ of size $n\times n$, we use $A\cdot Dg(x)$ to denote a $n\times n \times n$ tensor such that $(A\cdot Dg(x))(\cdot, \cdot, i) = A\cdot (Dg(x))(\cdot, \cdot, i)$.
	We use $\tr(A\cdot Dg(x))$ to denote a vector in $\Rn$ such that $(\tr(A\cdot Dg(x)))_i = \tr((A\cdot Dg(x))(\cdot, \cdot, i))$.
\end{itemize}

\subsection{Definitions\label{subsec:Definitions}}

\paragraph{Convex body.} A convex body is a compact and convex set.

\paragraph{Isotropy.} A random variable $X$ is said to be in isotropic position if $\E X = 0$ and $\E XX^{\top}=I$.

\paragraph{Pseudo-inverse.}

For a matrix $A\in\R^{m\times n}$, it is well known that there always
exists the unique pseudo-inverse matrix $A^{\dagger}$ that satisfies
the following conditions:
\begin{enumerate}
\item $A^{\dagger}AA^{\dagger}=A^{\dagger}$.
\item $AA^{\dagger}A=A$.
\item $AA^{\dagger}$ and $A^{\dagger}A$ are symmetric.
\end{enumerate}
It is also well known that $\nulls(A^{\dagger})=\nulls(A^{\top})$
and $\range(A^{\dagger})=\range(A^{\top})$.

\paragraph{Pseudo-determinant.}

For a square matrix $A$, its pseudo-determinant $\pdet(A)$ is defined
as the product of non-zero eigenvalues of $A$.

\paragraph{Leverage score.}

For a matrix $A\in\R^{m\times n}$, the leverage score of the $i^{th}$
row is $(A(A^{\top}A)^{\dagger}A^{\top})_{ii}$ for $i\in[m]$. When
$A$ is full-rank, it is simply $(A(A^{\top}A)^{-1}A^{\top})_{ii}$.

\paragraph{Log-barrier \& Dikin ellipsoid.}

For a polytope $P=\{x\in\Rn:Ax\leq b\}$ where $A\in\R^{m\times n}$
and $b\in\R^{m}$, let us denote the $i^{th}$ row of $A$ by $a_{i}$
and the $i^{th}$ row of $b$ by $b_{i}$. The log-barrier of $P$
is defined by
\[
\phi(x)=-\sum_{i=1}^{m}\log(b_{i}-a_{i}^{\top}x).
\]
For $x\in P$, the Dikin ellipsoid at $x$ is defined by $D(x) := \{y\in \Rn: (y-x)^{\top}\nabla^2\phi(x)(y-x)\leq 1\}$. 
The Dikin ellipsoid is always contained in $P$.

\paragraph{Analytic center.} The analytic center $x_{ac}$ of the polytope $P$ is the point minimizing the log-barrier (i.e., $x_{ac} = \arg\min \phi (x)$).

\paragraph{Self-concordant function.} A function $f:\Rn \to \R$ is self-concordant if it satisfies $\left|D^{3}f(x)[h,h,h]\right|\leq2\left(D^{2}f(x)[h,h]\right)^{3/2}$ for all $h\in\Rn$ and $x\in\Rn$.

\paragraph{Highly self-concordant function.}\label{def:selfcon}
	A barrier $\phi$ is called \emph{highly self-concordant} if it satisfies
	for all $h\in\R^{d}$ and $x\in\R^{d}$
	\[
	\left|D^{3}\phi(x)[h,h,h]\right|\leq2\left(D^{2}\phi(x)[h,h]\right)^{3/2}\quad\textrm{and}\quad\left|D^{4}\phi(x)[h,h,h,h]\right|\leq6\left(D^{2}\phi(x)[h,h]\right)^{2}.
	\]

\paragraph{Total variation.} For two probability distributions $P$ and $Q$ on support $K$, the total variation distance of $P$ and $Q$ is
\[
\norm{P-Q}_{\TV} 
\defeq \sup_{A\subseteq K} (P(A)-Q(A)).
\]

\subsection{Details\label{subsec:Details}}

\paragraph{Inverse and Determinant of block matrix.}

For a square matrix $M=\left[\begin{array}{cc}
A & B\\
C & D
\end{array}\right]$ with blocks $A,B,C,D$ of same size, if $D$ and $A-BD^{-1}C$ are
invertible, then its inverse and determinant can be computed by 

\begin{align*}
M^{-1} & =\left[\begin{array}{cc}
(A-BD^{-1}C)^{-1} & -(A-BD^{-1}C)^{-1}BD^{-1}\\
-D^{-1}C(A-BD^{-1}C)^{-1} & D^{-1}+D^{-1}C(A-BD^{-1}C)^{-1}BD^{-1}
\end{array}\right],\\
\det(M) & =\det(D)\det(A-BD^{-1}C).
\end{align*}

\paragraph{Orthogonal projection.} Let $S=\{x\in \Rn: Ax=b\}$ for $A\in \R^{m\times n}$ and $b\in \R^m$, and $x_0$ be a point in $S$.
Thus $S-x_0$ is the null space of $A$, due to $A(x-x_0)=0$.
The orthogonal projection $P$ to this null space is
\[
	P = I - A^\top (AA^\top )^{-1}A.
\]
Note that the range of $P$ always lies in the null space because
\[
	A(Pv) = A(I-A^\top (AA^\top )^{-1}A)v = 
	Av - AA^\top(AA^\top)^{-1}Av = Av-Av = 0.
\]
$I-P$ is also an orthogonal projection matrix, and eigenvalues of orthogonal projection matrices are either $0$ or $1$.

\paragraph{Matrix calculus.} Let $U(x)$ be a $n\times n$ matrix with a parameter $x\in \Rn$.
	\[
	\frac{\partial U^{-1}(x)}{\partial x_i} = - U(x)\frac{\partial U(x)}{\partial x_i}U(x).
	\]
	Hence using the notation $Dg$, we can write in a more compact way as 
	\[
		DU^{-1}(x) = -U(x)DU(x)U(x).
	\]
	For $\log \det$, 
	\[
		\frac{\partial \log \det U(x)}{\partial x_i} = \tr\Par{U^{-1}(x)\frac{\partial U(x)}{\partial x_i}}.
	\]
	In other words, 
	\[
		D(\log \det U(x)) = \tr(U^{-1}(x)DU(x)).
	\]

\paragraph{Cholesky decomposition.}

For a symmetric positive definite matrix $A$, there exists a lower
triangular matrix $L$ such that $LL^{\top}=A$.

\paragraph{Newton's method.} For $f$ convex and twice differentiable in $\Rn$, consider an unconstrained convex optimization $\min_x f(x)$.
Given a starting point $x_0 \in \Rn$, the Newton's method repeats 
\[
	x_i = x_{i-1} - (\nabla^2f(x_{i-1}))^{-1}\nabla f(x_{i-1})\quad \forall i\in \mathbb{N}
\]
to solve the optimization problem.

%% file: neurips_2022.bbl
\begin{thebibliography}{10}

\bibitem{andersen1983rattle}
Hans~C Andersen.
\newblock Rattle: A “velocity” version of the shake algorithm for molecular
  dynamics calculations.
\newblock {\em Journal of computational Physics}, 52(1):24--34, 1983.

\bibitem{bezakova2008accelerating}
Ivona Bez{\'a}kov{\'a}, Daniel {\v{S}}tefankovi{\v{c}}, Vijay~V Vazirani, and
  Eric Vigoda.
\newblock Accelerating simulated annealing for the permanent and combinatorial
  counting problems.
\newblock {\em SIAM Journal on Computing (SICOMP)}, 37(5):1429--1454, 2008.

\bibitem{bingham2019pyro}
Eli Bingham, Jonathan~P. Chen, Martin Jankowiak, Fritz Obermeyer, Neeraj
  Pradhan, Theofanis Karaletsos, Rohit Singh, Paul~A. Szerlip, Paul Horsfall,
  and Noah~D. Goodman.
\newblock {Pyro: Deep Universal Probabilistic Programming}.
\newblock {\em Journal of Machine Learning Research (JMLR)}, 20:28:1--28:6,
  2019.

\bibitem{brubaker2012family}
Marcus Brubaker, Mathieu Salzmann, and Raquel Urtasun.
\newblock A family of {MCMC} methods on implicitly defined manifolds.
\newblock In {\em Artificial intelligence and statistics (AISTATS)}, pages
  161--172, 2012.

\bibitem{campbell1995computing}
Yogin~E Campbell and Timothy~A Davis.
\newblock Computing the sparse inverse subset: an inverse multifrontal
  approach.
\newblock {\em University of Florida, Technical Report TR-95-021}, 1995.

\bibitem{chalkis2020volesti}
Apostolos Chalkis and Vissarion Fisikopoulos.
\newblock {volEsti: Volume} approximation and sampling for convex polytopes in
  {R}.
\newblock {\em arXiv preprint arXiv:2007.01578}, 2020.

\bibitem{chen2020fast}
Yuansi Chen, Raaz Dwivedi, Martin~J Wainwright, and Bin Yu.
\newblock Fast mixing of {Metropolized Hamiltonian Monte Carlo: Benefits} of
  multi-step gradients.
\newblock {\em Journal of Machine Learning Research (JMLR)}, 21:92--1, 2020.

\bibitem{cheng2018underdamped}
Xiang Cheng, Niladri~S Chatterji, Peter~L Bartlett, and Michael~I Jordan.
\newblock Underdamped {Langevin MCMC:} a non-asymptotic analysis.
\newblock In {\em Conference on Learning Theory (COLT)}, pages 300--323. PMLR,
  2018.

\bibitem{chewi2020exponential}
Sinho Chewi, Thibaut Le~Gouic, Chen Lu, Tyler Maunu, Philippe Rigollet, and
  Austin Stromme.
\newblock Exponential ergodicity of mirror-{L}angevin diffusions.
\newblock {\em Advances in Neural Information Processing Systems (NeurIPS)},
  33:19573--19585, 2020.

\bibitem{cobb2019introducing}
Adam~D Cobb, At{\i}l{\i}m~G{\"u}ne{\c{s}} Baydin, Andrew Markham, and Stephen~J
  Roberts.
\newblock Introducing an explicit symplectic integration scheme for
  {R}iemannian manifold {Hamiltonian Monte Carlo}.
\newblock {\em arXiv preprint arXiv:1910.06243}, 2019.

\bibitem{cousins2016practical}
Ben Cousins and Santosh Vempala.
\newblock A practical volume algorithm.
\newblock {\em Mathematical Programming Computation}, 8(2):133--160, 2016.

\bibitem{davis2006direct}
Timothy~A Davis.
\newblock {\em Direct methods for sparse linear systems}.
\newblock SIAM, 2006.

\bibitem{demmel1997applied}
James~W Demmel.
\newblock {\em Applied numerical linear algebra}.
\newblock SIAM, 1997.

\bibitem{duane1987hybrid}
Simon Duane, Anthony~D Kennedy, Brian~J Pendleton, and Duncan Roweth.
\newblock Hybrid {M}onte {C}arlo.
\newblock {\em Physics letters B}, 195(2):216--222, 1987.

\bibitem{dwivedi2018log}
Raaz Dwivedi, Yuansi Chen, Martin~J Wainwright, and Bin Yu.
\newblock Log-concave sampling: {Metropolis-Hastings} algorithms are fast!
\newblock In {\em Conference on Learning Theory (COLT)}, pages 793--797. PMLR,
  2018.

\bibitem{dyer1991random}
Martin Dyer, Alan Frieze, and Ravi Kannan.
\newblock A random polynomial-time algorithm for approximating the volume of
  convex bodies.
\newblock {\em Journal of the ACM (JACM)}, 38(1):1--17, 1991.

\bibitem{girolami2011riemann}
Mark Girolami and Ben Calderhead.
\newblock Riemann manifold {Langevin and Hamiltonian Monte Carlo} methods.
\newblock {\em Journal of the Royal Statistical Society: Series B (Statistical
  Methodology)}, 73(2):123--214, 2011.

\bibitem{griewank2008evaluating}
Andreas Griewank and Andrea Walther.
\newblock {\em Evaluating derivatives: principles and techniques of algorithmic
  differentiation}.
\newblock SIAM, 2008.

\bibitem{hairer2006geometric}
Ernst Hairer, Marlis Hochbruck, Arieh Iserles, and Christian Lubich.
\newblock Geometric numerical integration.
\newblock {\em Oberwolfach Reports}, 3(1):805--882, 2006.

\bibitem{haraldsdottir2017chrr}
Hulda~S Haraldsd{\'o}ttir, Ben Cousins, Ines Thiele, Ronan~MT Fleming, and
  Santosh Vempala.
\newblock Chrr: coordinate hit-and-run with rounding for uniform sampling of
  constraint-based models.
\newblock {\em Bioinformatics}, 33(11):1741--1743, 2017.

\bibitem{heirendt2019creation}
Laurent Heirendt, Sylvain Arreckx, Thomas Pfau, Sebasti{\'a}n~N Mendoza, Anne
  Richelle, Almut Heinken, Hulda~S Haraldsd{\'o}ttir, Jacek Wachowiak, Sarah~M
  Keating, Vanja Vlasov, et~al.
\newblock Creation and analysis of biochemical constraint-based models using
  the {COBRA Toolbox V.} 3.0.
\newblock {\em Nature protocols}, 14(3):639--702, 2019.

\bibitem{hoffman2014no}
Matthew~D Hoffman, Andrew Gelman, et~al.
\newblock The {No-U-Turn} sampler: adaptively setting path lengths in
  {Hamiltonian Monte Carlo}.
\newblock {\em Journal of Machine Learning Research (JMLR)}, 15(1):1593--1623,
  2014.

\bibitem{jerrum2004polynomial}
Mark Jerrum, Alistair Sinclair, and Eric Vigoda.
\newblock A polynomial-time approximation algorithm for the permanent of a
  matrix with nonnegative entries.
\newblock {\em Journal of the ACM (JACM)}, 51(4):671--697, 2004.

\bibitem{jia2021reducing}
He~Jia, Aditi Laddha, Yin~Tat Lee, and Santosh Vempala.
\newblock Reducing isotropy and volume to {KLS}: an {$O^*(n^3 \psi^2)$} volume
  algorithm.
\newblock In {\em Proceedings of the 53rd Annual ACM SIGACT Symposium on Theory
  of Computing (STOC)}, pages 961--974, 2021.

\bibitem{kannan1997random}
Ravi Kannan, L{\'a}szl{\'o} Lov{\'a}sz, and Mikl{\'o}s Simonovits.
\newblock Random walks and an {$O^*(n^5)$} volume algorithm for convex bodies.
\newblock {\em Random Structures \& Algorithms}, 11(1):1--50, 1997.

\bibitem{kannan2012random}
Ravindran Kannan and Hariharan Narayanan.
\newblock Random walks on polytopes and an affine interior point method for
  linear programming.
\newblock {\em Mathematics of Operations Research}, 37(1):1--20, 2012.

\bibitem{king2016bigg}
Zachary~A King, Justin Lu, Andreas Dr{\"a}ger, Philip Miller, Stephen
  Federowicz, Joshua~A Lerman, Ali Ebrahim, Bernhard~O Palsson, and Nathan~E
  Lewis.
\newblock {BiGG Models: A} platform for integrating, standardizing and sharing
  genome-scale models.
\newblock {\em Nucleic acids research}, 44(D1):D515--D522, 2016.

\bibitem{kook2022conditionnumberindependent}
Yunbum Kook, Yin~Tat Lee, Ruoqi Shen, and Santosh~S. Vempala.
\newblock Condition-number-independent {Convergence Rate of Riemannian
  Hamiltonian Monte Carlo with Numerical Integrators}.
\newblock {\em arXiv preprint arXiv:2210.07219}, 2022.

\bibitem{laddha2020convergence}
Aditi Laddha and Santosh Vempala.
\newblock Convergence of {Gibbs} sampling: {Coordinate Hit-and-Run} mixes fast.
\newblock {\em The 37th International Symposium on Computational Geometry
  (SoCG)}, 2021.

\bibitem{lee2020logsmooth}
Yin~Tat Lee, Ruoqi Shen, and Kevin Tian.
\newblock Logsmooth gradient concentration and tighter runtimes for
  metropolized {Hamiltonian Monte Carlo}.
\newblock In {\em Conference on Learning Theory (COLT)}, pages 2565--2597.
  PMLR, 2020.

\bibitem{lee2018convergence}
Yin~Tat Lee and Santosh~S Vempala.
\newblock Convergence rate of {Riemannian Hamiltonian Monte Carlo} and faster
  polytope volume computation.
\newblock In {\em Proceedings of the 50th Annual ACM SIGACT Symposium on Theory
  of Computing (STOC)}, pages 1115--1121, 2018.

\bibitem{lee2021universal}
Yin~Tat Lee and Man-Chung Yue.
\newblock Universal barrier is $n$-self-concordant.
\newblock {\em Mathematics of Operations Research}, 2021.

\bibitem{lewis2012constraining}
Nathan~E Lewis, Harish Nagarajan, and Bernhard~O Palsson.
\newblock Constraining the metabolic genotype--phenotype relationship using a
  phylogeny of in silico methods.
\newblock {\em Nature Reviews Microbiology}, 10(4):291--305, 2012.

\bibitem{lovasz2006hit}
L{\'a}szl{\'o} Lov{\'a}sz and Santosh Vempala.
\newblock Hit-and-run from a corner.
\newblock {\em SIAM Journal on Computing (SICOMP)}, 35(4):985--1005, 2006.

\bibitem{narayanan_srivastava_2021}
Hariharan Narayanan and Piyush Srivastava.
\newblock On the mixing time of coordinate {Hit-and-Run}.
\newblock {\em Combinatorics, Probability and Computing}, pages 1--13, 2021.

\bibitem{neal2011mcmc}
Radford~M Neal et~al.
\newblock {MCMC} using {Hamiltonian} dynamics.
\newblock {\em Handbook of Markov Chain Monte Carlo}, 2(11):2, 2011.

\bibitem{nesterov1994interior}
Yurii Nesterov and Arkadii Nemirovskii.
\newblock {\em Interior-point polynomial algorithms in convex programming}.
\newblock SIAM, 1994.

\bibitem{pihajoki2015explicit}
Pauli Pihajoki.
\newblock Explicit methods in extended phase space for inseparable hamiltonian
  problems.
\newblock {\em Celestial Mechanics and Dynamical Astronomy}, 121(3):211--231,
  2015.

\bibitem{reich1993symplectic}
Sebastian Reich.
\newblock {\em Symplectic integration of constrained Hamiltonian systems by
  Runge-Kutta methods}.
\newblock University of British Columbia, Department of Computer Science, 1993.

\bibitem{roberts1996exponential}
Gareth~O Roberts and Richard~L Tweedie.
\newblock Exponential convergence of {L}angevin distributions and their
  discrete approximations.
\newblock {\em Bernoulli}, pages 341--363, 1996.

\bibitem{salvatier2016probabilistic}
John Salvatier, Thomas~V Wiecki, and Christopher Fonnesbeck.
\newblock Probabilistic programming in {P}ython using {PyMC3}.
\newblock {\em PeerJ Computer Science}, 2:e55, 2016.

\bibitem{shen2019randomized}
Ruoqi Shen and Yin~Tat Lee.
\newblock The randomized midpoint method for log-concave sampling.
\newblock {\em Advances in Neural Information Processing Systems (NeurIPS)},
  32, 2019.

\bibitem{simsekli2016stochastic}
Umut Simsekli, Roland Badeau, Taylan Cemgil, and Ga{\"e}l Richard.
\newblock Stochastic quasi-newton {Langevin Monte Carlo}.
\newblock In {\em International Conference on Machine Learning (ICML)}, pages
  642--651. PMLR, 2016.

\bibitem{stan}
{Stan Development Team}.
\newblock {RStan}: the {R} interface to {Stan}, 2020.
\newblock R package version 2.21.2.

\bibitem{takahashi1973formation}
Kazuhiro Takahashi.
\newblock Formation of sparse bus impedance matrix and its application to short
  circuit study.
\newblock In {\em Proceeding of PICA Conference, June, 1973}, 1973.

\bibitem{thiele2013community}
Ines Thiele, Neil Swainston, Ronan~MT Fleming, Andreas Hoppe, Swagatika Sahoo,
  Maike~K Aurich, Hulda Haraldsdottir, Monica~L Mo, Ottar Rolfsson, Miranda~D
  Stobbe, et~al.
\newblock A community-driven global reconstruction of human metabolism.
\newblock {\em Nature biotechnology}, 31(5):419--425, 2013.

\end{thebibliography}
